\documentclass{article}

\usepackage{microtype}
\usepackage{graphicx}
\usepackage{subfigure}
\usepackage{booktabs} % for professional tables

\usepackage{epsfig,amsmath,amssymb,amsfonts,amstext,amsthm,mathrsfs,mathtools}
\usepackage{latexsym,graphics,epsf,epsfig,psfrag}
\usepackage{dsfont,color,epstopdf,fixmath}
\usepackage{enumitem}

\usepackage{hyperref}
\usepackage{cleveref}

\newcommand{\msP}{\mathsf{P}}

\newcommand{\mE}{\mathbb E}

\newcommand{\mcs}{\mathcal S}
\newcommand{\mca}{\mathcal A}
\newcommand{\mcb}{\mathcal B}
\newcommand{\mcm}{\mathcal M}
\newcommand{\mcd}{\mathcal D}
\newcommand{\mcp}{\mathcal P}
\newcommand{\mf}{\mathcal F}

\newcommand{\ltwo}[1]{\left\|#1\right\|_2}
\newcommand{\lone}[1]{\left|#1\right|}

\newcommand{\lF}[1]{\left\|#1\right\|_F}
\newcommand{\lTV}[1]{\left\|#1\right\|_{TV}}

\newcommand{\mR}{\mathbb{R}}

\newtheorem{theorem}{Theorem}

\newtheorem{lemma}{Lemma}
\newtheorem{proposition}{Proposition}
\newtheorem{assumption}{Assumption}

\allowdisplaybreaks[4]

\DeclareMathOperator*{\argmax}{argmax}

\usepackage[accepted]{icml2021}

\icmltitlerunning{Doubly Robust Off-Policy Actor-Critic}

\begin{document}

\twocolumn[
\icmltitle{Doubly Robust Off-Policy Actor-Critic: Convergence and Optimality}




\begin{icmlauthorlist}
\icmlauthor{Tengyu Xu}{to}
\icmlauthor{Zhuoran Yang}{go}
\icmlauthor{Zhaoran Wang}{ed}
\icmlauthor{Yingbin Liang}{to}
\end{icmlauthorlist}

\icmlaffiliation{to}{Department of Electrical and Computer Engineering, The Ohio State University}
\icmlaffiliation{go}{ Department of Operations Research and Financial Engineering, Princeton University}
\icmlaffiliation{ed}{Departments of Industrial Engineering \& Management Sciences, Northwestern University}

\icmlcorrespondingauthor{Tengyu Xu}{xu.3260@osu.edu}

\icmlkeywords{Reinforcement Learning, Off-Policy Policy Optimization, Global Convergence}

\vskip 0.3in
]



\printAffiliationsAndNotice{}  % leave blank if no need to mention equal contribution

\begin{abstract}
Designing off-policy reinforcement learning algorithms is typically a very challenging task, because a desirable iteration update often involves an expectation over an on-policy distribution. Prior off-policy actor-critic (AC) algorithms have introduced a new critic that uses the density ratio for adjusting the distribution mismatch in order to stabilize the convergence, but at the cost of potentially introducing high biases due to the estimation errors of both the density ratio and value function. In this paper, we develop a doubly robust off-policy AC (DR-Off-PAC) for discounted MDP, which can take advantage of learned nuisance functions to reduce estimation errors. Moreover, DR-Off-PAC adopts a single timescale structure, in which both actor and critics are updated simultaneously with constant stepsize, and is thus more sample efficient than prior algorithms that adopt either two timescale or nested-loop structure. We study the finite-time convergence rate and characterize the sample complexity for DR-Off-PAC to attain an $\epsilon$-accurate optimal policy. We also show that the overall convergence of DR-Off-PAC is doubly robust to the approximation errors that depend only on the expressive power of approximation functions. To the best of our knowledge, our study establishes the first overall sample complexity analysis for a single time-scale off-policy AC algorithm.
\end{abstract}
%Although the idea of doubly robust biased reduce policy gradient estimator has been proposed \cite{kallus2020statistically} for averaged MDP, the overall convergence of the algorithm with this estimator has not yet been shown to have bias reduced effect. 

\section{Introduction}
In reinforcement learning (RL) \cite{sutton2018reinforcement}, policy gradient and its variant actor-critic (AC) algorithms 
have achieved enormous success in various domains such as game playing \cite{mnih2016asynchronous}, Go \cite{silver2016mastering}, robotic \cite{haarnoja2018soft}, etc. However, these successes usually rely on the access to {\em on-policy} samples, i.e., samples collected online from on-policy visitation (or stationary) distribution. However, in many real-world applications, online sampling during a learning process is costly and unsafe \cite{gottesman2019guidelines}. This necessitates the use of {\em off-policy} methods, which use dataset sampled from a {\em behavior} distribution. Since the policy gradient is expressed in the form of the on-policy expectation, it is challenging to estimate the policy gradient with off-policy samples. A common approach to implement actor-critic algorithms in the off-policy setting is to simply ignore the distribution mismatch between on- and off-policy distributions \cite{degris2012off,silver2014deterministic,lillicrap2016continuous,fujimoto2018addressing,wang2016sample,houthooft2018evolved,meuleau2001exploration}
%. Although these methods have enjoyed certain success in practice, 
but it has been demonstrated that such distribution mismatch can often result in divergence and poor empirical performance \cite{liu2019off}.

Several attempts have been made to correct the distribution mismatch in off-policy actor-critic's update by introducing a reweighting factor in policy update \cite{imani2018off,zhang2019generalized,liu2019off,zhang2019provably,maei2018convergent}, but so far only COF-PAC \cite{zhang2019provably} and OPPOSD \cite{liu2019off} have been theoretically shown to converge without making strong assumptions about the estimation quality. Specifically, COF-PAC reweights the policy update with emphatic weighting approximated by a linear function, and OPPOSD reweights the policy with a density correction ratio learned by a method proposed in \cite{liu2018breaking}. Although both COF-PAC and OPPOSD show much promise by stabilizing the convergence, the convergence results in \cite{zhang2019provably} and \cite{liu2019off} indicate that both algorithms may suffer from a {\bf large bias error} induced by estimations of both reweighting factor and value function.

The doubly robust method arises as a popular technique to reduce such a {\bf bias error}, in which the bias vanishes as long as some (but not necessarily the full set of) estimations are accurate. Such an approach has been mainly applied to the off-policy {\em evaluation} problem \cite{tang2019doubly,jiang2016doubly,dudik2011doubly,dudik2014doubly}, and the development of such a method for solving the policy {\em optimization} problem is rather limited. \cite{huang2020importance} derives a doubly robust policy gradient for finite-horizon Markov Decision Process (MDP) and only for the on-policy setting. \cite{kallus2020statistically} proposed a doubly robust policy gradient estimator for the off-policy setting, but only for infinite-horizon averaged MDP, which does not extend easily to discounted MDP. Moreover, model-free implementation of such doubly robust policy gradient estimators typically requires the estimation of several nuisance functions via samples, but previous works proposed only methods for critic to estimate those nuisances in the finite-horizon setting, which cannot extend efficiently to the infinite-horizon setting.
%\cite{huang2020importance,kallus2020statistically}, which cannot extend efficiently to the infinite-horizon setting.

%So far, it is not clear how to design a doubly robust policy gradient estimator for infinity horizon discounted MDP. Moreover, even if we can obtain an appropriate formulation of doubly robust policy gradient estimator, it is challenging to perform policy update with this estimator, as constructing such an estimator typically requires estimation of several nuisance functions. Previous works have only proposed methods for critic to estimate those nuisances in the finite horizon setting \cite{huang2020importance,kallus2020statistically}, which cannot be applied efficiently and recursively in the infinite horizon setting.

{\em Thus, our first goal is to propose a novel doubly robust policy gradient estimator for infinite-horizon discounted MDP, and further design efficient model-free critics to estimate nuisance functions so that such an estimator can be effectively incorporated to yield a doubly robust off-policy actor-critic algorithm.}

%. Then, based on these developments, we would like to design a new doubly robust off-policy actor-critic algorithm that is provably convergent.}

On the theory side, previous work has established only the {\bf doubly robust estimation}, i.e., the policy gradient estimator is doubly robust \cite{huang2020importance,kallus2020statistically}. However, it is very unclear that by incorporating such a doubly robust estimator into an actor-critic algorithm, whether the overall convergence of the algorithm remains doubly robust, i.e., enjoys {\bf doubly robust optimality gap}. Several reasons may eliminate such a nice property. For example, the alternating update between actor and critic does not allow critics' each estimation to be sufficiently accurate, so that doubly robust estimation may not hold at each round of iteration. Furthermore, the optimality gap of the overall convergence of an algorithm depends on interaction between critics' estimation error and actor's update error as well as other sampling variance errors, so that the double robust {\em estimation} does not necessarily yield the doubly robust {\em optimality gap}.

%Although previous work have verified the effectiveness of doubly robust policy gradient estimator \cite{huang2020importance,kallus2020statistically}, it is not clear whether the doubly robust property can be persisted in the overall convergence of the algorithm. In practice, we usually adopt single-timescale update scheme, in which both actor and critics are updated simultaneously with constant stepsize. In this setting, the updates of actor and critics are strongly coupled with each other, and such an interaction would introduce nonvanishing errors in the policy update. Since the overall convergence of the algorithm is heavily depend on accumulation errors, it is possible that the doubly robust effect brought by the estimator is corrupted by the interaction errors, and as a result the overall convergence of the algorithm would not have doubly robust property.

{\em Thus, our second goal is to establish a finite-time convergence guarantee for our proposed algorithm, and show that the optimality gap of the overall convergence of our algorithm remains doubly robust.}

\subsection{Main Contributions}
{\bf Doubly Robust Estimator:} We propose a new method to derive a doubly robust policy gradient estimator for an infinite-horizon discounted MDP. Comparing with the previously proposed estimators that adjust only the distribution mismatch \cite{liu2019off,imani2018off,zhang2019generalized,zhang2019provably}, our new estimator significantly reduces the bias error when two of the four nuisances in our estimator are accurate (and is hence doubly robust). We further propose a new recursive method for critics to estimate the nuisances in the infinite-horizon off-policy setting. Based on our proposed new estimator and nuisance estimation methods, we develop a model-free doubly robust off-policy actor-critic (DR-Off-PAC) algorithm.

{\bf Doubly Robust Optimality Gap:} We provide the finite-time convergence analysis for our DR-Off-PAC algorithm with single timescale updates. We show that DR-Off-PAC is guaranteed to converge to the optimal policy, and the optimality gap of the overall convergence is also doubly robust to the approximation errors. This result is somewhat surprising, because the doubly robust policy gradient update suffers from both non-vanishing optimization error and approximation error at each iteration, whereas the double robustness of the optimality gap is independent of the optimization error. This also indicates that we can improve the optimality gap of DR-Off-PAC by adopting a powerful function class to estimate certain nuisance functions.

%To the best of our knowledge, we develop the first provably convergent doubly robust off-policy actor-critic algorithm and establish the first finite-time result for single-timescale off-policy actor-critic algorithm. 
Our work is the first that characterizes the doubly robust optimality gap for the overall convergence of off-policy actor-critic algorithms, for which we develop new tools for analyzing actor-critic and critic-critic error interactions.

\subsection{Related Work}
The first off-policy actor-critic algorithm is proposed in \cite{degris2012off} as Off-PAC, and has inspired the invention of many other off-policy actor-critic algorithms such as off-policy DPG \cite{silver2014deterministic}, DDPG \cite{lillicrap2016continuous}, TD3 \cite{fujimoto2018addressing}, ACER \cite{wang2016sample}, and off-policy EPG \cite{houthooft2018evolved}, etc, all of which have the distribution mismatch between the sampling distribution and visitation (or stationary) distribution of updated policy, and hence are not provably convergent under function approximation settings.

In one line of studies, off-policy design adopts reward shaping via entropy regularization and optimizes over a different objective function that does not require the knowledge of behaviour sampling \cite{haarnoja2018soft,o2016combining,dai2018sbeed,nachum2017bridging,nachum2018trust,schulman2017equivalence,haarnoja2017reinforcement,tosatto2020nonparametric}. Although the issue of distribution mismatch is avoided for this type of algorithms, they do not have convergence guarantee in general settings.
%The convergence analysis of aforementioned algorithms have not been developed in the function approximation setting except SBEED \cite{dai2018sbeed}. However, the analysis of SBEED in \cite{dai2018sbeed} assumes (nearly) exact solution of inner problem can be obtained, which is difficult to be satisfied in practice.
The distribution mismatch issue is also avoided in a gradient based algorithm AlgaeDICE \cite{nachum2019algaedice}, in which the original problem is reformulated into a minimax problem. However, since nonconvex minimax objective is in general difficult to optimize, the convergence of AlgaeDICE is not clear.

In another line of works, efforts have been made to address the issue of distribution mismatch in Off-PAC. \cite{imani2018off} developed actor-critic with emphatic weighting (ACE), in which the convergence of Off-PAC is ameliorated by using emphatic weighting \cite{sutton2016emphatic}. Inspired by ACE and the density ratio in \cite{gelada2019off}, \cite{zhang2019generalized} proposed Geoff-PAC to optimize a new objective. Based on Geoff-PAC, \cite{lyu2020variance} further applied the variance reduction technique in \cite{cutkosky2019momentum} to develop a new algorithm VOMPS/ACE-STORM. However, since the policy gradient estimator with emphatic weighting is only unbiased in asymptotic sense and emphatic weighting usually suffers from unbounded variance, the convergence of ACE, Geoff-PAC and VOMPS/ACE-STORM are in general not clear. 
So far, only limited off-policy actor-critic algorithms have been shown to have guaranteed convergence. \cite{zhang2019provably} proposed a provably convergent two timescale off-policy actor-critic via learning the emphatic weights with linear features, and \cite{liu2019off} proposed to reweight the off-PAC update via learning the density ratio with the approach in \cite{liu2018breaking}. 
However, both convergence results in \cite{liu2019off} and \cite{zhang2019provably} suffer from bias errors of function approximation, and the two timescale update and the double-loop structure adopted in \cite{zhang2019provably} and \cite{liu2019off}, respectively, can cause significant sample inefficiency.
Recently, \cite{kallus2020statistically} proposed an off-policy gradient method with doubly robust policy gradient estimator. However, they also adopted an inefficient double-loop structure and the overall convergence of the algorithm with such an estimator was not shown to have the doubly robust property.
In contrast to previous works, we develop a new doubly robust off-policy actor-critic that provably converges with the overall convergence also being doubly robust to the function approximation errors. Our algorithm adopts a single-timescale update scheme, and is thus more sample efficient than the previous methods \cite{liu2019off,zhang2019provably,kallus2020statistically}.
%\cite{kallus2020statistically} studied the convergence of doubly robust PG in the finite-horizon MDP. However, the convergence result in \cite{kallus2020statistically} dose not capture the overall doubly robust property of the algorithm. Moreover, the analysis in \cite{kallus2020statistically} ignore the iteration of critics and thus is less challenging than the single-loop, tri-level structure algorithm that we studied here. Our work also characterizes the global convergence of doubly robust PG, which has not been studied in previous works.
\section{Background and Problem Formulation}\label{sc: bkgrd}
In this section, we introduce the background of Markov Decision Process (MDP) and problem formulation.
%\subsection{Markov Decision Process (MDP)}
We consider an infinite-horizon MDP described by $(\mcs, \mca, \msP, r, \mu_0, \gamma)$, where $\mcs$ denotes the set of states, $\mca$ denotes the set of actions, and $\msP(s^\prime|s, a)$ denotes the transition probability from state $s\in\mcs$ to state $s^\prime$ with action $a\in\mca$. Note that $\lone{\mcs}$ and $\lone{\mca}$ can be infinite such that $\msP(s^\prime|s, a)$ is then a Markov kernel. Let $r(s,a,s^\prime)$ be the reward that an agent receives if the agent takes an action $a$ at state $s$ and the system transits to state $s^\prime$. Moreover, we denote $\mu_0$ as the distribution of the initial state $s_0\in\mcs$ and $\gamma\in (0,1)$ as the discount factor. Let $\pi(a|s)$ be the policy which is the probability of taking action $a$ given current state $s$. Then, for a given policy $\pi$, we define the state value function as $V_\pi(s)=\mE[\gamma^t r(s_t, a_t, s_{t+1})|s_0=s,\pi]$ and the state-action value function as $Q_\pi(s,a)=\mE[\gamma^t r(s_t, a_t, s_{t+1})|s_0=s,a_0=a,\pi]$. Note that $V_{\pi}(s) = \mE_{\pi}[Q_{\pi}(s,a)|s]$ and $Q_{\pi}(s,a)$ satisfies the following Bellman equation:
\begin{flalign}\label{eq: bellman_q}
	Q_\pi(s,a) = R(s,a) + \gamma \mcp_\pi Q_\pi(s,a),
\end{flalign}
where $R(s,a) = \mE[r(s,a,s^\prime)|s,a]$ and 
\begin{flalign*}
	\mcp_\pi Q_\pi(s,a)\coloneqq \mE_{s^\prime\sim\msP(\cdot|s,a), a^\prime\sim\pi(\cdot|s^\prime)}[Q_\pi(s^\prime,a^\prime)].
\end{flalign*}
We further define the expected total reward function as $J(\pi)=(1-\gamma)\mE[\gamma^t r(s_t, a_t, s_{t+1})|s_0\sim\mu_0,\pi]=\mE_{\mu_0}[V_\pi(s)]=\mE_{\nu_\pi}[r(s,a,s^\prime)]$, where $\nu_\pi(s,a)=(1-\gamma)\sum_{t=0}^{\infty}\gamma^t\msP(s_t=s, a_t=a|s_0\sim\mu_0,\pi)$ is the visitation distribution. The visitation distribution satisfies the following ``inverse" Bellman equation:
\begin{flalign}\label{eq: bellman_vis}
	\nu_\pi(s^\prime,a^\prime) &= \pi(a^\prime|s^\prime)[ (1-\gamma)\mu_0(s^\prime) \nonumber\\
	&+ \gamma \int_{(s,a)} \msP(s^\prime|s,a)\nu_\pi(s,a)dsda].
\end{flalign}
%where 
%\begin{flalign*}
%	\mcp^*_\pi \nu_\pi(s,a) \coloneqq \mE_{(s,a)\sim\nu_\pi(\cdot),s^\prime\sim\msP(\cdot|s,a), a^\prime\sim\pi(\cdot|s^\prime)}[\nu_\pi(s,a)].
%\end{flalign*}
In policy optimization, the agent's goal is to find an optimal policy $\pi^*$ that maximizes $J(\pi)$, i.e., $\pi^*=\argmax_\pi J(\pi)$. We consider the setting in which policy $\pi$ is parametrized by $w\in\mR^d$. Then, the policy optimization is to solve the problem $\max_wJ(\pi_w)$. In the sequel we write $J(\pi_w) := J(w)$ for notational simplicity. A popular approach to solve such a maximization problem is the policy gradient method, in which we update the policy in the gradient ascent direction as $w_{t+1} = w_t + \alpha \nabla_wJ(w_t)$. A popular form of $\nabla_wJ(w)$ is derived by \cite{sutton2000policy} as
\begin{flalign}\label{pg}
	\nabla_wJ(w) = \mE_{\nu_{\pi_w}}[Q_{\pi_w}(s,a)\nabla_w\log\pi_w(a|s)].
\end{flalign}
In the on-policy setting, many works adopt the policy gradient formulation in \cref{pg} to estimate $\nabla_wJ(w)$, which requires sampling from the visitation distribution $\nu_{\pi_w}$ and Monte Carlo rollout from policy $\pi_w$ to estimate the value function $Q_{\pi_w}(s,a)$ \cite{zhang2019global,xiong2020non}.
%\subsection{Policy Gradient}
%In the function approximation setting, we consider finding a good policy over a parameterized class $\{\pi_w, w\in\mR^d\}$, where $\pi_w$ can be general nonlinear differentiable parameterization. In order to optimize $J(\pi_w)\coloneqq J(w)$ by gradient-based approaches, the {\em policy gradient theorem} \cite{sutton2000policy} provides the gradient $\nabla J(w)$ as 
%\begin{flalign}\label{pg-I}
%	\nabla_w J(w) = \mE_{\nu_{\pi_w}}[Q_{\pi_w}(s,a)\nabla_w\log\pi_w(a|s)].
%\end{flalign}
%To estimate this gradient one requires access to sample from the current visitation distribution $\nu_{\pi_w}$ and estimation of the Q-value function $Q_{\pi_w}(s,a)$.

In this paper we focus on policy optimization in the behavior-agnostic off-policy setting. Specifically, we are given access to samples from a fixed distribution $ \{(s_i, a_i, r_i, s^\prime_i)\}\sim \mcd_d$, where the state-action pair $(s_i, a_i)$ is sampled from an \textbf{\em unknown} distribution $d(\cdot): \mcs\times\mca\rightarrow[0,1]$, the successor state $s^\prime_i$ is sampled from $\msP(\cdot|s_i,a_i)$ and $r_i$ is the received reward. We also have access to samples generated from the initial distribution, i.e., $s_{0,i}\sim \mu_0$. In the behavior-agnostic off-policy setting, it is difficult to estimate $\nabla_wJ(w)$ directly with the form in \cref{pg}, as neither $\nu_{\pi_w}$ nor Monte Carlo rollout sampling is accessible. Thus, our goal is to develop an {\em efficient} algorithm to estimate $\nabla_wJ(w)$ with off-policy samples from $\mcd_d$, and furthermore, establish the convergence guarantee for our proposed algorithm. 
\section{DR-Off-PAC: Algorithm and Convergence}\label{sc: convegence}
In this section, we first develop a new doubly robust policy gradient estimator and then design a new doubly robust off-policy actor-critic algorithm.
%optimization algorithm DR-Off-PAC, and then establish the finite-time convergence guarantee for our proposed algorithm.

\subsection{Doubly Robust Policy Gradient Estimator}\label{sc: nuisance}
In this subsection, we construct a new doubly robust policy gradient estimator for an infinite-horizon discounted MDP. We first denote the density ratio as $\rho_{\pi_w} = \nu_{\pi_w}(s,a)/d(s,a)$, and denote the derivative of $Q_{\pi_w}$ and $\rho_{\pi_w}$ as $d^q_{\pi_w}$ and $d^\rho_{\pi_w}$, respectively. 

Previous constructions \cite{kallus2020statistically} for such an estimator directly combine the policy gradient with a number of error terms under various filtrations to guarantee the double robustness. Such a method does not appear to extend easily to the discounted MDP. Specifically, the method in \cite{kallus2020statistically} considers finite-horizon MDP with $\gamma=1$, and further extends their result to infinite-horizon average-reward MDP. Their extension relies on the fact that the objective function $J(w)$ in average-reward MDP is independent of the initial distribution $\mu_0$. In contrast, $J(w)$ in discounted-reward MDP depends on $\mu_0$. Any direct extension necessarily results in a bias due to the lack of the initial distribution, which is unknown a priori, and hence loses the doubly robust property. 

To derive a doubly robust gradient estimator in the discounted MDP setting, we first consider a bias reduced estimator of the objective $J(w)$ with off-policy sample $(s,a,r,s^\prime)$ and $s_0$, and then take the derivative of such an estimator to obtain a doubly robust policy gradient estimator. The idea behind this derivation is that as long as the objective estimator has small bias, the gradient of such an estimator can also have small bias. More detailed discussion can be referred to the supplement material. 

Given sample $s_{0}\sim \mu_0(\cdot)$ and $(s,a,r,s^\prime)\sim \mcd_d$ and estimators $\hat{Q}_{\pi_w}$, $\hat{\rho}_{\pi_w}$, $\hat{d}^{q}_{\pi_w}$ and $\hat{d}^{\rho}_{\pi_w}$, our constructed doubly robust policy gradient error is given as follows. 
	{\small \begin{flalign}
	&G_{\text{DR}}(w) \nonumber\\
	&= (1-\gamma)\Big(\hat{Q}_{\pi_w}(s_{0},a_{0})\nabla_w\log\pi_w(a_{0}|s_{0}) + \hat{d}^q_{\pi_w}(s_{0},a_{0})\Big) \nonumber\\
	&+ \hat{d}^\rho_{\pi_w}(s,a)\left(r(s,a,s^\prime) - \hat{Q}_{\pi_w}(s,a) + \gamma \hat{Q}_{\pi_w}(s^\prime,a^\prime)\right) \nonumber\\
	&+ \hat{\rho}_{\pi_w}(s,a)\Big[- \hat{d}^q_{\pi_w}(s,a) \nonumber\\
	&+ \gamma \Big( \hat{Q}_{\pi_w}(s^\prime,a^\prime)\nabla_w\log\pi_w(a^\prime|s^\prime) + \hat{d}^q_{\pi_w}(s^\prime,a^\prime) \Big)\Big],\label{dr_pg}
	\end{flalign}}
%	\begin{flalign}
%	&G_{\text{DR}}(w) = (1-\gamma)\hat{d}^v_{\pi_w}(s_{0,i}) + \hat{d}^\rho_{\pi_w}(s_i,a_i)(r(s_i,a_i,s^\prime_i) \nonumber\\
%	&- \hat{Q}_{\pi_w}(s_i,a_i) + \gamma \hat{V}_{\pi_w}(s^\prime_i)) \nonumber\\
%	&+ \hat{\rho}_{\pi_w}(s_i,a_i)(- \hat{d}^q_{\pi_w}(s_i,a_i) + \gamma \hat{d}^v_{\pi_w}(s^\prime_i)),
%	\end{flalign}
where $a_{0}\sim \pi_w(\cdot|s_{0})$ and $a^\prime\sim\pi_w(\cdot|s^\prime)$.
The following theorem establishes that our proposed estimator $G_{\text{DR}}$ satisfies the doubly robust property.
\begin{theorem}\label{thm1}
%	Given sample $s_{0}\sim \mu_0(\cdot)$ and $(s,a,r,s^\prime)\sim \mcd_d$ and estimators $\hat{Q}_{\pi_w}$, $\hat{\rho}_{\pi_w}$, $\hat{d}^{q}_{\pi_w}$ and $\hat{d}^{\rho}_{\pi_w}$, we can formulate the following policy gradient estimator
%	{\small \begin{flalign}
%	&G_{\text{DR}}(w) \nonumber\\
%	&= (1-\gamma)\Big(\hat{Q}_{\pi_w}(s_{0},a_{0})\nabla_w\log\pi_w(a_{0}|s_{0}) + \hat{d}^q_{\pi_w}(s_{0},a_{0})\Big) \nonumber\\
%	&+ \hat{d}^\rho_{\pi_w}(s,a)\left(r(s,a,s^\prime) - \hat{Q}_{\pi_w}(s,a) + \gamma \hat{Q}_{\pi_w}(s^\prime,a^\prime)\right) \nonumber\\
%	&+ \hat{\rho}_{\pi_w}(s,a)\Big[- \hat{d}^q_{\pi_w}(s,a) \nonumber\\
%	&+ \gamma \Big( \hat{Q}_{\pi_w}(s^\prime,a^\prime)\nabla_w\log\pi_w(a^\prime|s^\prime) + \hat{d}^q_{\pi_w}(s^\prime,a^\prime) \Big)\Big],\label{dr_pg}
%	\end{flalign}}
%%	\begin{flalign}
%%	&G_{\text{DR}}(w) = (1-\gamma)\hat{d}^v_{\pi_w}(s_{0,i}) + \hat{d}^\rho_{\pi_w}(s_i,a_i)(r(s_i,a_i,s^\prime_i) \nonumber\\
%%	&- \hat{Q}_{\pi_w}(s_i,a_i) + \gamma \hat{V}_{\pi_w}(s^\prime_i)) \nonumber\\
%%	&+ \hat{\rho}_{\pi_w}(s_i,a_i)(- \hat{d}^q_{\pi_w}(s_i,a_i) + \gamma \hat{d}^v_{\pi_w}(s^\prime_i)),
%%	\end{flalign}
%	where $a_{0}\sim \pi_w(\cdot|s_{0})$ and $a^\prime\sim\pi_w(\cdot|s^\prime)$. 
The bias error of estimator $G_{\text{DR}}(w)$ in \cref{dr_pg} satisfies
	{\small \begin{flalign*}
	&\mE[G_{\text{DR}}(w)] - \nabla_w J(w)\nonumber\\
	&=-\mE[\varepsilon_{\rho}(s,a)\varepsilon_{d^q}(s,a)] - \mE[\varepsilon_{d^\rho}(s,a)\varepsilon_{q}(s,a)] \nonumber\\
	&\quad + \gamma \mE[\varepsilon_{\rho}(s,a)\varepsilon_{q}(s^\prime,a^\prime)\nabla_w\log(a^\prime|s^\prime)] \nonumber\\
	&\quad + \gamma\mE[\varepsilon_{\rho}(s,a)\varepsilon_{d^q}(s^\prime,a^\prime)] + \gamma\mE[\varepsilon_{\rho}(s,a)\varepsilon_{q}(s^\prime, a^\prime)],
	\end{flalign*}}
	where the estimation errors are defined as
	{\begin{flalign*}
	&\varepsilon_{\rho} = \rho_{\pi_w}- \hat{\rho}_{\pi_w}, \quad \varepsilon_{q} = Q_{\pi_w} - \hat{Q}_{\pi_w},\nonumber\\
	&\varepsilon_{d^\rho} = d^\rho_{\pi_w}- \hat{d}^\rho_{\pi_w}, \quad \varepsilon_{d^q} = d^q_{\pi_w} - \hat{d}^q_{\pi_w}.
	\end{flalign*}}
\end{theorem}
% ($\hat{\rho}_{\pi_w}=\rho_{\pi_w}$ and $\hat{d}^\rho_{\pi_w}=d^\rho_{\pi_w}$)
% ($\hat{Q}_{\pi_w}=Q_{\pi_w}$ and $\hat{d}^q_{\pi_w}=d^q_{\pi_w}$)
% ($\hat{\rho}_{\pi_w}=\rho_{\pi_w}$ and $\hat{Q}_{\pi_w}=Q_{\pi_w}$)
\Cref{thm1} shows that the estimation error of $G_{\text{DR}}(w)$ takes a {\bf multiplicative} form of pairs of individual estimation errors rather than the summation over all errors. Such a structure thus exhibits a {\bf three-way doubly robust} property. Namely, as long as {\bf one} of the three pairs $(\hat{\rho}_{\pi_w},\hat{d}^\rho_{\pi_w})$, $(\hat{Q}_{\pi_w},\hat{d}^q_{\pi_w})$, $(\hat{\rho}_{\pi_w},\hat{Q}_{\pi_w})$ are accurately estimated, our estimator $G_{\text{DR}}(w)$ is unbiased, i.e., $\mE[G_{\text{DR}}(w)] - \nabla_w J(w)=0$. There is no need for all of the individual errors to be small.

\subsection{Estimation of Nuisance Functions}\label{subsc: nuisances}

In order to incorporate the doubly robust estimator \cref{dr_pg} into an actor-critic algorithm, we develop critics to respectively construct efficient estimators $\hat{Q}_{\pi_w}$, $\hat{\rho}_{\pi_w}$, $\hat{d}^q_{\pi_w}$, $\hat{d}^\rho_{\pi_w}$ in $G_{\text{DR}}(w)$ in the linear function approximation setting.

\textbf{Critic I: Value function $\hat{Q}_{\pi_w}$ and density ratio $\hat{\rho}_{\pi_w}$.} In the off-policy evaluation problem, \cite{yang2020off} shows that the objective function $J(w)$ can be expressed by the following primal linear programming (LP):
{\small \begin{flalign*}
	&\min_{Q_{\pi_w}}\quad (1-\gamma)\mE_{\mu_0\pi_w}[Q_{\pi_w}(s,a)]\nonumber\\
	&\text{s.t.,}\quad Q_{\pi_w}(s,a) = R(s,a) + \gamma \mcp_\pi Q_{\pi_w}(s,a),
\end{flalign*}}
with the corresponding dual LP given by
{\small \begin{flalign*}
	&\max_{\nu_{\pi_w}}\quad \mE_{\nu_{\pi_w}}[R(s,a)]\nonumber\\
	&\text{s.t.,}\quad \nu_\pi(s^\prime,a^\prime) = (1-\gamma)\mu_0(s^\prime)\pi(a^\prime|s^\prime) + \gamma \mcp^*_\pi\nu_\pi(s,a).
\end{flalign*}}
Then, the value function $Q_{\pi_w}(s,a)$ and the distribution correction ratio $\rho_{\pi_w}(s,a)$  can be learned by solving the following regularized Lagrangian:
{\small \begin{flalign}
	&\quad \min_{\hat{\rho}_{\pi_w}\geq 0}\max_{\hat{Q}_{\pi_w}, \eta} L(\hat{\rho}_{\pi_w}, \hat{Q}_{\pi_w}, \eta)\nonumber\\
	&\coloneqq (1-\gamma)\mE_{\mu_0}[\hat{Q}_{\pi_w}(s,a)] + \mE_{\mcd_d}[\rho_{\pi_w}(s,a)(r(s,a,s^\prime)\nonumber\\
	&\quad + \gamma \hat{Q}_{\pi_w}(s^\prime,a^\prime) - \hat{Q}_{\pi_w}(s,a))] - \frac{1}{2}\mE_{\mcd_d}[\hat{Q}_{\pi_w}(s,a)^2]\nonumber\\
	&\quad + \mE_{\mcd_d}[\eta\hat{\rho}_{\pi_w}(s,a)-\eta]-0.5\eta^2.\label{eq: 4}
\end{flalign}}
We construct $\hat{\rho}_{\pi_w}$ and $\hat{Q}_{\pi_w}$ with linearly independent feature $\phi(s,a)\in\mR^{d_1}$: $\hat{\rho}_{\pi_w}(s,a) = \phi(s,a)^\top\theta_{\rho}$ and $\hat{Q}_{\pi_w}(s,a) = \phi(s,a)^\top \theta_{q}$ for all $(s,a)\in\mcs\times\mca$. In such a case, $L(\rho_{\pi_w}, Q_{\pi_w}, \eta)$ is strongly-concave in both $\theta_q$ and $\eta$, and convex in $\theta_\rho$. We denote the global optimum of $L(\theta_\rho, \theta_q, \eta)$ as $\theta^*_{\rho,w}$, $\theta^*_{q,w}$ and $\eta^*_{w}$. The errors of approximating $Q_{\pi_w}$ and $\rho_{\pi_w}$ with estimators $\hat{Q}_{\pi_{w}}(s,a,\theta^*_{q,w})=\phi(s,a)^\top\theta^*_{q,w}$ and $\hat{\rho}_{\pi_{w}}(s,a, \theta^*_{q,w})=\phi(s,a)^\top\theta^*_{\rho,w}$, respectively, are defined as
{\small \begin{flalign*}
	\epsilon_q & = \max\Big\{ \max_w\sqrt{\mE_{\mcd}[(\hat{Q}_{\pi_{w}}(s,a,\theta^*_{q,w}) - {Q}_{\pi_{w}}(s,a))^2]},\nonumber\\ &\max_w\sqrt{\mE_{\mcd_d\cdot\pi_{w}}[(\hat{Q}_{\pi_{w}}(s^\prime,a^\prime,\theta^*_{q,w}) - {Q}_{\pi_{w}}(s^\prime,a^\prime))^2]} \Big\},\nonumber\\
	\epsilon_\rho &= \max_w\sqrt{ \mE_{\mcd}[(\hat{\rho}_{\pi_{w}}(s,a, \theta^*_{\rho,w}) - {\rho}_{\pi_{w}}(s,a))^2]}.
\end{flalign*}}

To solve the minimax optimization problem in \cref{eq: 4}, we adopt stochastic gradient descent-ascent method with mini-batch samples $\mcb_t = \{(s_i, a_i, r_i, s^\prime_i)\}_{i=1\cdots N}\sim\mcd_d$, $a^\prime_i\sim\pi_{w_t}(\cdot|s^\prime_i)$ and $\mcb_{t,0} = \{(s_{0,i})\}_{i=1\cdots N}\sim\mu_0$, $a_{0,i}\sim\pi_{w_t}(\cdot|s^\prime_{0,i})$, which update parameters recursively as follows
{\small \begin{flalign}
	\delta_{t,i} &= (1-\gamma)\phi_{0,i} + \gamma \phi_i^\top\theta_{\rho,t}\phi^\prime_i - \phi_i^\top\theta_{\rho,t}\phi_i\nonumber\\
	\eta_{t+1} &= \theta_{\rho,t} + \beta_1\frac{1}{N}\sum_{i\in\mcb_t} (\phi_i^\top\theta_{\rho,t} - 1 - \eta_t)\nonumber\\
	\theta_{q,t+1} &= \Gamma_{R_q}\Big[\theta_{q,t} + \beta_1\frac{1}{N}\sum_{i\in\mcb_t,\mcb_{t,0}}(\delta_{t,i} - \phi^\top_i\theta_{q,t}\phi_i)\Big]\nonumber\\
	\theta_{\rho,t+1} &=\Gamma_{R_\rho} \Big[\theta_{\rho,t} - \beta_1\frac{1}{N}\sum_{i\in\mcb_t}(r_i\phi_i + \gamma \phi^{\prime\top}_i\theta_{q,t}\phi_i \nonumber\\
	&\qquad\qquad\qquad - \phi_i^\top\theta_{q,t}\phi_i + \eta_t\phi_i )\Big],\label{eq: c1}
\end{flalign}}
where $\Gamma_R$ indicates the projection onto a ball with radius $R$. Such a projection operator stabilizes the algorithm \cite{konda2000actor,bhatnagar2009natural}. Note that the iteration in \cref{eq: c1} is similar to but difference from the GradientDICE update in \cite{zhang2020gradientdice}, as GradientDICE can learn only the density ratio $\rho_{\pi_w}$, while our approach in \cref{eq: c1} can learn both the value function $Q_{\pi_w}$ and the density ratio $\rho_{\pi_w}$.

\textbf{Critic II: Derivative of value function $\hat{d}^q_{\pi_w}$.} Taking derivative on both sides of \cref{eq: bellman_q} yields
{\begin{flalign}
d^q_{\pi_w}(s,a) &= \gamma\mE[d^q_{\pi_w}(s^\prime,a^\prime)|s,a] \nonumber\\
&+ \gamma\mE[Q_{\pi_w}(s^\prime,a^\prime)\nabla_w\log\pi_w(a^\prime|s^\prime)|s,a],\label{eq: dq}
\end{flalign}}
We observe that \cref{eq: dq} takes a form analogous to the Bellman equation in \cref{eq: bellman_q}, and thus suggests a recursive approach to estimate $d^q_{\pi_w}$, similarly to  temporal difference (TD) learning.  
%\cref{eq: dq}, $d^q_{\pi_w}(s,a)$ could be estimated using an approach similar to temporal difference (TD) learning. 
Specifically, suppose we estimate $d^q_{\pi_w}$ with a feature matrix $x(s,a)\in \mR^{d_3\times d}$, i.e., $\hat{d}^q_{\pi_w}(s,a) = x(s,a)^\top \theta_{d_q}$ for all $(s,a)\in\mcs\times\mca$. Replace $Q_{\pi_w}(s,a)$ with its estimator $\hat{Q}_{\pi_w}(s,a) = \phi(s,a)^\top \theta_q$ in \cref{eq: dq}. The temporal difference error is then given as
{\begin{flalign*}
\delta_{d_q}&(s,a,\theta_q) = \gamma x(s^\prime,a^\prime)^\top \theta_{d_q} \nonumber\\
&+ \gamma \phi(s^\prime,a^\prime)^\top \theta_q\nabla_w\log\pi_w(a^\prime|s^\prime) - x(s,a)^\top \theta_{d_q}
\end{flalign*}}
and $\theta_{d^q}$ can be updated with the TD-like semi-gradient
\begin{flalign}\label{eq: 5}
\theta_{d^q,t+1} = \theta_{d^q,t} + \beta_2  x(s,a) \delta_{d^q}(s,a,\theta_{d^q,t}).
\end{flalign}
However, in the off-policy setting, the iteration in \cref{eq: 5} may not converge due to the off-policy sampling. To solve such an issue, we borrow the idea from gradient TD (GTD) and formulate the following strongly convex objective
{\begin{flalign*}
&H(\theta_{d_q},\theta_q) \nonumber\\
&= \mE[x(s,a)\delta_{d_q}(s,a,\theta_q)]^\top\mE[x(s,a)\delta_{d_q}(s,a,\theta_q)].
\end{flalign*}}
We denote the global optimum of $H(\theta_{d_q},\theta^*_{q,w})$ as $\theta^*_{d_q,w}$, i.e., $H(\theta^*_{d_q,w},\theta^*_{q,w})=0$. The approximation error of estimating $d^q_{\pi_w}$ with estimator $\hat{d}^q_{\pi_{w}}(s^\prime,a^\prime,\theta^*_{d_q,w}) = x(s,a)^\top\theta^*_{d_q,w}$ is defined as
{\small\begin{flalign}
	&\epsilon_{d_q} = \max\Big\{ \max_w\sqrt{\mE_{\mcd}\left[ \ltwo{\hat{d}^q_{\pi_{w}}(s,a,\theta^*_{d_q,w}) - {d}^q_{\pi_{w}}(s,a)}^2\right]},\nonumber\\
	&\max_w\sqrt{\mE_{\mcd_d\cdot\pi_{w}}\left[ \ltwo{\hat{d}^q_{\pi_{w}}(s^\prime,a^\prime,\theta^*_{d_q,w}) - {d}^q_{\pi_{w}}(s^\prime,a^\prime)}^2\right]} \Big\}.\nonumber
	\end{flalign}}
Similarly to GTD, we introduce an auxiliary variable $w_{d^q}$ to avoid the issue of double sampling when using gradient based approach to minimize $H(\theta_{d_q},\theta_q)$.
With mini-batch samples $\mcb_t = \{(s_i, a_i, s^\prime_i)\}_{i=1\cdots N}\sim\mcd_d$, we have the following update for $\theta_{d_q}$.
{\small \begin{flalign}
\theta_{d_q,t+1} &= \theta_{d_q,t} + \beta_3 \frac{1}{N}\sum_{i\in {\mcb}_t} (x_i - \gamma x^\prime_i)x^\top_iw_{d_q,t},\nonumber\\
w_{d_q, t+ 1} &= w_{d_q,t} + \beta_3 \frac{1}{N}\sum_{i\in {\mcb}_t} (x_i \delta_{d_q,i}(\theta_{q,t}) - w_{d_q,t}).\label{eq: c3}
\end{flalign}}

\textbf{Critic III: Derivative of density ratio $\hat{d}^\rho_{\pi_w}$.} We denote $\psi_{\pi_w}(s,a): = \nabla_w\log(\nu_{\pi_w}(s,a))$, and construct an estimator for ${d}^\rho_{\pi_w}$ as $\hat{d}^\rho_{\pi_w}(s,a) = \hat{\rho}_{\pi_w}(s,a)\hat{\psi}_{\pi_w}(s,a)$, where $\hat{\rho}_{\pi_w}$ and $\hat{\psi}_{\pi_w}$ are approximation of $\rho_{\pi_w}$ and $\psi_{\pi_w}$, respectively. 
Note that \cref{eq: bellman_vis} can be rewritten in the following alternative form
\begin{flalign}\label{eq: bellman_vis2}
	\nu_{\pi_w}(\tilde{s}^\prime,a^\prime) = \int \pi_w(a^\prime|\tilde{s}^\prime)\tilde{\msP}(\tilde{s}^\prime|s,a)\nu_{\pi_w}(s,a)dsda,
\end{flalign}
where $\tilde{\msP}(\cdot|s,a)=(1-\gamma)\mu_0 + \gamma \msP(\cdot|s,a)$. Taking derivative on both sides of \cref{eq: bellman_vis2} and using $\nabla g(w)=g(w)\nabla \log g(w)$, we obtain
\begin{align*}
&\textstyle \nu_{\pi_w}(\tilde{s}^\prime,a^\prime)\psi_{\pi_w}(\tilde{s}^\prime,a^\prime)\\ 
& \textstyle =\nabla_w\log(\pi_w(a^\prime|\tilde{s}^\prime))\cdot \Big[\pi_w(a^\prime|\tilde{s}^\prime)\nonumber\\
&\quad\int_{s,a}\tilde{\msP}(\tilde{s}^\prime|s,a)\nu_{\pi_w}(s,a)dsda\Big] \\
&\textstyle \quad + \int_{s,a} \Big[\pi_w(a^\prime|\tilde{s}^\prime)\tilde{\msP}(\tilde{s}^\prime|s,a)\nu_{\pi_w}(s,a)\Big] \psi_{\pi_w}(s,a)dsda \\
& \textstyle = \nabla_w\log(\pi_w(a^\prime|\tilde{s}^\prime))\cdot\nu_{\pi_w}(\tilde{s}^\prime,a^\prime)\\
&\textstyle \quad + \int_{s,a}\Big[\pi_w(a^\prime|\tilde{s}^\prime)\tilde{\msP}(\tilde{s}^\prime|s,a)\nu_{\pi_w}(s,a)\Big]\psi_{\pi_w}(s,a)dsda \\
& \textstyle = \nabla_w\log(\pi_w(a^\prime|\tilde{s}^\prime))\cdot \nu_{\pi_w}(\tilde{s}^\prime,a^\prime) \\
&\textstyle \quad + \int_{s,a} \nu_{\pi_w}(\tilde{s}^\prime,a^\prime)P\left(s,a|\tilde{s}^\prime,a^\prime\right) \psi_{\pi_w}(s,a)dsda
\end{align*}
where the second equality follows because ${\pi_w(a^\prime|\tilde{s}^\prime)\int_{s,a}\tilde{\msP}(\tilde{s}^\prime|s,a)\nu_{\pi_w}(s,a)dsda}=\nu_{\pi_w}(\tilde{s}^\prime,a^\prime)$, and the third equality follows because if $(s,a)\sim \nu_{\pi_w}(\cdot)$, then $(\tilde{s}^\prime,a^\prime)\sim\nu_{\pi_w}(\cdot)$, and Bayes' theorem implies that $\frac{\pi_w(a^\prime|\tilde{s}^\prime)\tilde{\msP}(\tilde{s}^\prime|s,a)\nu_{\pi_w}(s,a)}{\nu_{\pi_w}(\tilde{s}^\prime,a^\prime)} = P\left(s,a|\tilde{s}^\prime,a^\prime\right)$. Then, dividing both sides by $\nu_{\pi_w}(\tilde{s}^\prime,a^\prime)$ yields
\begin{flalign}
&\psi_{\pi_w}(\tilde{s}^\prime,a^\prime) \nonumber\\
&=\nabla_w\log(\pi_w(a^\prime|\tilde{s}^\prime)) + \int_{s,a}P(s,a|\tilde{s}^\prime,a^\prime)\psi_{\pi_w}(s,a)dsda.\label{eq: bellman_vis3}
\end{flalign}
%\begin{flalign*}
%	&\nu_{\pi_w}(x^\prime)\psi_{\pi_w}(x^\prime)\nonumber\\
%	&=  \nabla_w\log(\pi_w(a^\prime|\tilde{s}^\prime))\cdot{\Big[\pi_w(a^\prime|\tilde{s}^\prime)\int_{x}\tilde{\msP}(\tilde{s}^\prime|x)\nu_{\pi_w}(x)dx\Big]} \nonumber\\
%	&\quad + \int_{x}{\Big[\pi_w(a^\prime|\tilde{s}^\prime)\tilde{\msP}(\tilde{s}^\prime|x)\nu_{\pi_w}(x)\Big]} \psi_{\pi_w}(x)dx
%\end{flalign*}
%Then, dividing both sides of the above equation by $\nu_{\pi_w}(x^\prime)$ yields
%\begin{flalign}
%&\psi_{\pi_w}(x^\prime) =  \nabla_w\log(\pi_w(a^\prime|\tilde{s}^\prime)) \nonumber\\
%&\quad + \int_{x}{\Big[\pi_w(a^\prime|\tilde{s}^\prime)\tilde{\msP}(\tilde{s}^\prime|x)\nu_{\pi_w}(x)\Big]} \psi_{\pi_w}(x)dx\nonumber\\
%&=\nabla_w\log(\pi_w(a^\prime|\tilde{s}^\prime)) + \int_{x}P(x|x^\prime)\psi_{\pi_w}(x)dx,\label{eq: bellman_vis3}
%\end{flalign}
%where the first equality follows from the fact that ${{\pi_w(a^\prime|\tilde{s}^\prime)\cdot\int_{x}\tilde{\msP}(\tilde{s}^\prime|x)\nu_{\pi_w}(x)dx}=\nu_{\pi_w}(x^\prime)}$, the second equality follows from the fact that when $x\sim \nu_{\pi_w}(\cdot)$, then $x^\prime\sim\nu_{\pi_w}(\cdot)$, which further implied by Bayes' theorem that
%\begin{flalign*}
%	\frac{\pi_w(a^\prime|\tilde{s}^\prime)\tilde{\msP}(\tilde{s}^\prime|x)\nu_{\pi_w}(x)}{\nu_{\pi_w}(x^\prime)} = P\left(x|x^\prime\right).
%\end{flalign*}
%In the sequel, we denote $\psi_{\pi_w}(x) = \nabla_w\log(\nu_{\pi_w}(x))$.
With linear function approximation, we estimate $\psi_{\pi_w}(s,a)$ with feature matrix $\varphi(s,a)\in\mR^{d_2\times d}$ i.e., $\hat{\psi}_{\pi_w}(s,a) = \varphi(s,a)^\top \theta_\psi$ for all $(s,a)\in\mcs\times\mca$. The temporal difference error is given as
\begin{flalign}
	&\delta_{\psi}(\tilde{s}^\prime,a^\prime) \nonumber\\
	&= \nabla_w\log\pi_w(a^\prime|\tilde{s}^\prime) +  \varphi(s,a)^\top \theta_\psi - \varphi(\tilde{s}^\prime,a^\prime)^\top \theta_\psi,\label{eq: 69}
\end{flalign}
Note that in \cref{eq: 69}, we require $\tilde{s}^\prime\sim \tilde{\msP}(\cdot|s,a)$. To obtain a sample triple $(s,a,\tilde{s}^\prime)$ from such a ``hybrid" transition kernel, for a given sample $(s, a, s^\prime)$, we take a Bernoulli choice between $s^\prime$ and $s_0\sim\mu_0$ with probability $\gamma$ and $1-\gamma$, respectively, to obtain a state $\tilde{s}^\prime$ that satisfies the requirement. Then, similarly to how we obtain the estimator $\hat{d}^q_{\pi_w}$, we adopt the method in GTD to formulate the following objective
\begin{flalign}\label{eq: 6}
	F(\theta_{\psi})= \mE[\varphi(s^\prime,a^\prime)\delta_{\psi}(\tilde{s}^\prime,a^\prime)]^\top\mE[\varphi(s^\prime,a^\prime)\delta_{\rho}(\tilde{s}^\prime,a^\prime)].
\end{flalign}
We denote the global optimum of $F(\theta_{\psi})$ as $\theta^*_{\psi,w}$, i.e., $F(\theta^*_{\psi,w})=0$, and define the approximation error of estimating $d^\rho_{\pi_w}$ with estimator $\hat{d}^\rho_{\pi_{w}}(s,a,\theta^*_{\rho,w}, \theta^*_{\psi,w}) = \phi(s,a)^\top\theta^*_{\rho,w} \varphi(s,a)^\top\theta^*_{\psi,w}$ as
{\small \begin{flalign}
	\epsilon_{d_\rho} =  \max_w\sqrt{ \mE_{\mcd}\left[ \ltwo{\hat{d}^\rho_{\pi_{w}}(s,a,\theta^*_{\rho,w}, \theta^*_{\psi,w}) - {d}^\rho_{\pi_{w}}(s,a)}^2\right]}.\nonumber
\end{flalign}}
Given mini-batch samples $\mcb_t = \{(s_i, a_i, s^\prime_i)\}_{i=1\cdots N}\sim\mcd_d$, $a^\prime_i\sim\pi_{w_t}(\cdot|s^\prime_i)$ and $\mcb_{t,0} = \{(s_{0,i}, a_{0,i})\}_{i=1\cdots N}\sim\mu_0$, we have the following update for $\theta_\psi$:
\begin{flalign}
	\theta_{\psi,t+1} &= \theta_{\psi,t} + \beta_2 \frac{1}{N}\sum_{i\in\tilde{\mcb}_t} (\varphi^\prime_i - \varphi_i)\varphi^{\prime\top}_iw_{\psi,t},\nonumber\\
	w_{\psi, t+ 1} &= w_{\psi,t} + \beta_2 \frac{1}{N}\sum_{i\in\tilde{\mcb}_t} (\varphi^\prime_i \delta_{\psi,i} - w_{\psi,t}),\label{eq: c2}
\end{flalign}
where $w_{\psi,t}$ is the auxiliary variable that we introduce to avoid the double sampling issue. 
%Note that \cite{morimura2010derivatives} have also presented a way to estimate the gradient of the stationary distribution in an on-policy setting. In comparison, our approach in \cref{eq: c2} handles the off-policy setting.

\textbf{DR-Off-PAC Estimator.} Given parameters $\theta_{\rho,t}$, $\theta_{q,t}$, $\theta_{\psi,t}$ and $\theta_{d_q,t}$, the doubly robust policy gradient can be obtained as follows
% G^i_{\text{DR}}(w_t, \theta_{\rho,t}, \theta_{q,t}, \theta_{\psi,t}, \theta_{d_q,t})
%\begin{flalign}
%	&\quad G^i_{\text{DR}}(w_t) \nonumber\\
%	&= (1-\gamma)\mE_{\pi_w}[\phi_{0,i}^\top\theta_{q,t}\nabla_w\log\pi_w(s_{0,i},a_{0,i}) + x_{0,i}^\top\theta_{d_q,t}]\nonumber\\
%	&+ \psi_{i}^\top\theta_{\psi,t}(r(s_i,a_i,s^\prime_i) - \phi_{i}^\top\theta_{q,t} + \gamma \mE_{\pi_{w_t}} [\phi^{\prime\top}_{i}\theta_{q,t}] ) \nonumber\\
%	& + \phi_{i}^\top\theta_{\rho,t}(- x_{i}^\top\theta_{d_q,t} \nonumber\\
%	&\qquad\quad + \gamma \mE_{\pi_w}[\phi_{i}^\top\theta_{q,t}\nabla_w\log\pi_w(s_{t,i},a_{t,i}) + x_{i}^\top\theta_{d_q,t}]).\label{eq: dr_pg}
%\end{flalign}
\begin{flalign}
&\quad G^i_{\text{DR}}(w_t) \nonumber\\
&= (1-\gamma)\left(\phi_{0,i}^\top\theta_{q,t}\nabla_w\log\pi_w(s_{0,i},a_{0,i}) + x_{0,i}^\top\theta_{d_q,t}\right)\nonumber\\
&+ \psi_{i}^\top\theta_{\psi,t}(r(s_i,a_i,s^\prime_i) - \phi_{i}^\top\theta_{q,t} + \gamma \mE_{\pi_{w_t}} [\phi^{\prime\top}_{i}\theta_{q,t}] ) \nonumber\\
& + \phi_{i}^\top\theta_{\rho,t}(- x_{i}^\top\theta_{d_q,t} \nonumber\\
&\qquad\quad + \gamma \phi_{i}^\top\theta_{q,t}\nabla_w\log\pi_w(s_{t,i},a_{t,i}) + x_{i}^\top\theta_{d_q,t}).\label{eq: dr_pg}
\end{flalign}

\begin{algorithm}[tb]
	\caption{DR-Off-PAC}
	\label{algorithm_drpg}
	\begin{algorithmic}
		\STATE {\bfseries Initialize:} Policy parameter $w_0$, and estimator parameters $\theta_{q,0}$, $\theta_{\rho,0}$, $\theta_{d_q,0}$ and $\theta_{\psi,0}$.
		\FOR{$t=0,\cdots,T-1$}
		\STATE Obtain mini-batch samples $\mcb_t\sim \mcd_d$ and $\mcb_{t,0}\sim\mu_0$
		\STATE \textbf{Critic I: }Update density ratio and value function estimation via \cref{eq: c1}: $\theta_{q,t}, \theta_{\rho,t}\rightarrow \theta_{q,t+1}, \theta_{\rho,t+1}$
		\STATE \textbf{Critic II: }Update derivative of value function estimation via \cref{eq: c3}: $\theta_{d_q,t}\rightarrow \theta_{d_q,t+1}$
		\STATE \textbf{Critic III: }Update derivative of density ratio estimation via \cref{eq: c2}: $\theta_{\psi,t}\rightarrow \theta_{\psi,t+1}$
		\STATE \textbf{Actor: }Update policy parameter via \cref{eq: dr_pg}
		\STATE $w_{t+1} = w_t + \alpha \frac{1}{N}\sum_{i}G^i_{\text{DR}}(w_t)$
		\ENDFOR
		\STATE {\bfseries Output:} $w_{\hat{T}}$ with $\hat{T}$ chosen uniformly in $\{0,\cdots,T-1\}$
	\end{algorithmic}
\end{algorithm}

\textbf{DR-Off-PAC Algorithm.} We now propose a doubly robust off-policy actor-critic (DR-Off-PAC) algorithm as detailed in \Cref{algorithm_drpg}. The stepsizes $\beta_1$, $\beta_2$, $\beta_3$, and $\alpha$ are set to be $\Theta(1)$ to yield a single-timescale update, i.e., all parameters are updated equally fast. At each iteration, critics I, II, and III perform one-step update respectively for parameters $\theta_{q}$, $\theta_{\rho}$, $\theta_{\psi}$, and $\theta_{d_q}$, and then actor performs one-step policy update based on all critics' return. Note that \Cref{algorithm_drpg} is inherently a tri-level optimization process, as the update of $w$ depends on $\theta_{\rho}$, $\theta_{q}$, $\theta_{\psi}$, and $\theta_{d_q}$, in which the update of $\theta_{d_q}$ depends on $\theta_{q}$. Thus the interactions between actor and critics and between critic and critic are more complicated than previous actor-critic algorithms that solve bilevel problems \cite{konda2000actor,bhatnagar2010actor,xu2020improving}.
Due to the single timescale scheme that \Cref{algorithm_drpg} adopts, actor's update is based on inexact estimations of critics, which can significantly affect the overall convergence of the algorithm. Interestingly, as we will show in the next section, \Cref{algorithm_drpg} is guaranteed to converge to the optimal policy, and at the same time attains doubly robust optimality gap with respect to approximation errors.

\section{Convergence Analysis of DR-Off-PAC}

In this section, we establish the local and global convergence rate for DR-Off-PAC in the single-timescale update setting.

\subsection{Local Convergence}

We first state a few standard technical assumptions, which have also been adopted in previous studies \cite{xu2020improving,xu2019two,zhang2020gendice,zhang2020gradientdice,wu2020finite}
\begin{assumption}\label{ass1}
	For any $(s,a)\in\mcs\times\mca$ and $w\in\mR^d$, there exists a constant $C_d>0$ such that $\rho_{\pi_{w}}(s,a)> C_d$.
\end{assumption}

\begin{assumption}\label{ass2}
	For any $(s,a)\in\mcs\times\mca$, there exist positive constants $C_\phi$, $C_\varphi$, $C_\psi$, and $C_x$ such that the following hold: (1) $\ltwo{\phi(s,a)}\leq C_\phi$; (2) $\ltwo{\varphi(s,a)}\leq C_\varphi$; (3) $\ltwo{\psi(s,a)}\leq C_\psi$; (4) $\ltwo{x(s,a)}\leq C_x$.
\end{assumption}

\begin{assumption}\label{ass3}
	The matrices $A = \mE_{\mcd_d\cdot\pi_{w}}[(\phi-\gamma\phi^\prime)\phi^\top]$, $B = \mE_{\tilde{\mcd}_d\cdot\pi_{w}}[(\varphi-\varphi^\prime)\varphi^{\prime\top}]$ and $C=\mE_{\mcd_d\cdot\pi_{w}}[(\gamma x^\prime-x)x^{\top}]$ are nonsingular.
\end{assumption}

\begin{assumption}\label{ass4}
	For any $w, w^\prime\in\mR^d$ and any $(s,a)\in\mcs\times\mca$, there exist positive constants $C_{sc}$, $L_{sc}$, and $L_\pi$ such that the following hold: (1) $\ltwo{\nabla_w\log\pi_{w}(a|s)}\leq C_{sc}$; (2) $\ltwo{\nabla_w\log\pi_{w}(a|s) - \nabla_w\log\pi_{w^\prime}(a|s)}\leq L_{sc}\ltwo{w-w^\prime}$; (3) $\lTV{\pi_w(\cdot|s) - \pi_{w^\prime}(\cdot|s)}\leq L_\pi\ltwo{w-w^\prime}$, where $\lTV{\cdot}$ denotes the total-variation norm.
\end{assumption}

The following theorem characterizes the convergence rate of \Cref{algorithm_drpg}, as well as its doubly robust optimality gap.
\begin{theorem}[Local convergence]\label{thm2}
	Consider the DR-Off-PAC in \Cref{algorithm_drpg}. Suppose \Cref{ass1} - \ref{ass4} hold. Let the stepsize $\alpha, \beta_1, \beta_2, \beta_3=\Theta(1)$. We have
	\begin{flalign*}
		&\mE[\ltwo{\nabla_w J(w_{\hat{T}})}]\nonumber\\
		&\leq \Theta\left(\frac{1}{\sqrt{T}}\right) + \Theta\left(\frac{1}{\sqrt{N}}\right) + \Theta(\epsilon_\rho \epsilon_{d_q} + \epsilon_{d_\rho}\epsilon_q + \epsilon_\rho \epsilon_q).
	\end{flalign*}
\end{theorem}
\Cref{thm2} shows that \Cref{algorithm_drpg} is guaranteed to converge to a first-order stationary point (i.e., locally optimal policy). In particular, the optimality gap (i.e., the overall convergence error) scales as $(\epsilon_\rho \epsilon_{d_q} + \epsilon_{d_\rho}\epsilon_q + \epsilon_\rho \epsilon_q)$. 
%Furthermore, \Cref{thm2} also characterizes the 3-way doubly robustness property for the optimality gap (i.e., the overall convergence error), which scales as $(\epsilon^2_\rho \epsilon^2_{d_q} + \epsilon^2_{d_\rho}\epsilon^2_q + \epsilon^2_\rho \epsilon^2_q)$. Namely, 
%Recall that the approximation errors $\epsilon_\rho$, $\epsilon_q$, $\epsilon_\psi$ and $\epsilon_{d_q}$ are determined by the expressive power of the approximation class, as the error diminishes as the expressive power of approximation function increases. 
Thus, the optimality gap of \Cref{algorithm_drpg} is {\bf 3-way doubly robust} with respect to the function approximation errors, i.e., the optimality gap is small as long as one of the three pairs $(\epsilon_\rho,\epsilon_q)$, $(\epsilon_\rho,\epsilon_{d_\rho})$, $(\epsilon_{q},\epsilon_{d_q})$ is small.

%(1) $\epsilon_\rho$ and $\epsilon_q$ are small, or (2) $\epsilon_\rho$ and $\epsilon_\psi$ are small, or (3) $\epsilon_{q}$ and $\epsilon_{d_q}$ are small.

There are two key differences between the doubly robust properties characterized in \Cref{thm2} and \Cref{thm1}. (a) At the high level, \Cref{thm1} characterizes the doubly robust property only for the policy gradient estimator, and such a property has been characterized in the previous work for other estimators. In contrast, \Cref{thm2} characterizes the doubly robust property for the optimality gap of the overall convergence of an algorithm, which has not been characterized in any of the previous studies. (b) At the more technical level, the estimation error $\varepsilon$ defined in \Cref{thm1} captures both the optimization error $\epsilon_{opt}$ determined by how well we solve the nuisances estimation problem, and the approximation error $\epsilon_{approx}$ determined by the representation power of approximation function classes. Thus, \Cref{thm1} shows that $G_{\text{DR}}(w)$ is doubly robust to the per-iteration estimation errors that depend on both the optimization process and the approximation function class. As a comparison, \Cref{thm2} indicates that the optimality gap of DR-Off-PAC is doubly robust only to approximation errors determined by the approximation function class, which implies that the doubly robust property of the overall convergence of DR-Off-PAC is not affected by the optimization process.

%Note that the doubly robust properties characterized in \Cref{thm2} is different from that of \Cref{thm1}. The estimation error $\varepsilon$ defined in \Cref{thm1} can be decomposed as $\varepsilon = \epsilon_{approx} + \epsilon_{opt}$, where $\epsilon_{opt}$ represents the optimization error controlled by how well we solve the nuisances estimation problem, and $\epsilon_{approx}$ is a fixed error controlled by the representation power of approximation function classes. \Cref{thm1} shows that $G_{\text{DR}}(w)$ is doubly robust to the per-iteration estimation errors that depend on both the optimization process and the approximation function class, while \Cref{thm2} indicate that the overall convergence of DR-Off-PAC is doubly robust to approximation errors that only depend on the approximation function class, which implies that the doubly robust property of the overall convergence of DR-Off-PAC is not affected by the optimization process.

%Even though the per-iteration optimization error $\epsilon_{approx}$ is nonvanishing in the single-timescale update setting, \Cref{thm2} shows that the optimize process does not affects the doubly robustness of the overall convergence of DR-Off-PAC.

Now in order to attain an optimization target accuracy $\epsilon$ (besides the doubly robust optimality gap), we let $T=\Theta(1/\epsilon^2)$ and $B=\Theta(1/\epsilon^2)$. Then \Cref{thm2} indicates that \Cref{algorithm_drpg} converges to an $\epsilon$-accurate stationary point
%, i.e., $\mE[\ltwo{\nabla_w J(w_{\hat{T}})}^2]\leq \epsilon$ 
with the total sample complexity $NT=\Theta(1/\epsilon^4)$. This result outperforms the best known sample complexity of on-policy actor-critic algorithm by an factor of $\mathcal{O}(\log(1/\epsilon))$ in \cite{xu2020improving}. Such an improvement is mainly due to the single-loop structure that we adopt in \Cref{algorithm_drpg}, in which critics inherit the most recently output from the last iteration as actor updates in order to be more sample efficient. But critic in the nested-loop algorithm in \cite{xu2020improving} always restarts from an random initialization after each actor's update, which yields more sample cost. 
%Thus, critic in \cite{xu2020improving} is less efficient and requires more samples than the critics in \Cref{algorithm_drpg} to guarantee the convergence.

\subsection{Global Convergence}

In this subsection, we establish the global convergence guarantee for DR-Off-PAC in \Cref{algorithm_drpg}. We first make the following standard assumption on the Fisher information matrix induced by the policy class $\pi_w$.
\begin{assumption}\label{ass: fisher}
	For all $w\in\mR^d$, the Fisher information matrix induced by policy $\pi_w$ and initial state distribution $\mu_0$ satisfies
	{\small\begin{flalign*}
		F(w) = \mE_{\nu_{\pi_w}}[\nabla_w\log\pi_{w}(a|s)\nabla_w\log\pi_{w}(a|s)^\top]\succeq \lambda_F\cdot I_d,
	\end{flalign*}}
	for some constant $\lambda_F>0$.
\end{assumption}
\Cref{ass: fisher} essentially states that $F(w)$ is well-conditioned. This assumption can be satisfied by some commonly used policy classes. More detailed justification of such an assumption can be referred to Appendix B.2 in \cite{liu2020improved}.

We further define the following {\em compatible function approximation} error as
	{\small \begin{flalign}
		&\epsilon_{compat}\nonumber\\
		&=\max_{w\in\mR^d}\sqrt{\mE_{\nu_{\pi^*}}\left[ (A_{\pi_{w}}(s,a) - (1-\gamma) \chi^{*\top}_{\pi_w}\nabla_w\log\pi_{w}(a|s))^2 \right]},\nonumber
	\end{flalign}}
where $A_{\pi_{w}}(s,a) = Q_{\pi_w}(s,a) - V_{\pi_{w}}(s)$ is the advantage function and $\chi^{*\top}_{\pi_w}=F(w)^{-1}\nabla_w J(w)$. Such an error $\epsilon_{compat}$ captures the approximating error of the advantage function by the score function. It measures the capacity of the policy class $\pi_w$, and takes small or zero values if the expressive power of the policy class is large \cite{wang2019neural,agarwal2019optimality}.

%We further make the following definition before stating our next theorem.
%\begin{defination}\label{compatible_error}
%	The {compatible function approximation} error is defined as
%	{\small \begin{flalign}
%		&\quad \epsilon_{compat}\nonumber\\
%		&=\max_{w\in\mR^d}\sqrt{\mE_{\nu_{\pi^*}}\left[ (A_{\pi_{w}}(s,a) - (1-\gamma) \chi^{*\top}_{\pi_w}\nabla_w\log\pi_{w}(a|s))^2 \right]},\nonumber
%	\end{flalign}}
%	where $A_{\pi_{w}}(s,a) = Q_{\pi_w}(s,a) - V_{\pi_{w}}(s)$ is the advantage function and $\chi^{*\top}_{\pi_w}=F(w)^{-1}\nabla_w J(w)$.
%\end{defination}
%$\epsilon_{compat}$ reflects the error when approximating the advantage function from the score function. It measures the capacity of the policy class $\pi_w$ and is small or zero when the expressive power of the policy class is large \cite{wang2019neural,agarwal2019optimality}. 

The following theorem establishes the global convergence guarantee for \Cref{algorithm_drpg}.

\begin{theorem}[Global convergence]\label{thm3}
	Consider the DR-Off-PAC update in \Cref{algorithm_drpg}. Suppose \Cref{ass1}, \ref{ass2}, \ref{ass3} and \ref{ass: fisher} hold. For the same parameter setting as in \Cref{thm2}, we have
	\begin{flalign*}
		J(\pi^*) - J(w_{\hat{T}})&\leq \frac{\epsilon_{compat}}{1-\gamma} + \Theta\left(\frac{1}{\sqrt{T}}\right) + \Theta\left(\frac{1}{\sqrt{N}}\right) \nonumber\\
		&\quad + \Theta(\epsilon_\rho \epsilon_{d_q} + \epsilon_{d_\rho}\epsilon_q + \epsilon_\rho \epsilon_q)
	\end{flalign*}
\end{theorem}

\Cref{thm3} shows that \Cref{algorithm_drpg} is guaranteed to converge to the global optimum at a sublinear rate, and the optimality gap is bounded by $\Theta(\epsilon_{compat}) + \Theta(\epsilon_\rho \epsilon_{d_q} + \epsilon_{d_\rho}\epsilon_q + \epsilon_\rho \epsilon_q)$. Note that the error term $\Theta(\epsilon_{compat})$ is introduced by the parametrization of policy and thus exists even for exact policy gradient algorithm \cite{liu2020improved,wang2019neural}. The global convergence of DR-Off-PAC in \Cref{thm3} also enjoys doubly robust optimality gap as in \Cref{thm2}. By letting $T=\Theta(1/\epsilon^2)$ and $N=\Theta(1/\epsilon^2)$, \Cref{algorithm_drpg} converges to an $\epsilon$-level global optimum (besides the approximation errors) with a total sample complexity $NT=\Theta(1/\epsilon^4)$. This result matches the global convergence rate of single-loop actor-critic in \cite{xu2020non,fu2020single}.

%we ignore the optimality gap and 

%\textbf{Technical Remark. }Although single-loop actor-critic algorithm has been studied in some previous works \cite{xu2020non,fu2020single,wu2020finite}, all those algorithms adopt bi-level structures, which is simpler than the tri-level structure that we considered in \Cref{algorithm_drpg}. Since the correlations between actor and critics in \Cref{algorithm_drpg} are more complicate than that in \cite{xu2020non,fu2020single,wu2020finite}, the convergence analysis in our work is more challenging than previous works.
\section{Experiments}\label{sc: exp}
We conduct empirical experiments to answer the following two questions: (a) does the overall convergence of DR-Off-PAC doubly robust to function approximation errors as \Cref{thm2} \& \ref{thm3} indicate? (2) how does DR-Off-PAC compare with other off-policy methods? 
\begin{figure}[ht]
	\vskip 0.2in
	\begin{center}
		\centerline{\includegraphics[width=65mm]{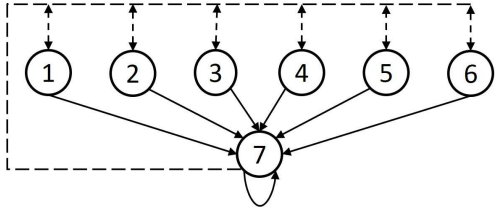}}
		\caption{A variant of Baird's counterexample.}
		\label{fig: baird}
	\end{center}
	\vskip -0.2in
\end{figure}
%\begin{figure*}[ht]  
%	\centering 
%	\subfigure{\includegraphics[width=56mm]{INCFT0.eps}}
%	\subfigure{\includegraphics[width=56mm]{INCFT6.eps}}
%	\subfigure{\includegraphics[width=56mm]{INCFT12.eps}}
%	\caption{\small  Comparison of the convergence rate.}   \label{Experment_1}
%\end{figure*}
%\begin{figure*}[ht]  
%	\centering 
%	\subfigure{\includegraphics[width=56mm]{INCFT0_log.eps}}
%	\subfigure{\includegraphics[width=56mm]{INCFT6_log.eps}}
%	\subfigure{\includegraphics[width=56mm]{INCFT12_log.eps}}
%	\caption{\small  Comparison of the convergence rate.}   \label{Experment_1}
%\end{figure*}

We consider a variant of Baird's counterexample \cite{baird1995residual,sutton2018reinforcement} as shown in \Cref{fig: baird}. There are two actions represented by solid line and dash line, respectively. The {\em solid} action always leads to state 7 and a reward $0$, and the {\em dash} action leads to states 1-6 with equal probability and a reward $+1$. The initial distribution $\mu_0$ chooses all states $s$ with equal probability $\frac{1}{7}$ and the behavior distribution chooses all state-action pairs $(s,a)$ with equal probability $\frac{1}{14}$. We consider two types of one-hot features for estimating the nuisances: complete feature (CFT) and incomplete feature (INCFT), where CFT for each $(s,a)$ lies in $\mR^{14}$ and INCFT for each $(s,a)$ lies in $\mR^{d}$ with $(d<14)$. Note that CFT has large enough expressive power so that the approximation error is zero, while INCFT does not have enough expressive power, and thus introduces non-vanishing approximation errors. In our experiments, we consider fixed learning rates $0.1$, $0.5$, $0.1$, $0.05$, $0.01$ for updating $w$, $\theta_q$, $\theta_\psi$, $\theta_{dq}$, and $\theta_{d\rho}$, respectively, and we set the mini-batch size as $N = 5$. All curves are averaged over 20 independent runs.

{\bf Doubly Robust Optimality Gap:} We first investigate how the function approximation error affects the optimality gap of the overall convergence of DR-Off-PAC. In this experiment, we set the dimension of INCFTs as $0$, which results in trivial critics that always provide constant estimations.
\begin{figure}[ht]
	\vskip -0.08in
	\begin{center}
		\centerline{\includegraphics[width=60mm]{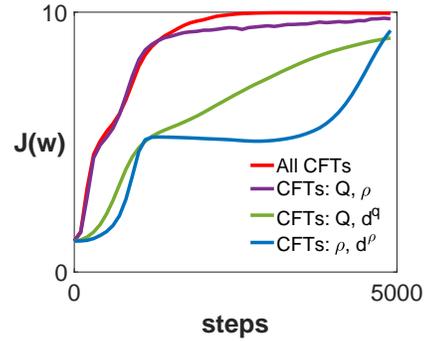}}
		\caption{DR-Off-PAC under difference feature settings.}
		\label{fig: dr-convergence}
	\end{center}
	\vskip -0.2in
\end{figure}
We consider the following four feature settings for critics to estimate the nuisance functions $(Q,\rho, d^q, d^\rho)$: 
{\bf(I)} all nuisances with CFTs.
{\bf(II)} $(Q,\rho)$ with CFTs and $(d^\rho, d^q)$ with INCFTs; 
{\bf(III)} $(Q, d^q)$ with CFTs and $(\rho, d^\rho)$ with INCFTs; 
{\bf(IV)} $(\rho, d^\rho)$ with CFTs and $(Q,d^q)$ with INCFTs.
The results are provided in \Cref{fig: dr-convergence}. We can see that DR-Off-PAC with all nuisances estimated by CFTs (red line) enjoys the fastest convergence speed and smallest optimality gap, and DR-Off-PAC with only two nuisances estimated with CFTs can still converge to the same optimal policy as the red line, validating the doubly robust optimality gap in the overall convergence characterized by \Cref{thm2} and \Cref{thm3}.

{\bf Comparison to AC-DC: }As we have mentioned before, previous provably convergent off-policy actor-critic algorithms introduce an additional critic to correct the distribution mismatch \cite{liu2019off,zhang2019provably}. Such a strategy can be viewed as a special case of DR-Off-PAC when both $\theta_{{d}^q}$ and $\theta_{\psi}$ equal zero. Here we call such a type of algorithms as actor-critic with distribution correction (AC-DC). In this experiment, we set the dimension of INCFTs as 4 and compare the convergence of DR-Off-AC and AC-DC in the settings considered in our previous experiment.
\begin{figure}[ht]
	\vskip -0.08in
	\begin{center}
		\centerline{\includegraphics[width=80mm]{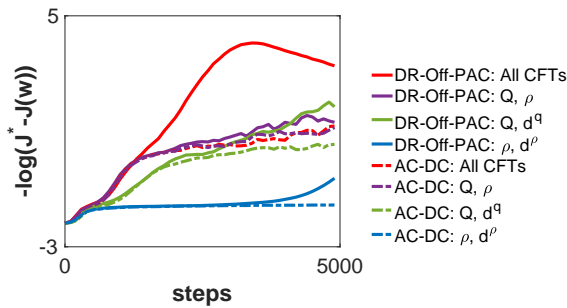}}
		\caption{Comparison between DR-Off-PAC and AC-DC.}
		\label{fig: comparison}
	\end{center}
	\vskip -0.2in
\end{figure}
The learning curves of DR-Off-PAC and AC-DC are reported in \Cref{fig: comparison}. We can see that the overall convergence of DR-Off-PAC (each solid line) outperforms that of AC-DC (dash line with the same color) for all feature settings (where each color corresponds to one feature setting). Specifically, In {\bf(III)} or {\bf(IV)}, when either $Q$ or $\rho$ is estimated with incomplete features, the performance of AC-DC is significantly impeded by the approximation error and thus has lower accuracy, whereas DR-Off-PAC has better convergence performance by mitigating the effect of such approximation errors via the doubly robust property. Interestingly, even in the settings where both $Q$ and $\rho$ are estimated with complete features ($S_1$ and $S_2$) so that AC-DC is expected to achieve zero optimality gap, our DR-Off-PAC still converges faster and more accurately than AC-DC, demonstrating that DR-Off-PAC can improve the convergence of AC-DC even when both $\rho$ and $Q$ are estimated with a complete approximation function class.
%in almost all cases (except $S_4$), our DR-Off-PAC (solid line) solves the task faster and more accurate than AC-DC (dash line) does.

\section{Conclusion}
In this paper, we first develop a new doubly robust policy gradient estimator for an infinite-horizon discounted MDP, and propose new methods to estimate the nuisances in the off-policy setting. Based on such an estimator, we propose a doubly robust off-policy algorithm called DR-Off-PAC for solving the policy optimization problem. We further study the finite-time convergence of DR-Off-PAC under the single timescale update setting. We show that DR-Off-PAC provably converges to the optimal policy, with the optimality gap being doubly robust to approximation errors that depend only on the expressive power of function classes. For future work, it is interesting to incorporate variance reduction technique \cite{xu2020reanalysis,cutkosky2019momentum} to DR-Off-PAC to improve its convergence performance.

%To the best of our knowledge, it is the first time that the doubly robust optimality gap of the overall convergence is characterized for off-policy actor-critic algorithms. 

\section*{Acknowledgements}
The work of T. Xu and Y. Liang was supported in part by the U.S. National Science Foundation under the grants CCF-1761506 and CCF-1900145. Z. Wang acknowledges National Science Foundation (Awards 2048075, 2008827, 2015568, 1934931), Simons Institute (Theory of Reinforcement Learning), Amazon, J.P. Morgan, and Two Sigma for their supports. Z. Yang acknowledges Simons Institute (Theory of Reinforcement Learning).


\bibliography{ref}
\bibliographystyle{icml2021}

\onecolumn
\newpage
\appendix
\noindent {\Large \textbf{Supplementary Materials}}
\section{Derivation of Doubly Robust Policy Gradient Estimator}
In this section, we introduce how to derive the doubly robust policy gradient $G_{\text{DR}}(w)$ in \cref{dr_pg}

Consider the setting of off-policy sampling specified in \Cref{sc: bkgrd}. Note that $J(w)$ has the following alternative form:
\begin{flalign}\label{eq: 70}
J(w) = & (1-\gamma)\mE_{\mu_0}\big[V_{\pi_w}(s_0)\big] + \mE_d\big[\rho_{\pi_w}(s,a)(r(s,a,s^\prime) - Q_{\pi_w}(s,a) + \gamma \mE\big[V_{\pi_w}(s^\prime)|s,a\big])\big],
\end{flalign}
where $\rho_{\pi_w}(s,a)=\nu_{\pi_w}(s,a)/d(s,a)$ denotes the {\em distribution correction ratio}. With a sample $(s, a, r, s^\prime,a^\prime)\sim \mcd_d\cdot\pi_w(\cdot)$ and a sample $(s_{0}, a_{0})\sim \mu_0\cdot\pi_{w}(\cdot)$, we can formulate the following stochastic estimator of $J(w)$:
\begin{flalign}\label{eq: 71}
\hat{J}(w) =  \underbrace{(1-\gamma)V_{\pi_w}(s_{0})}_{\text{unbiased estimator}} + \underbrace{\rho_{\pi_w}(s,a)(r(s,a,s^\prime) - Q_{\pi_w}(s,a) + \gamma V_{\pi_w}(s^\prime))}_{\text{baseline}}.
\end{flalign}
Note that the first term in \cref{eq: 71} is an unbiased estimator of $J(w)$ and the second term in \cref{eq: 71} is the baseline that can help to reduce the variance \cite{jiang2016doubly,huang2020importance}. Note that if we replace the value functions $V_{\pi_w}$, $Q_{\pi_w}$ and the density ratio $\rho_{\pi_{w}}$ with their estimators $\hat{V}_{\pi_w}$ $\hat{Q}_{\pi_w}$, and $\hat{\rho}_{\pi_{w}}$, respectively, we can obtain a doubly robust bias reduced value function estimator \cite{tang2019doubly}. Next, we take the derivative of $\hat{J}(w)$ to obtain an unbiased estimator of $\nabla {J}(w)$ which takes the following form:
\begin{flalign}\label{eq: 72}
\nabla_w\hat{J}(w) &= (1-\gamma)d^v_{\pi_w}(s_{0}) + d^\rho_{\pi_w}(s,a)(r(s,a,s^\prime) - Q_{\pi_w}(s,a) + \gamma V_{\pi_w}(s^\prime)) \nonumber\\
&\quad + \rho_{\pi_w}(s,a)(- d^q_{\pi_w}(s,a) + \gamma d^v_{\pi_w}(s^\prime))\nonumber\\
&= (1-\gamma)\mE_{\pi_w}[{Q}_{\pi_w}(s_{0},a_{0})\nabla_w\log\pi_w(s_{0},a_{0}) + {d}^q_{\pi_w}(s_{0},a_{0})] \nonumber\\
&\quad + d^\rho_{\pi_w}(s,a)(r(s,a,s^\prime) - Q_{\pi_w}(s,a) + \gamma \mE_{\pi_w}[Q_{\pi_w}(s^\prime,a^\prime)]) \nonumber\\
&\quad + \rho_{\pi_w}(s,a)(- d^q_{\pi_w}(s,a) + \gamma \mE_{\pi_w}[{Q}_{\pi_w}(s^\prime_{i},a^\prime_{i})\nabla_w\log\pi_w(s^\prime_{i},a^\prime_{i}) + {d}^q_{\pi_w}(s^\prime_{i},a^\prime_{i})] ),
\end{flalign}
where $d^v_{\pi_w}$, $d^q_{\pi_w}$, $d^\rho_{\pi_w}$ denote $\nabla_w V_{\pi_w}$, $\nabla_w Q_{\pi_w}$, $\nabla_w \rho_{\pi_w}$, respectively. Given samples $s_{0}\sim \mu_0(\cdot)$, $a_{0}\sim \pi_w(\cdot|s_{0})$ and $(s,a,r,s^\prime)\sim \mcd_d$, $a^\prime\sim\pi_w(\cdot|s^\prime)$, and replace ${Q}_{\pi_w}$, ${\rho}_{\pi_w}$, ${d}^\rho_{\pi_w}$ and ${d}^q_{\pi_w}$ with estimators $\hat{Q}_{\pi_w}$, $\hat{\rho}_{\pi_w}$, $\hat{d}^\rho_{\pi_w}$ and $\hat{d}^q_{\pi_w}$, respectively, we can obtain the following doubly robust estimator $G_{\text{DR}}(w)$:
\begin{flalign}\label{eq: 73}
	G_{\text{DR}}(w) &= (1-\gamma)\Big(\hat{Q}_{\pi_w}(s_{0},a_{0})\nabla_w\log\pi_w(a_{0}|s_{0}) + \hat{d}^q_{\pi_w}(s_{0},a_{0})\Big) + \hat{d}^\rho_{\pi_w}(s,a)\left(r(s,a,s^\prime) - \hat{Q}_{\pi_w}(s,a) + \gamma \hat{Q}_{\pi_w}(s^\prime,a^\prime)\right) \nonumber\\
	&\quad + \hat{\rho}_{\pi_w}(s,a)\Big[- \hat{d}^q_{\pi_w}(s,a) + \gamma \Big( \hat{Q}_{\pi_w}(s^\prime,a^\prime)\nabla_w\log\pi_w(a^\prime|s^\prime) + \hat{d}^q_{\pi_w}(s^\prime,a^\prime) \Big)\Big].
\end{flalign}

{\bf Connection with other off-policy gradient estimators: }Our doubly robust estimator $G_{\text{DR}}$ can recover a number of existing off-policy policy gradient estimators as special cases by deactivating certain estimators, i.e., letting those estimators be zero.

(1) Deactivating $\hat{d}^q_{\pi_w}$ and $\hat{d}^\rho_{\pi_w}$: In this case, $G_{\text{DR}}(w)$ takes the following form
\begin{flalign}\label{eq: 74}
	{G}^I_{\text{DR}}(w) &= (1-\gamma)\hat{Q}_{\pi_w}(s_{0},a_{0})\nabla_w\log\pi_w(a_{0}|s_{0}) + \gamma \hat{\rho}_{\pi_w}(s,a) \Big( \hat{Q}_{\pi_w}(s^\prime,a^\prime)\nabla_w\log\pi_w(a^\prime|s^\prime) + \hat{d}^q_{\pi_w}(s^\prime,a^\prime) \Big)\nonumber\\
	&=\hat{\rho}_{\pi_w}(s,a)\mE_{s^\prime\sim\tilde{\msP}(\cdot|s,a),\tilde{a}^\prime\sim\pi_{w}(\cdot|\tilde{s}^\prime)}\left[\hat{Q}_{\pi_w}(\tilde{s}^\prime,\tilde{a}^\prime)\nabla_w\log\pi_w(\tilde{a}^\prime|\tilde{s}^\prime)\right],
\end{flalign}
where $(\tilde{s}^\prime,\tilde{a}^\prime)$ is generated using the method in the discussion of Critic III in \Cref{subsc: nuisances}. Note that the policy gradient $\nabla_wJ(w)$ has the following equivalent form
\begin{flalign*}
	\nabla_wJ(w) = \mE_{(s,a)\sim\mcd}\left[ {\rho}_{\pi_w}(s,a)\mE_{s^\prime\sim\tilde{\msP}(\cdot|s,a),\tilde{a}^\prime\sim\pi_{w}(\cdot|\tilde{s}^\prime)}\left[{Q}_{\pi_w}(\tilde{s}^\prime,\tilde{a}^\prime)\nabla_w\log\pi_w(\tilde{a}^\prime|\tilde{s}^\prime)\big|(s,a)\right] \right].
\end{flalign*}
Thus, ${G}^I_{\text{DR}}(w)$ can be viewed as an off-policy policy gradient estimator with only approximations of $\rho_{\pi_{w}}$ and $Q_{\pi_w}$. Such an estimator has been adopted in the previous studies of provably convergent off-policy actor-critic \cite{zhang2019provably,liu2019off}, which is referred as AC-DC in our experiment in \Cref{sc: exp}.
Such an estimator has also been adopted by many off-policy actor-critic algorithms such as ACE \cite{imani2018off}, Geoff-PAC \cite{zhang2019generalized}, OPPOSD \cite{liu2019off} and COF-PAC \cite{zhang2019provably}.

(2) Deactivating $\rho_{\pi_w}$ and $\hat{d}^\rho_{\pi_w}$: In this case, $G_{\text{DR}}(w)$ has the following form
\begin{flalign}\label{eq: 75}
	 {G}^{II}_{\text{DR}}(w) = (1-\gamma)\Big(\hat{Q}_{\pi_w}(s_{0},a_{0})\nabla_w\log\pi_w(a_{0}|s_{0}) + \hat{d}^q_{\pi_w}(s_{0},a_{0})\Big).
\end{flalign}
Such an estimator ${G}^{II}_{\text{DR}}(w)$ can be viewed as the one adopted by off-policy DPG/DDPG \cite{silver2014deterministic,lillicrap2016continuous} when the policy $\pi_w$ converges to a deterministic policy.

(3) Deactivting $\hat{Q}_{\pi_w}$ and $\hat{d}^q_{\pi_w}$: In this case, $G_{\text{DR}}(w)$ has the following form
\begin{flalign}\label{eq: 76}
	{G}^{III}_{\text{DR}}(w) = \hat{d}^\rho_{\pi_w}(s,a)r(s,a,s^\prime).
\end{flalign}
This off-policy policy gradient estimator has been adopted in \cite{morimura2010derivatives} for averaged MDP setting.
%The advantage of this estimator is the formulation in \cref{eq: 76} does not depend on the value function estimator, thus is free from the bias-variance trade-off in the on-policy setting.

\section{Proof of \Cref{thm1}}

	Without specification, the expectation is taken with respect to the randomness of samples $(s, a, r(s,a,s^\prime), s^\prime)$ and $s_0$, in which $(s,a)\sim d(\cdot)$, $s^\prime\sim\msP(\cdot|s,a)$ and $s_0\sim \mu_0(\cdot)$. If $a^\prime$ or $a_0$ appears, then the expectation is taken with respect to the the policy i.e., $a^\prime\sim\pi_w(\cdot|s^\prime)$ and $a_0\sim\pi_w(\cdot|s_0)$. We compute the bias of $G_{\text{DR}}(w)$ as follows.
	\begin{flalign}
		&\mE[G_{\text{DR}}(w)] - \nabla_w J(w)\nonumber\\
		&=(1-\gamma)\mE[\hat{d}^v_{\pi_w}(s_0)] + \mE[\hat{d}^\rho_{\pi_w}(s,a)(r(s,a,s^\prime) - \hat{Q}_{\pi_w}(s,a) + \gamma \hat{V}_{\pi_w}(s^\prime))] + \mE[\hat{\rho}_{\pi_w}(s,a)(- \hat{d}^q_{\pi_w}(s,a) + \gamma \hat{d}^v_{\pi_w}(s^\prime))]\nonumber\\
		&\quad -(1-\gamma)\mE[{d}^v_{\pi_w}(s_0)] - \mE[{d}^\rho_{\pi_w}(s,a)(r(s,a,s^\prime) - {Q}_{\pi_w}(s,a) + \gamma {V}_{\pi_w}(s^\prime))] - \mE[{\rho}_{\pi_w}(s,a)(- {d}^q_{\pi_w}(s,a) + \gamma {d}^v_{\pi_w}(s^\prime))]\nonumber\\
		&= \mE[(\hat{\rho}_{\pi_w}(s,a) - {\rho}_{\pi_w}(s,a))(-\hat{d}^q_{\pi_w}(s,a) + {d}^q_{\pi_w}(s,a))] + \mE[(-\hat{d}^\rho_{\pi_w}(s,a) + {d}^\rho_{\pi_w}(s,a))(-\hat{Q}_{\pi_w}(s,a) + {Q}_{\pi_w}(s,a))]\nonumber\\
		&\quad + \gamma\mE[(\hat{\rho}_{\pi_w}(s,a) - {\rho}_{\pi_w}(s,a))(\hat{d}^v_{\pi_w}(s) - {d}^v_{\pi_w}(s))] + \gamma\mE[(\hat{d}^\rho_{\pi_w}(s,a) - {d}^\rho_{\pi_w}(s,a))(\hat{V}_{\pi_w}(s^\prime) - {V}_{\pi_w}(s^\prime))]\nonumber\\
		&\quad + \mE[{d}^\rho_{\pi_w}(s,a)(-\hat{Q}_{\pi_w}(s,a) + Q_{\pi_w}(s,a))] + \mE[\rho_{\pi_w}(s,a)(-\hat{d}^q_{\pi_w}(s,a) + {d}^q_{\pi_w}(s,a))] \nonumber\\
		&\quad + \gamma \mE[{d}^\rho_{\pi_w}(s,a)(\hat{V}_{\pi_w}(s^\prime) - V_{\pi_w}(s^\prime))] + \gamma \mE[\rho_{\pi_w}(s,a)(\hat{d}^v_{\pi_w}(s^\prime) - {d}^v_{\pi_w}(s^\prime))] + (1-\gamma)\mE[\hat{d}^v_{\pi_w}(s_0) - {d}^v_{\pi_w}(s_0)]\nonumber\\
		&\quad + \mE[(\hat{\rho}_{\pi_w}(s,a) - {\rho}_{\pi_w}(s,a))(-d^q_{\pi_w}(s,a) + \gamma d^v_{\pi_w}(s^\prime))] \nonumber\\
		&\quad + \mE[(\hat{d}^\rho_{\pi_w}(s,a) - {d}^\rho_{\pi_w}(s,a))(r(s,a,s^\prime) - Q_{\pi_w}(s,a)+\gamma V_{\pi_w}(s^\prime))]\nonumber\\
		&= -\mE[\varepsilon_{\rho}(s,a)\varepsilon_{d^q}(s,a)] - \mE[\varepsilon_{d^\rho}(s,a)\varepsilon_{q}(s,a)] + \gamma\mE[\varepsilon_{\rho}(s,a)\varepsilon_{d^v}(s^\prime)] + \gamma\mE[\varepsilon_{\rho}(s,a)\varepsilon_{v}(s^\prime)]\nonumber\\
		&\quad + S_1 + S_2 + S_3,\label{eq: 3}
	\end{flalign}
	where
	\begin{flalign*}
		S_1 &= \mE[{d}^\rho_{\pi_w}(s,a)(-\hat{Q}_{\pi_w}(s,a) + Q_{\pi_w}(s,a))] + \mE[\rho_{\pi_w}(s,a)(-\hat{d}^q_{\pi_w}(s,a) + {d}^q_{\pi_w}(s,a))] \nonumber\\
		&\quad + \gamma \mE[{d}^\rho_{\pi_w}(s,a)(\hat{V}_{\pi_w}(s^\prime) - V_{\pi_w}(s^\prime))] + \gamma \mE[\rho_{\pi_w}(s,a)(\hat{d}^v_{\pi_w}(s^\prime) - {d}^v_{\pi_w}(s^\prime))] + (1-\gamma)\mE[\hat{d}^v_{\pi_w}(s_0) - {d}^v_{\pi_w}(s_0)],\nonumber\\
		S_2 &= \mE[(\hat{\rho}_{\pi_w}(s,a) - {\rho}_{\pi_w}(s,a))(-d^q_{\pi_w}(s,a) + \gamma d^v_{\pi_w}(s^\prime))],\nonumber\\
		S_3 & = \mE[(\hat{d}^\rho_{\pi_w}(s,a) - {d}^\rho_{\pi_w}(s,a))(r(s,a,s^\prime) - Q_{\pi_w}(s,a)+\gamma V_{\pi_w}(s^\prime))].
	\end{flalign*}
	We then proceed to show that $S_1=S_2=S_3=0$. First consider $S_1$. Following from the definitions of $\hat{d}^v_{\pi_w}$ and ${d}^v_{\pi_w}$, we have
	\begin{flalign}
		S_1 &= \mE[{d}^\rho_{\pi_w}(s,a)(-\hat{Q}_{\pi_w}(s,a) + Q_{\pi_w}(s,a))] + \mE[\rho_{\pi_w}(s,a)(-\hat{d}^q_{\pi_w}(s,a) + {d}^q_{\pi_w}(s,a))] + \gamma \mE[{d}^\rho_{\pi_w}(s,a)(\hat{V}_{\pi_w}(s^\prime) - V_{\pi_w}(s^\prime))] \nonumber\\
		&\quad + \gamma \mE[\rho_{\pi_w}(s,a)( \mE[\hat{Q}_{\pi_w}(s^\prime,a^\prime)\nabla_w\log\pi_w(s^\prime,a^\prime) + \hat{d}^q_{\pi_w}(s^\prime)|s^\prime] - \mE[{Q}_{\pi_w}(s^\prime,a^\prime)\nabla_w\log\pi_w(s^\prime,a^\prime) + {d}^q_{\pi_w}(s^\prime)|s^\prime] )] \nonumber\\
		&\quad + (1-\gamma)\mE[ \mE[\hat{Q}_{\pi_w}(s_0,a_0)\nabla_w\log\pi_w(s_0,a_0) + \hat{d}^q_{\pi_w}(s_0)|s_0] - \mE[{Q}_{\pi_w}(s_0,a_0)\nabla_w\log\pi_w(s_0,a_0) + {d}^q_{\pi_w}(s_0)|s_0]]\nonumber\\
		&=\mE[\rho_{\pi_w}(s,a)(d^q_{\pi_w}(s,a)-\hat{d}^q_{\pi_w}(s,a))] - \gamma\mE[\mE[d^q_{\pi_w}(s^\prime,a^\prime) - \hat{d}^q_{\pi_w}(s^\prime,a^\prime)|(s,a)\sim \nu_{\pi_w}(s,a)]] \nonumber\\
		&\quad - (1-\gamma)\mE[d^q_{\pi_w}(s_0,a_0) - \hat{d}^q_{\pi_w}(s_0,a_0)] \nonumber\\
		&\quad + \mE[d^\rho_{\pi_w}(s,a)(Q_{\pi_w}(s,a) - \hat{Q}_{\pi_w}(s,a))] - \gamma\mE[d^\rho_{\pi_w}(s,a)(V_{\pi_w}(s^\prime) - \hat{V}_{\pi_w}(s^\prime))]\nonumber\\
		&\quad -\gamma\mE[\rho_{\pi_w}(s,a)\mE[(Q_{\pi_w}(s^\prime,a^\prime) - \hat{Q}_{\pi_w}(s^\prime,a^\prime))\nabla_w\log\pi_w(s^\prime,a^\prime)|s,a]] \nonumber\\
		&\quad - (1-\gamma)\mE[(Q_{\pi_w}(s_0,a_0) - \hat{Q}_{\pi_w}(s_0,a_0))\nabla_w\log\pi_w(s_0,a_0)]\label{eq: 1}.
	\end{flalign}
	For the first three terms in \cref{eq: 1}, we have
	\begin{flalign}
	    \mE&[\rho_{\pi_w}(s,a)(d^q_{\pi_w}(s,a)-\hat{d}^q_{\pi_w}(s,a))] - \gamma\mE[\mE[d^q_{\pi_w}(s^\prime,a^\prime) - \hat{d}^q_{\pi_w}(s^\prime,a^\prime)|(s,a)\sim \nu_{\pi_w}(s,a)]] \nonumber\\
	    &\qquad\qquad - (1-\gamma)\mE[d^q_{\pi_w}(s_0,a_0) - \hat{d}^q_{\pi_w}(s_0,a_0)]\nonumber\\
		&= \mE_{(s,a)\sim\nu_{\pi_w}}[d^q_{\pi_w}(s,a)-\hat{d}^q_{\pi_w}(s,a)] - \gamma\mE_{(s,a)\sim\nu_{\pi_w}}[\mE[d^q_{\pi_w}(s^\prime,a^\prime) - \hat{d}^q_{\pi_w}(s^\prime,a^\prime)|s,a]] \nonumber\\
		&\quad - (1-\gamma)\mE[d^q_{\pi_w}(s_0,a_0) - \hat{d}^q_{\pi_w}(s_0,a_0)]\nonumber\\
		&\overset{(i)}{=}\mE_{(s,a)\sim\nu_{\pi_w}}[d^q_{\pi_w}(s,a)-\hat{d}^q_{\pi_w}(s,a)] - \mE_{(s,a)\sim\nu_{\pi_w}}[\mE[d^q_{\pi_w}(s^\prime,a^\prime)-\hat{d}^q_{\pi_w}(s^\prime,a^\prime)|s^\prime\sim \tilde{\msP}(\cdot|s,a), a^\prime\sim\pi_w(\cdot|s^\prime)]]\nonumber\\
		&\overset{(ii)}{=}\mE_{(s,a)\sim\nu_{\pi_w}}[d^q_{\pi_w}(s,a)-\hat{d}^q_{\pi_w}(s,a)] - \mE_{(s^\prime,a^\prime)\sim\nu_{\pi_w}}[d^q_{\pi_w}(s^\prime,a^\prime)-\hat{d}^q_{\pi_w}(s^\prime,a^\prime)]\nonumber\\
		&=0,\nonumber
	\end{flalign}
	where $(i)$ follows from the definition $\tilde{\msP}(\cdot|s,a) = \gamma \msP(\cdot|s,a) + (1-\gamma)\mu_0(\cdot)$, and $(ii)$ follows from the fact that $\nu_{\pi_w}$ is the stationary distribution of MDP with the transition kernel $\tilde{\msP}(\cdot|s,a)$ and policy $\pi_w$, i.e., $\pi_w(a^\prime|s^\prime)\sum_{(s,a)}\nu_{\pi_w}(s,a)\tilde{\msP}(s^\prime|s,a)=\nu_{\pi_w}(s^\prime,a^\prime)$. For the last four terms in \cref{eq: 1}, note that for any function $f(s,a)$, we have the following holds
	\begin{flalign}
		\mE&[d^\rho_{\pi_w}(s,a)f(s,a)] - \gamma\mE[d^\rho_{\pi_w}(s,a)\mE[f(s^\prime,a^\prime)|s^\prime]] -\gamma\mE[\rho_{\pi_w}(s,a)\mE[ f(s^\prime,a^\prime) \nabla_w\log\pi_w(s^\prime,a^\prime)|s,a]] \nonumber\\
		&\quad\quad - (1-\gamma)\mE[f(s_0,a_0)\nabla_w\log\pi_w(s_0,a_0)]\nonumber\\
		&=\nabla_w \mE[\rho_{\pi_w}(s,a)f(s,a)] - \gamma \nabla_w\mE[\rho_{\pi_w}(s,a)\mE[f(s^\prime,a^\prime)|s^\prime]] - (1-\gamma)\nabla_w\mE[\mE[f(s_0,a_0)|s_0]]\nonumber\\
		&=\nabla_w(\mE[\rho_{\pi_w}(s,a)f(s,a)] - \gamma \mE[\rho_{\pi_w}(s,a)\mE[f(s^\prime,a^\prime)|s^\prime]] - (1-\gamma)\nabla_w\mE[\mE[f(s_0,a_0)|s_0]] )\nonumber\\
		&\overset{(i)}{=}\nabla_w (\mE_{(s,a)\sim\nu_{\pi_w}}[f(s,a)] - \mE_{(s^\prime,a^\prime)\sim\nu_{\pi_w}}[f(s^\prime,a^\prime)])\nonumber\\
		&=0,\nonumber
	\end{flalign}
	where $(i)$ follows from the reasons similar to how we proceed \cref{eq: 1}. Letting $f(s,a) = Q_{\pi_w}(s,a) - \hat{Q}_{\pi_w}(s,a)$, we can then conclude that the summation of the last four terms in \cref{eq: 1} is $0$, which implies $S_1=0$.
	
	We then consider the term $S_2$. Note that for any function $f(s,a)$, we have
	\begin{flalign}
		\mE&[f(s,a)(-d^q_{\pi_w}(s,a) + \gamma d^v_{\pi_w}(s^\prime))]\nonumber\\
		&=\nabla_w\mE[f(s,a)(r(s,a,s^\prime) + \gamma V_{\pi_w}(s^\prime) - Q_{\pi_w}(s,a))]\nonumber\\
		&=\nabla_w\mE[f(s,a)(\mE[r(s,a,s^\prime)|s] + \gamma \mE[V_{\pi_w}(s^\prime)|s,a] -Q_{\pi_w}(s,a) )]\nonumber\\
		&=0.
	\end{flalign}
	Letting $f(s,a) = \hat{\rho}_{\pi_w}(s,a) - {\rho}_{\pi_w}(s,a)$, we can then conclude that $S_2=0$. To consider $S_3$, we proceed as follows:
	\begin{flalign*}
		S_3 & = \mE[(\hat{d}^\rho_{\pi_w}(s,a) - {d}^\rho_{\pi_w}(s,a))(r(s,a,s^\prime) - Q_{\pi_w}(s,a)+\gamma V_{\pi_w}(s^\prime))]\nonumber\\
		&= \mE[(\hat{d}^\rho_{\pi_w}(s,a) - {d}^\rho_{\pi_w}(s,a))(\mE[r(s,a,s^\prime)|s,a] - Q_{\pi_w}(s,a)+\gamma \mE[V_{\pi_w}(s^\prime)|s,a])]\nonumber\\
		&=0.
	\end{flalign*}
	Since we have shown that $S_1=S_2=S_3=0$, \cref{eq: 3} becomes
	\begin{flalign}
		\mE&[G_{\text{DR}}(w)] - \nabla_w J(w)\nonumber\\
		&= -\mE[\varepsilon_{\rho}(s,a)\varepsilon_{d^q}(s,a)] - \mE[\varepsilon_{d^\rho}(s,a)\varepsilon_{q}(s,a)] + \gamma\mE[\varepsilon_{\rho}(s,a)\varepsilon_{d^v}(s^\prime)] + \gamma\mE[\varepsilon_{\rho}(s,a)\varepsilon_{v}(s^\prime)],\nonumber
	\end{flalign}
	which completes the proof.

\section{Supporting Lemmas for \Cref{thm2}}\label{sc: lemma_pf_thm2}
In order to develop the property for Critic I's update in \cref{eq: c1}, we first introduce the following definitions.

Given a sample from mini-batch $(s_i, a_i, r_i, s^\prime_i)$, $a^\prime_i\sim\pi_{w_t}(\cdot|s^\prime_i)$ and a sample $s_{0,i}\sim \mu_0$, we define the following matrix $M_{i,t}\in\mR^{(2d_1+1)\times (2d_1+1)}$ and vector $m_{i,t}\in \mR^{2d_1+1\times 1}$
\begin{flalign}
M_{i,t} &= \left[\begin{array}{ccc}
-\phi_i^\top\phi_i & -(\phi_i-\gamma\phi^\prime_i)\phi^\top_i & 0 \\
\phi_i(\phi^\top_i-\gamma\phi^{\prime\top}_i) & 0 & -\phi_i \\
0 & \phi^\top_i & -1
\end{array}
\right],\qquad
m_{i,t} = \left[\begin{array}{c}
(1-\gamma)\phi_{0,i}\\
-r_i\phi_i\\
-1
\end{array}
\right].\label{eq: 66}
\end{flalign}
Moreover, consider the matrix $M_{i,t}$. We have the following holds
\begin{flalign}\label{eq: 77}
\lF{M_{i,t}}^2 &= \lF{\phi_i^\top\phi_i}^2 + 2\lF{(\phi_i-\gamma\phi^\prime_i)\phi^\top_i}^2 + 2\ltwo{\phi_i}^2 + 1\nonumber\\
&\leq C^4_\phi + 2 (1+\gamma)^2C^4_\phi + 2C^2_\phi + 1,
\end{flalign}
which implies $\lF{M_{i,t}}\leq C_M$, where $C_M = \sqrt{9C^4_\phi + 2C^2_\phi + 1}$. For the vector $m_{i,t}$, we have
\begin{flalign}\label{eq: 78}
\ltwo{m_{i,t}}^2 \leq (1-\gamma)^2\ltwo{\phi_i}^2 + r^2_{\max}\ltwo{\phi_i}^2 + 1  \leq [(1-\gamma)^2 + r^2_{\max}]C^2_\phi + 1,
\end{flalign}
which implies $\ltwo{m_{i,t}}\leq C_m$, where $C_m = \sqrt{(1+r^2_{\max})C^2_\phi + 1}$.

We define the semi-stochastic-gradient as $g_{i,t}(\kappa) = M_{i,t}\kappa + m_{i,t}$. Then the iteration in \cref{eq: c1} can be rewritten as
\begin{flalign}
\kappa_{t+1} = \kappa_t + \beta_1\hat{g}_t(\kappa_t),\label{eq: 7}
\end{flalign}
where $\hat{g}_t(\kappa_t) = \frac{1}{N}\sum_ig_{i,t}(\kappa_t)$. 
We also define $M_t = \mE_i[M_{i,t}]$ and $m_t = \mE_i[m_{i,t}]$, i.e., 
\begin{flalign*}
M_t &= \left[\begin{array}{ccc}
-\mE_{\mcd_d\cdot\pi_{w_t}}[\phi^\top\phi] & -\mE_{\mcd_d\cdot\pi_{w_t}}[(\phi-\gamma\phi^\prime)\phi^\top] & 0 \\
\mE_{\mcd_d\cdot\pi_{w_t}}[\phi(\phi^\top-\gamma\phi^{\prime\top})] & 0 & -\mE_{\mcd_d\cdot\pi_{w_t}}[\phi] \\
0 & \mE_{\mcd_d\cdot\pi_{w_t}}[\phi^\top] & -1
\end{array}
\right],\qquad
m_t = \left[\begin{array}{c}
(1-\gamma)\mE_{\mu_0\cdot\pi_{w_t}}[\phi]\\
-\mE_{\mcd}[r\phi]\\
-1
\end{array}
\right],
\end{flalign*}
and semi-gradient $g_t(\kappa) = M_t\kappa + m_t$. We further define the fixed point of the iteration \cref{eq: 7} as
\begin{flalign}
\kappa^*_t=M^{-1}_tm_t = [\theta^{*\top}_{q,t}, \theta^{*\top}_{\rho,t}, \eta^*_t]^\top.\label{eq: 14}
\end{flalign}

\begin{lemma}\label{lemma1}
	Consider one step update in \cref{eq: c1}. Define $\kappa_t = [\theta^\top_{q,t}, \theta^\top_{\rho,t}, \eta_t]^\top$ and $\kappa^*_t$ in \cref{eq: 14}. Let $\beta_1\leq \min\{ 1/\lambda_M, \lambda_m/52C^2_M\}$, we have
	\begin{flalign*}
		\mE\left[\ltwo{\kappa_{t+1} - \kappa^*_t}^2|\mf_t\right] \leq \left(1-\frac{1}{2}\beta_1\lambda_M \right)\ltwo{\kappa_t - \kappa^*_t}^2 + \frac{C_1}{N},
	\end{flalign*}
	where $C_1=24\beta^2_1(C^2_MR^2_\kappa + C^2_m)$.
\end{lemma}
\begin{proof}
	It has been shown in \cite{zhang2020gradientdice} that $M_t$ is a Hurwitz matrix which satisfies $(\kappa_t-\kappa^*_t)^\top M_t(\kappa_t-\kappa^*_t)\leq -{\lambda_M}\ltwo{\kappa_t-\kappa^*_t}^2$, where $\lambda_M>0$ is a constant.
	We then proceed as follows:
	\begin{flalign}
		\ltwo{\kappa_{t+1}-\kappa^*_t}^2&=\ltwo{\kappa_t + \beta_1\hat{g}_t(\kappa_t)-\kappa^*_t}^2\nonumber\\
		&=\ltwo{\kappa_t - \kappa^*_t}^2 + \beta_1 \hat{g}_t(\kappa_t)^\top (\kappa_t-\kappa^*_t) + \beta^2_1\ltwo{\hat{g}_t(\kappa_t)}^2\nonumber\\
		&=\ltwo{\kappa_t - \kappa^*_t}^2 + \beta_1 {g}_t(\kappa_t)^\top (\kappa_t-\kappa^*_t) + \beta_1 (\hat{g}_t(\kappa_t)-{g}_t(\kappa_t))^\top (\kappa_t-\kappa^*_t) + \beta^2_1\ltwo{\hat{g}_t(\kappa_t) - g_t(\kappa_t) + g_t(\kappa_t)}^2\nonumber\\
		&\overset{(i)}{\leq} (1-\beta_1\lambda_M)\ltwo{\kappa_t - \kappa^*_t}^2 + \beta_1 (\hat{g}_t(\kappa_t)-{g}_t(\kappa_t))^\top (\kappa_t-\kappa^*_t) + 2\beta^2_1\ltwo{\hat{g}_t(\kappa_t) - g_t(\kappa_t)}^2 + 2\beta^2_1\ltwo{g_t(\kappa_t)}^2\nonumber\\
		&\overset{(ii)}{\leq} (1-\beta_1\lambda_M + 2\beta^2_1C^2_M)\ltwo{\kappa_t - \kappa^*_t}^2 + \beta_1 (\hat{g}_t(\kappa_t)-{g}_t(\kappa_t))^\top (\kappa_t-\kappa^*_t) + 2\beta^2_1\ltwo{\hat{g}_t(\kappa_t) - g_t(\kappa_t)}^2,\label{eq: 8}
	\end{flalign}
	where $(i)$ follows because $g_t(\kappa_t) = M_t(\kappa_t-\kappa^*_t)$ and $(ii)$ follows because$\lF{M}\leq C_M$. Taking the expectation on both sides of \cref{eq: 8} conditional on $\mf_t$ yields
	\begin{flalign}
		\mE\left[\ltwo{\kappa_{t+1}-\kappa^*_t}^2|\mf_t\right]\leq (1-\beta_1\lambda_M + \beta^2_1C^2_M)\ltwo{\kappa_t - \kappa^*_t}^2 + 2\beta^2_1\mE\left[\ltwo{\hat{g}_t(\kappa_t) - g_t(\kappa_t)}^2|\mf_t\right].\label{eq: 9}
	\end{flalign}
	Next we bound the term $\mE\left[\ltwo{\hat{g}_t(\kappa_t) - g_t(\kappa_t)}^2|\mf_t\right]$ as follows
	\begin{flalign}
		&\mE\left[\ltwo{\hat{g}_t(\kappa_t) - g_t(\kappa_t)}^2|\mf_t\right]\nonumber\\
		&=\mE\left[\ltwo{ (\hat{M}_{t}-M_t)\kappa_t + (\hat{m}_{t}-m_t) }^2|\mf_t\right]\nonumber\\
		&=\mE\left[\ltwo{ (\hat{M}_{t}-M_t)(\kappa_t-\kappa^*_t) + (\hat{M}_{t}-M_t)\kappa^*_t + (\hat{m}_{t}-m_t) }^2|\mf_t\right]\nonumber\\
		&\leq 3\mE\left[\ltwo{\hat{M}_{t}-M_t }^2|\mf_t\right] \ltwo{\kappa_t-\kappa^*_t}^2 + 3R^2_\kappa\mE\left[\ltwo{\hat{M}_{t}-M_t}^2|\mf_t\right] + 3\mE\left[\ltwo{\hat{m}_{t}-m_t}^2|\mf_t\right].\label{eq: 10}
	\end{flalign}
	Recall that $\lF{M_{i,t}}\leq C_M$ and $\ltwo{m_{i,t}}\leq C_m$. We then have
	\begin{flalign}
		\mE\left[\ltwo{\hat{M}_{t}-M_t }^2|\mf_t\right] &\leq \mE\left[\lF{\hat{M}_{t}-M_t }^2|\mf_t\right] = \mE\left[\lF{\frac{1}{N}\sum_{i}M_{i,t}-M_t }^2|\mf_t\right] \nonumber\\
		&\leq \frac{1}{N^2}\sum_{i}\sum_{j}\mE\left[ \langle M_{i,t}-M_t, M_{j,t}-M_t \rangle  |\mf_t \right]\nonumber\\
		&= \frac{1}{N^2}\sum_{i}\mE\left[ \ltwo{M_{i,t}-M_t}^2   |\mf_t \right] \leq \frac{4C^2_M}{N}.\label{eq: 11}
	\end{flalign}
	Similarly, we can obtain
	\begin{flalign}
		\mE\left[\ltwo{\hat{m}_{t}-m_t }^2|\mf_t\right]\leq \frac{4C^2_m}{N}.\label{eq: 12}
	\end{flalign}
	Substituting \cref{eq: 11} and \cref{eq: 12} into \cref{eq: 10} yields
	\begin{flalign}
		\mE\left[\ltwo{\hat{g}_t(\kappa_t) - g_t(\kappa_t)}^2|\mf_t\right]\leq \frac{12C^2_M}{N}\ltwo{\kappa_t - \kappa^*_t}^2 + \frac{12(C^2_MR^2_\kappa + C^2_m)}{N}.\label{eq: 13}
	\end{flalign}
	Substituting \cref{eq: 13} into \cref{eq: 9} yields
	\begin{flalign}
		\mE\left[\ltwo{\kappa_{t+1}-\kappa^*_t}^2|\mf_t\right]\leq (1-\beta_1\lambda_M + 26\beta^2_1C^2_M )\ltwo{\kappa_t - \kappa^*_t}^2 + \frac{24\beta^2_1(C^2_MR^2_\kappa + C^2_m)}{N}.
	\end{flalign}
	Letting $\beta_1 \leq \frac{\lambda_m}{52C^2_M}$, we have
	\begin{flalign*}
		\mE\left[\ltwo{\kappa_{t+1}-\kappa^*_t}^2|\mf_t\right]\leq \left(1-\frac{1}{2}\beta_1\lambda_M \right)\ltwo{\kappa_t - \kappa^*_t}^2 + \frac{24\beta^2_1(C^2_MR^2_\kappa + C^2_m)}{N},
	\end{flalign*}
	which completes the proof.
\end{proof}

We next develop the property for Critic II's update in \cref{eq: c3}. We first introduce the following definitions.

Given a sample from mini-batch $(s_i, a_i, r_i, s^\prime_i)$, $a^\prime_i\sim\pi_{w_t}(\cdot|s^\prime_i)$, we define the following matrix $U_{i,t}\in\mR^{2d_3\times 2d_3}$ and vector $u_{i,t}\in \mR^{2d_3\times 1}$ as
\begin{flalign}
U_{i,t} = \left[\begin{array}{cc}
0 & (\gamma x^\prime_i-x_i)x^{\top}_i  \\
x_i(\gamma x^\prime_i-x_i)^\top & -I 
\end{array}
\right],\qquad
u_{i,t}(\theta_{q,t}) = \left[\begin{array}{c}
0\\
\phi^{\prime\top}_i\theta_{q,t}x_i\nabla_w\log\pi_{w_t}(a^\prime_i|s^\prime_i)
\end{array}
\right].\label{eq: 68}
\end{flalign}
Moreover, consider the matrix $U_t$ and the vector $u_t$. Following the steps similar to those in \cref{eq: 77} and \cref{eq: 78}, we obtain $\lF{U_{i,t}}\leq C_U$ and $\ltwo{u_{i,t}}\leq C_u$, where $C_U = \sqrt{8C^4_x+d_3}$ and $C_u=C_\phi R_q C_{sc}$.

We also define the semi-stochastic-gradient as $\ell_{i,t}(\zeta,\theta_q)=U_{i,t}\xi + u_{i,t}(\theta_{q})$. Then the iteration in \cref{eq: c2} can be rewritten as
\begin{flalign}
\zeta_{t+1} = \zeta_t + \beta_3 \hat{\ell}_t(\zeta_t,\theta_{q,t}),\label{eq: 17}
\end{flalign}
where $\hat{\ell}_t(\zeta_t,\theta_{q,t})=\frac{1}{N}\sum_{i}\ell_{i,t}(\zeta_t,\theta_{q_t})$.
We define $U_t=\mE_i[U_{i,t}]$ and $u_t(\theta^*_{q,t})=\mE_i[u_{i,t}(\theta^*_{q,t})]$, i.e.,
\begin{flalign*}
U_{t} = \left[\begin{array}{cc}
0 & \mE_{\mcd_d\cdot\pi_{w_t}}[(\gamma x^\prime-x)x^{\top}]  \\
\mE_{\mcd_d\cdot\pi_{w_t}}[x(\gamma x^\prime-x)^\top] & -I 
\end{array}
\right],\qquad
u_{t}(\theta^*_{q,t}) = \left[\begin{array}{c}
0\\
\mE_{\mcd_d\cdot\pi_{w_t}}[\phi^{\prime\top}\theta^*_{q,t}x\nabla_w\log\pi_{w_t}(a^\prime|s^\prime)]
\end{array}
\right].
\end{flalign*}
We define the semi-gradient as $\ell_t(\zeta_t, \theta^*_{q,t}) = U_t\zeta_t + u_{t}(\theta^*_{q,t})$, and the fixed point of the iteration \cref{eq: 16} as
\begin{flalign}
\zeta^*_t=U^{-1}_tu_t(\theta^*_{q,t}) = [\theta^{*\top}_{d_q,t}, 0^\top]^\top.\label{eq: 18}
\end{flalign}
where $\theta^{*\top}_{d_q,w_t}=A^{d_q-1}_{w_t}b^{d_q}_{w_t}$, with $A^{d_q}_{w}=\mE_{{\mcd}_d\cdot\pi_{w}}[(\gamma x^\prime-x)x^{\prime\top}]$ and $b^{d_q}_{w}=\mE_{\mcd_d\cdot\pi_{w}}[\phi^{\prime\top}\theta^*_{q,w}x\nabla_w\log\pi_{w}(a^\prime|s^\prime)]$. 

\begin{lemma}\label{lemma3}
	Consider one step update in \cref{eq: c3}. Define $\zeta_t=[\theta^\top_{d_q,t},w^\top_{d_q.t}]^\top$ and $\zeta^*_t$ in \cref{eq: 18}. Let $\beta_3\leq \min\{1/\lambda_U, \lambda_U/16C^2_U\}$ and $N\geq \frac{192C^2_U}{\lambda_U}\left(\frac{2}{\lambda_U}+2\beta_3\right)$, we have
	\begin{flalign*}
	\mE\left[\ltwo{\zeta_{t+1} - \zeta^*_t}^2|\mf_t\right] \leq \left(1 - \frac{1}{4}\beta_3\lambda_U\right)\ltwo{\zeta_t-\zeta^*_t}^2 + C_3\beta_3\ltwo{\theta_{q,t} - \theta^*_{q,t}}^2 + \frac{C_4}{N},
	\end{flalign*}
	where $C_3=\left( \frac{4}{\lambda_U} + 4\beta_3 \right)C^2_\phi C^2_x C^2_{\pi}$ and $C_4=\left( \frac{48\beta_3}{\lambda_U} + 48\beta^2_3 \right)(C^2_UR^2_\zeta + C^2_\phi R^2_{\theta_q}C^2_x C^2_{\pi})$.
\end{lemma}
\begin{proof}
	Following the steps similar to those in the proof of Theorem 1 in Chapter 5 of \cite{maei2011gradient}, we can show that $U_t$ is a Hurwitz matrix which $(\zeta_t-\zeta^*_t)^\top U_t(\zeta_t-\zeta^*_t)\leq -{\lambda_U}\ltwo{\zeta_t-\zeta^*_t}^2$, where $\lambda_U>0$ is a constant. Following steps similar to those in the proof of Theorem 4 in \cite{xu2020improving}, we can obtain
	\begin{flalign}
	\mE\left[\ltwo{\zeta_{t+1}-\zeta^*_t}^2|\mf_t\right]&=\ltwo{\zeta_t + \beta_3\hat{\ell}_t(\zeta_t, \theta_{q,t})-\zeta^*_t}^2\nonumber\\
	&\leq \left(1 - \frac{1}{2}\beta_3\lambda_U + 2C^2_U\beta^2_3\right)\ltwo{\zeta_t-\zeta^*_t}^2 + \left( \frac{2\beta_3}{\lambda_U} + 2\beta^2_3 \right)\mE\left[\ltwo{\hat{\ell}_t(\zeta_t, \theta_{q,t}) - \ell_t(\zeta_t, \theta^*_{q,t})}^2\Big|\mf_t\right].\label{eq: 19}
	\end{flalign}
	Next we bound the term $\mE\left[\ltwo{\hat{\ell}_t(\zeta_t, \theta_{q,t}) - \ell_t(\zeta_t, \theta^*_{q,t})}^2\Big|\mf_t\right]$ as follows:
	\begin{flalign}
	\mE&\left[\ltwo{\hat{\ell}_t(\zeta_t, \theta_{q,t}) - \ell_t(\zeta_t, \theta^*_{q,t})}^2\Big|\mf_t\right]\nonumber\\
	&=\mE\left[\ltwo{\hat{\ell}_t(\zeta_t, \theta_{q,t}) - \hat{\ell}_t(\zeta_t, \theta^*_{q,t}) + \hat{\ell}_t(\zeta_t, \theta^*_{q,t}) - \ell_t(\zeta_t, \theta^*_{q,t})}^2\Big|\mf_t\right]\nonumber\\
	&=2\mE\left[\ltwo{\hat{\ell}_t(\zeta_t, \theta_{q,t}) - \hat{\ell}_t(\zeta_t, \theta^*_{q,t}) }^2\Big|\mf_t\right] + 2\mE\left[\ltwo{ \hat{\ell}_t(\zeta_t, \theta^*_{q,t}) - \ell_t(\zeta_t, \theta^*_{q,t})}^2\Big|\mf_t\right]\nonumber\\
	&=2\mE\left[\ltwo{ \frac{1}{N}\sum_{i} (\ell_{i,t}(\zeta_t,\theta_{q_t}) - \ell_{i,t}(\zeta_t,\theta^*_{q_t}) ) }^2\Big|\mf_t\right] + 2\mE\left[\ltwo{ \hat{\ell}_t(\zeta_t, \theta^*_{q,t}) - \ell_t(\zeta_t, \theta^*_{q,t})}^2\Big|\mf_t\right]\nonumber\\
	&\leq \frac{2}{N} \sum_{i}\mE\left[\ltwo{ (\ell_{i,t}(\zeta_t,\theta_{q_t}) - \ell_{i,t}(\zeta_t,\theta^*_{q_t}) ) }^2\Big|\mf_t\right] + 2\mE\left[\ltwo{ \hat{\ell}_t(\zeta_t, \theta^*_{q,t}) - \ell_t(\zeta_t, \theta^*_{q,t})}^2\Big|\mf_t\right]\nonumber\\
	&= \frac{2}{N} \sum_{i}\mE\left[\ltwo{ \phi^{\prime\top}_i(\theta_{q,t} - \theta^*_{q,t})x_i\nabla_w\log\pi_{w_t}(a^\prime_i|s^\prime_i) }^2\Big|\mf_t\right] + 2\mE\left[\ltwo{ \hat{\ell}_t(\zeta_t, \theta^*_{q,t}) - \ell_t(\zeta_t, \theta^*_{q,t})}^2\Big|\mf_t\right]\nonumber\\
	&= 2C^2_\phi C^2_x C^2_{\pi} \ltwo{\theta_{q,t} - \theta^*_{q,t}}^2 + 2\mE\left[\ltwo{ \hat{\ell}_t(\zeta_t, \theta^*_{q,t}) - \ell_t(\zeta_t, \theta^*_{q,t})}^2\Big|\mf_t\right].\label{eq: 20}
	\end{flalign}
	To bound the term $\mE\left[\ltwo{ \hat{\ell}_t(\zeta_t, \theta^*_{q,t}) - \ell_t(\zeta_t, \theta^*_{q,t})}^2\Big|\mf_t\right]$, we follow the steps similar to those in the proof of bounding the term $\mE\left[\ltwo{\hat{g}_t(\kappa_t) - g_t(\kappa_t)}^2|\mf_t\right]$ in \Cref{lemma1} to obtain
	\begin{flalign}
	\mE\left[\ltwo{ \hat{\ell}_t(\zeta_t, \theta^*_{q,t}) - \ell_t(\zeta_t, \theta^*_{q,t})}^2|\mf_t\right]\leq \frac{12C^2_U}{N}\ltwo{\zeta_t - \zeta^*_t}^2 + \frac{12(C^2_UR^2_\zeta + C^2_\phi R^2_{\theta_q}C^2_x C^2_{\pi})}{N},\label{eq: 21}
	\end{flalign}
	Substituting \cref{eq: 21} and \cref{eq: 20} into \cref{eq: 19} yields
	\begin{flalign}
	\mE&\left[\ltwo{\zeta_{t+1}-\zeta^*_t}^2|\mf_t\right]\nonumber\\
	&\leq \left(1 - \frac{1}{2}\beta_3\lambda_U + 2C^2_U\beta^2_3\right)\ltwo{\zeta_t-\zeta^*_t}^2 \nonumber\\
	&\quad + \left( \frac{2\beta_3}{\lambda_U} + 2\beta^2_3 \right)\left( 2C^2_\phi C^2_x C^2_{\pi} \ltwo{\theta_{q,t} - \theta^*_{q,t}}^2 + \frac{24C^2_U}{N}\ltwo{\zeta_t - \zeta^*_t}^2 + \frac{24(C^2_UR^2_\zeta + C^2_\phi R^2_{\theta_q}C^2_x C^2_{\pi})}{N} \right)\nonumber\\
	&= \left[1 - \frac{1}{2}\beta_3\lambda_U + 2C^2_U\beta^2_3 +  \left( \frac{2\beta_3}{\lambda_U} + 2\beta^2_3 \right) \frac{24C^2_U}{N} \right]\ltwo{\zeta_t-\zeta^*_t}^2\nonumber\\
	&\quad + \left( \frac{4\beta_3}{\lambda_U} + 4\beta^2_3 \right)C^2_\phi C^2_x C^2_{\pi}\ltwo{\theta_{q,t} - \theta^*_{q,t}}^2 + \left( \frac{2\beta_3}{\lambda_U} + 2\beta^2_3 \right)\frac{24(C^2_UR^2_\zeta + C^2_\phi R^2_{\theta_q}C^2_x C^2_{\pi})}{N}.\nonumber
	\end{flalign}
	Letting $\beta_3\leq \min\{1/\lambda_U, \lambda_U/16C^2_U\}$ and $N\geq \frac{192C^2_U}{\lambda_U}\left(\frac{2}{\lambda_U}+2\beta_3\right)$, we have
	\begin{flalign*}
	\mE\left[\ltwo{\zeta_{t+1}-\zeta^*_t}^2|\mf_t\right]\leq \left(1 - \frac{1}{4}\beta_3\lambda_U\right)\ltwo{\zeta_t-\zeta^*_t}^2 + C_3\beta_3\ltwo{\theta_{q,t} - \theta^*_{q,t}}^2 + \frac{C_4}{N},
	\end{flalign*}
	where $C_3=\left( \frac{4}{\lambda_U} + 4\beta_3 \right)C^2_\phi C^2_x C^2_{\pi}$ and $C_4=\left( \frac{48\beta_3}{\lambda_U} + 48\beta^2_3 \right)(C^2_UR^2_\zeta + C^2_\phi R^2_{\theta_q}C^2_x C^2_{\pi})$.
\end{proof}

We next develop the property for Critic III's update. We first introduce the following definitions.

Given a sample $(s_i, a_i, r_i, \tilde{s}^\prime_i)$ generated as we discuss in \Cref{sc: convegence} and $a^\prime_i\sim\pi_{w_t}(\cdot|s^\prime_i)$. We define the following matrix $P_{i,t}\in\mR^{2d_1\times 2d_2}$ and vector $p_{i,t}\in \mR^{2d_2\times 1}$ as
\begin{flalign}
P_{i,t} = \left[\begin{array}{cc}
0 & (\varphi_i-\tilde{\varphi}^\prime_i)\tilde{\varphi}^{\prime\top}_i  \\
\tilde{\varphi}^{\prime}_i(\varphi_i-\tilde{\varphi}^\prime_i)^\top & -I 
\end{array}
\right],\qquad
p_{i,t} = \left[\begin{array}{c}
0\\
\tilde{\varphi}^\prime_i\nabla_w\log\pi_{w_t}(a^\prime_i|\tilde{s}^\prime_i)
\end{array}
\right].\label{eq: 67}
\end{flalign}
Consider the matrix $P_{i,t}$ and the vector $p_{i,t}$. Following the steps similar to those in \cref{eq: 77} and \cref{eq: 78}, we obtain $\lF{P_{i,t}}\leq C_P$ and $\ltwo{p_{i,t}}\leq C_p$, where $C_P=\sqrt{8C^4_\varphi + d_2}$ and $C_p = {C_\varphi C_{sc}}$.

We also define the semi-stochastic-gradient as $h_{i,t}(\xi)=P_{i,t}\xi + p_{i,t}$. Then the iteration in \cref{eq: c2} can be rewritten as
\begin{flalign}
\xi_{t+1} = \xi_t + \beta_2 \hat{h}_t(\xi_t),\label{eq: 16}
\end{flalign}
where $\hat{h}_t(\xi_t)=\frac{1}{N}\sum_{i}h_{i,t}(\xi_t)$. 
We define $P_t=\mE_{i}[P_{i,t}]$ and $p_t=\mE_{i}[p_{i,t}]$, i.e.,
\begin{flalign*}
P_{t} = \left[\begin{array}{cc}
0 & \mE_{\tilde{\mcd}_d\cdot\pi_{w_t}}[(\varphi-\varphi^\prime)\varphi^{\prime\top}]  \\
\mE_{\tilde{\mcd}_d\cdot\pi_{w_t}}[\varphi^{\prime}(\varphi-\varphi^\prime)^\top] & -I 
\end{array}
\right],\qquad
p_{t} = \left[\begin{array}{c}
0\\
\mE_{\tilde{\mcd}_d\cdot\pi_{w_t}}[\varphi^\prime\nabla_w\log\pi_{w_t}(a^\prime|{s}^\prime)]
\end{array}
\right].
\end{flalign*}
We further define the fixed point of the iteration \cref{eq: 16} as
\begin{flalign}
\xi^*_{t}=P^{-1}_tp_t = [\theta^{*\top}_{\psi,w_t}, 0^\top]^\top,\label{eq: 15}
\end{flalign}
where $\theta^{*\top}_{\psi,w_t}=A^{\xi-1}_{w_t}b^{\xi}_{w_t}$, with $A^{\xi}_{w}=\mE_{\tilde{\mcd}_d\cdot\pi_{w}}[(\varphi-\varphi^\prime)\varphi^{\prime\top}]$ and $b^{\xi}_{w}=\mE_{\tilde{\mcd}_d\cdot\pi_{w}}[\varphi^\prime\nabla_w\log\pi_{w}(a^\prime|\tilde{s}^\prime)]$. 

\begin{lemma}\label{lemma2}
	Consider one step update in \cref{eq: c2}. Define $\xi_t=[\theta^\top_{\psi,t},w^\top_{\psi,t}]^\top$ and $\xi^*_t$ in \cref{eq: 15}. Let $\beta_2\leq \min\{ 1/\lambda_P, \lambda_P/52C^2_P\}$, we have
	\begin{flalign*}
		\mE\left[\ltwo{\xi_{t+1} - \xi^*_t}^2|\mf_t\right] \leq \left(1-\frac{1}{2}\beta_2\lambda_P \right)\ltwo{\xi_t - \xi^*_t}^2 + \frac{C_2}{N},
	\end{flalign*}
	where $C_2=24\beta^2_2(C^2_PR^2_\xi + C^2_p)$.
\end{lemma}
\begin{proof}
	Following the steps similar to those in the proof of Theorem 1 in Chapter 5 of \cite{maei2011gradient}, we can show that $P_t$ is a Hurwitz matrix that satisfies $(\xi_t-\xi^*_t)^\top P_t(\xi_t-\xi^*_t)\leq -{\lambda_P}\ltwo{\xi_t-\xi^*_t}^2$, where $\lambda_P>0$ is a constant. Then letting $\beta_2 \leq \min\{\frac{\lambda_p}{52C^2_P}, 1/\lambda_P\}$, following the steps similar to those in the proof of bounding the term $\mE\left[\ltwo{\hat{g}_t(\kappa_t) - g_t(\kappa_t)}^2|\mf_t\right]$ in \Cref{lemma1}, we can obtain
	\begin{flalign*}
	\mE\left[\ltwo{\xi_{t+1}-\xi^*_t}^2|\mf_t\right]\leq \left(1-\frac{1}{2}\beta_2\lambda_P \right)\ltwo{\xi_t - \xi^*_t}^2 + \frac{24\beta^2_2(C^2_PR^2_\xi + C^2_p)}{N}.
	\end{flalign*}
	%where $C_P$ and $C_p$ is characterized in \Cref{lemma8}
\end{proof}

\begin{lemma}\label{lemma6}
	Consider policy $\pi_{w_1}$ and $\pi_{w_2}$, respectively, with the fixed points $\kappa^*_1=[\theta^{*\top}_{q,w_1}, \theta^{*\top}_{\rho,w_1}, \eta^*_{w_1}]^\top$ and $\kappa^*_2=[\theta^{*\top}_{q,w_1}, \theta^{*\top}_{\rho,w_1}, \eta^*_{w_1}]^\top$ as defined in \cref{eq: 15}. We have
	\begin{flalign*}
	\ltwo{\theta^{*\top}_{q,w_1}-\theta^{*\top}_{q,w_2}}&\leq L_q\ltwo{w_1-w_2},\nonumber\\
	\ltwo{\theta^{*\top}_{\rho,w_1}-\theta^{*\top}_{\rho,w_2}}&\leq L_\rho\ltwo{w_1-w_2},\nonumber\\
	\ltwo{\eta^*_{w_1}-\eta^*_{w_2}}&\leq L_\eta\ltwo{w_1-w_2},\nonumber\\
	\end{flalign*}
	where $L_q=\frac{L^\kappa_C(C^\kappa_E+C^\kappa_AR_\rho) + \lambda^{\kappa}_C(L^\kappa_E+ C^\kappa_AL\rho + L^\kappa_A R_\rho)}{(\lambda^{\kappa}_C)^2}$, $L_\rho=\frac{C^\kappa_GL^\kappa_F + \lambda^{\kappa}_FL^\kappa_G}{(\lambda^{\kappa}_F)^2}$, and $L_\eta=C^\kappa_DL_\rho + L^\kappa_DR_\rho$, which further implies that 
	\begin{flalign*}
		\ltwo{\kappa^*_1-\kappa^*_2}\leq L_\kappa\ltwo{w_1-w_2},
	\end{flalign*}
	where $L_\kappa = \sqrt{L^2_q + L^2_p + L^2_\eta}$.
\end{lemma}
\begin{proof}
	Define $A^\kappa_w = \mE_{\mcd_d\cdot\pi_{w}}[(\phi-\gamma\phi^\prime)\phi^\top]$, $C^\kappa_w=\mE_{\mcd_d\cdot\pi_{w_t}}[\phi^\top\phi]$, $D^\kappa_w=\mE_{\mcd_d\cdot\pi_{w_t}}[\phi]$ and $E^\kappa_w=(1-\gamma)\mE_{\mu_0\cdot\pi_{w_t}}[\phi]$. \cite{zhang2020gradientdice,zhang2020gendice} showed that
	\begin{flalign*}
		\theta^*_{\rho,w} &= -F^{\kappa\,-1}_wG^\kappa_w,\nonumber\\
		\theta^*_{q,w} &= C^{\kappa-1}_w(E^\kappa_w - A^\kappa_w\theta^*_{\rho,w} ),\nonumber\\
		\eta^*_w&=D^\kappa_w\theta^*_{\rho,w}-1,\nonumber
	\end{flalign*}
	where
	\begin{flalign*}
		F^\kappa_w &= A^{\kappa\top}_wC^{\kappa\,-1}_wA^{\kappa}_w + D^{\kappa}_wD^{\kappa\top}_w,\nonumber\\
		G^\kappa_w &= A^{\kappa\top}_wC^{\kappa\,-1}_wE^\kappa_w + D^{\kappa}_w.
	\end{flalign*}
	We first develop the Lipschitz property for the matrices $A^\kappa_w$, $C^\kappa_w$, $D^\kappa_w$, and $E^\kappa_w$. For $A^\kappa_w$, we obtain the following
	\begin{flalign}
	&\ltwo{A^{\kappa}_{w_1} - A^{\kappa}_{w_2} }\nonumber\\
	&= \ltwo{-\mE_{\tilde{\mcd}_d\cdot\pi_{w_1}}[(\gamma \phi^\prime-\phi)\phi^{\top}] + \mE_{\tilde{\mcd}_d\cdot\pi_{w_2}}[(\gamma \phi^\prime-\phi)\phi^{\top}]}\nonumber\\
	&= \ltwo{\int \gamma \phi(s^\prime,a^\prime)\phi(s,a)^\top(\pi_{w_2}(da^\prime|s^\prime) - \pi_{w_1}(da^\prime|s^\prime))\tilde{\msP}(ds^\prime|s,a)\mcd(ds,da) }\nonumber\\
	&\quad + \ltwo{\int \phi(s,a)\phi(s,a)^\top(\pi_{w_1}(da^\prime|s^\prime) - \pi_{w_2}(da^\prime|s^\prime))\tilde{\msP}(ds^\prime|s,a)\mcd(ds,da) }\nonumber\\
	&\leq \int \ltwo{\gamma \phi(s^\prime,a^\prime)\phi(s,a)^\top} \lone{(\pi_{w_1}(da^\prime|s^\prime) - \pi_{w_2}(da^\prime|s^\prime))}\tilde{\msP}(ds^\prime|s,a)\mcd(ds,da) \nonumber\\
	&\quad + \int \ltwo{\phi(s,a)\phi(s,a)^\top} \lone{(\pi_{w_2}(da^\prime|s^\prime) - \pi_{w_1}(da^\prime|s^\prime))}\tilde{\msP}(ds^\prime|s,a)\mcd(ds,da) \nonumber\\
	&\leq 2C^2_\phi \int \lone{\pi_{w_1}(da^\prime|s^\prime) - \pi_{w_2}(da^\prime|s^\prime)}\tilde{\msP}(ds^\prime|s,a)\mcd(ds,da)\nonumber\\
	&\leq 2C^2_\phi \lTV{\pi_{w_1}(\cdot) - \pi_{w_2}(\cdot)} \nonumber\\
	&\leq 2C^2_\phi L_\pi \ltwo{w_1-w_2}.\label{eq: 26}
	\end{flalign}
	For $C^\kappa_w$, $D^\kappa_w$, and $E^\kappa_w$, following the steps similar to those in \cref{eq: 26}, we obtain
	\begin{flalign}
		\ltwo{C^{\kappa}_{w_1} - C^{\kappa}_{w_2} }&\leq C^2_\phi L_\pi \ltwo{w_1-w_2} = L^\kappa_C\ltwo{w_1-w_2},\nonumber\\
		\ltwo{D^{\kappa}_{w_1} - D^{\kappa}_{w_2} }&\leq C_\phi L_\pi \ltwo{w_1-w_2}= L^\kappa_D\ltwo{w_1-w_2},\nonumber\\
		\ltwo{E^{\kappa}_{w_1} - E^{\kappa}_{w_2} }& \leq (1-\gamma)C_\phi L_\pi \ltwo{w_1-w_2}= L^\kappa_E\ltwo{w_1-w_2},\label{eq: 33}
	\end{flalign}
	where $L^\kappa_C = C^2_\phi L_\pi$, $L^\kappa_D=C_\phi L_\pi$, and $L^\kappa_E=(1-\gamma)C_\phi L_\pi$.
	We then proceed to bound the two terms $\ltwo{F^\kappa_{w_1}-F^\kappa_{w_2}}$ and $\ltwo{G^\kappa_{w_1}-G^\kappa_{w_2}}$. For $\ltwo{F^\kappa_{w_1}-F^\kappa_{w_2}}$, we obtain the following bound:
	\begin{flalign}
		&\ltwo{F^\kappa_{w_1}-F^\kappa_{w_2}}\nonumber\\
		&= \ltwo{A^{\kappa\top}_{w_1}C^{\kappa\,-1}_{w_1}A^{\kappa}_{w_1} - A^{\kappa\top}_{w_2}C^{\kappa\,-1}_{w_2}A^{\kappa}_{w_2} + D^{\kappa}_{w_1}D^{\kappa\top}_{w_1} - D^{\kappa}_{w_2}D^{\kappa\top}_{w_2}}\nonumber\\
		&\leq \ltwo{A^{\kappa\top}_{w_1}C^{\kappa\,-1}_{w_1}A^{\kappa}_{w_1} - A^{\kappa\top}_{w_2}C^{\kappa\,-1}_{w_2}A^{\kappa}_{w_2}} + \ltwo{D^{\kappa}_{w_1}D^{\kappa\top}_{w_1} - D^{\kappa}_{w_2}D^{\kappa\top}_{w_2}}\nonumber\\
		&= \ltwo{A^{\kappa\top}_{w_1}C^{\kappa\,-1}_{w_1} (C^{\kappa}_{w_2} - C^{\kappa}_{w_1})C^{\kappa\,-1}_{w_2} A^{\kappa}_{w_1} + (A^{\kappa}_{w_1}-A^{\kappa}_{w_2})^\top C^{\kappa\,-1}_{w_2} A^{\kappa}_{w_1} + A^{\kappa}_{w_2}C^{\kappa\,-1}_{w_2}(A^{\kappa}_{w_1}-A^{\kappa}_{w_2}) } \nonumber\\
		&\quad + \ltwo{D^{\kappa}_{w_1} (D^{\kappa}_{w_1} - D^{\kappa}_{w_2})^\top + (D^{\kappa}_{w_1} - D^{\kappa}_{w_2} ) D^{\kappa\top}_{w_2} }\nonumber\\
		&\leq \ltwo{A^{\kappa\top}_{w_1}C^{\kappa\,-1}_{w_1} (C^{\kappa}_{w_2} - C^{\kappa}_{w_1})C^{\kappa\,-1}_{w_2} A^{\kappa}_{w_1}} + \ltwo{(A^{\kappa}_{w_1}-A^{\kappa}_{w_2})^\top C^{\kappa\,-1}_{w_2} A^{\kappa}_{w_1}} + \ltwo{A^{\kappa}_{w_2}C^{\kappa\,-1}_{w_2}(A^{\kappa}_{w_1}-A^{\kappa}_{w_2})}\nonumber\\
		&\quad + \ltwo{D^{\kappa}_{w_1} (D^{\kappa}_{w_1} - D^{\kappa}_{w_2})^\top} + \ltwo{(D^{\kappa}_{w_1} - D^{\kappa}_{w_2} ) D^{\kappa\top}_{w_2}}\nonumber\\
		&\leq \left(\frac{L^\kappa_C(C^{\kappa}_A)^2 + 2C^{\kappa}_A\lambda^{\kappa}_CL^\kappa_C}{(\lambda^{\kappa}_C)^2} + C^\kappa_DL^\kappa_D\right)\ltwo{w_1-w_2}\nonumber\\
		&=L^\kappa_F\ltwo{w_1-w_2},\label{eq: 30}
	\end{flalign}
	where
	\begin{flalign*}
		L^\kappa_F = \frac{L^\kappa_C(C^{\kappa}_A)^2 + 2C^{\kappa}_A\lambda^{\kappa}_CL^\kappa_C}{(\lambda^{\kappa}_C)^2} + C^\kappa_DL^\kappa_D.
	\end{flalign*}
	For $\ltwo{G^\kappa_{w_1}-G^\kappa_{w_2}}$, we have
	\begin{flalign}
		&\ltwo{G^\kappa_{w_1}-G^\kappa_{w_2}}\nonumber\\
		&= \ltwo{A^{\kappa\top}_{w_1}C^{\kappa\,-1}_{w_1}E^\kappa_{w_1} - A^{\kappa\top}_{w_2}C^{\kappa\,-1}_{w_2}E^\kappa_{w_2} + D^{\kappa}_{w_1} - D^{\kappa}_{w_2}}\nonumber\\
		&\leq \ltwo{A^{\kappa\top}_{w_1}C^{\kappa\,-1}_{w_1}E^\kappa_{w_1} - A^{\kappa\top}_{w_2}C^{\kappa\,-1}_{w_2}E^\kappa_{w_2}} + \ltwo{D^{\kappa}_{w_1} - D^{\kappa}_{w_2}}\nonumber\\
		&= \ltwo{A^{\kappa\top}_{w_1}C^{\kappa\,-1}_{w_1} (C^{\kappa}_{w_2} - C^{\kappa}_{w_1})C^{\kappa\,-1}_{w_2} E^{\kappa}_{w_1} + (A^{\kappa}_{w_1}-A^{\kappa}_{w_2})^\top C^{\kappa\,-1}_{w_2} E^{\kappa}_{w_1} + A^{\kappa}_{w_2}C^{\kappa\,-1}_{w_2}(E^{\kappa}_{w_1}-E^{\kappa}_{w_2}) } + \ltwo{D^{\kappa}_{w_1} - D^{\kappa}_{w_2}} \nonumber\\
		&\leq \left(\frac{L^\kappa_C C^{\kappa}_AC^{\kappa}_E + C^{\kappa}_E\lambda^{\kappa}_CL^\kappa_A + C^{\kappa}_A\lambda^{\kappa}_CL^\kappa_E}{(\lambda^{\kappa}_C)^2} + L^\kappa_D\right)\ltwo{w_1-w_2}\nonumber\\
		&=L^\kappa_G\ltwo{w_1-w_2},\label{eq: 31}
	\end{flalign}
	where
	\begin{flalign*}
		L^\kappa_G = \frac{L^\kappa_C C^{\kappa}_AC^{\kappa}_E + C^{\kappa}_E\lambda^{\kappa}_CL^\kappa_A + C^{\kappa}_A\lambda^{\kappa}_CL^\kappa_E}{(\lambda^{\kappa}_C)^2} + L^\kappa_D.
	\end{flalign*}
	We next prove the Lipschitz property for $\theta^{*}_{\rho}$. To bound $\ltwo{\theta^{*\top}_{\rho,w_1} - \theta^{*\top}_{\rho,w_2}}$, we proceed as follows.
	\begin{flalign}
	&\theta^{*\top}_{\rho,w_1} - \theta^{*\top}_{\rho,w_2}=F^{\kappa-1}_{w_1}G^{\kappa}_{w_1} - F^{\kappa-1}_{w_2}G^{\kappa}_{w_2} = (F^{\kappa-1}_{w_1} - F^{\kappa-1}_{w_2})G^{\kappa}_{w_1} + F^{\kappa-1}_{w_2}(G^{\kappa}_{w_1} - G^{\kappa}_{w_2})\nonumber\\
	&=(F^{\kappa-1}_{w_1}F^{\kappa}_{w_2}F^{\kappa-1}_{w_2} - F^{\kappa-1}_{w_1}F^{\kappa}_{w_1}F^{\kappa-1}_{w_2})G^{\kappa}_{w_1} + F^{\kappa-1}_{w_2}(G^{\kappa}_{w_1} - G^{\kappa}_{w_2})\nonumber\\
	&= F^{\kappa-1}_{w_1}(F^{\kappa}_{w_2} - F^{\kappa}_{w_1})F^{\kappa-1}_{w_2} G^{\kappa}_{w_1} + F^{\kappa-1}_{w_2}(G^{\kappa}_{w_1} - G^{\kappa}_{w_2}),\label{eq: 25}
	\end{flalign}
	which implies
	\begin{flalign}
	\ltwo{\theta^{*\top}_{\rho,w_1} - \theta^{*\top}_{\rho,w_2}}&\leq \ltwo{F^{\kappa-1}_{w_1}}\ltwo{F^{\kappa}_{w_2} - F^{\kappa}_{w_1}}\ltwo{F^{\kappa-1}_{w_2}} \ltwo{G^{\kappa}_{w_1}} + \ltwo{F^{\kappa-1}_{w_2}}\ltwo{G^{\kappa}_{w_1} - G^{\kappa}_{w_2}}\nonumber\\
	&\leq \frac{C^\kappa_G}{(\lambda^{\kappa}_F)^2}\ltwo{F^{\kappa}_{w_1} - F^{\kappa}_{w_2}} + \frac{1}{\lambda^{\kappa}_F}\ltwo{G^{\kappa}_{w_1} - G^{\kappa}_{w_2}}\nonumber\\
	&\leq \frac{C^\kappa_GL^\kappa_F + \lambda^{\kappa}_FL^\kappa_G}{(\lambda^{\kappa}_F)^2}\ltwo{w_1-w_2} = L_\rho\ltwo{w_1-w_2},\label{eq: 22}
	\end{flalign}
	where $L_\rho = \frac{C^\kappa_GL^\kappa_F + \lambda^{\kappa}_FL^\kappa_G}{(\lambda^{\kappa}_F)^2}$.
	
	We then consider the Lipschitz property of $\theta^{*}_{q}$. To bound $\ltwo{\theta^{*}_{q,w_1} - \theta^{*\top}_{q,w_2}}$, we proceed as follows.
	\begin{flalign}
		&\theta^{*}_{q,w_1} - \theta^{*}_{q,w_2} \nonumber\\
		&= C^{\kappa-1}_{w_1}(E^\kappa_{w_1} - A^\kappa_{w_1}\theta^*_{\rho,w_1} ) - C^{\kappa-1}_{w_2}(E^\kappa_{w_2} - A^\kappa_{w_2}\theta^*_{\rho,w_2} )\nonumber\\
		&= (C^{\kappa-1}_{w_1} - C^{\kappa-1}_{w_2})(E^\kappa_{w_1} - A^\kappa_{w_1}\theta^*_{\rho,w_1} ) + C^{\kappa-1}_{w_2}(E^\kappa_{w_1} - E^\kappa_{w_2} + A^\kappa_{w_2}\theta^*_{\rho,w_2} - A^\kappa_{w_1}\theta^*_{\rho,w_1} )\nonumber\\
		&= C^{\kappa-1}_{w_1}(C^{\kappa}_{w_2} - C^{\kappa}_{w_1})C^{\kappa-1}_{w_2}(E^\kappa_{w_1} - A^\kappa_{w_1}\theta^*_{\rho,w_1} )\nonumber\\
		&\quad  + C^{\kappa-1}_{w_2}\left[(E^\kappa_{w_1} - E^\kappa_{w_2}) + A^\kappa_{w_2} (\theta^*_{\rho,w_2} - \theta^*_{\rho,w_1}) + (A^\kappa_{w_2} - A^\kappa_{w_1})\theta^*_{\rho,w_1} \right],\nonumber
	\end{flalign}
	which implies
	\begin{flalign*}
		&\ltwo{\theta^{*}_{q,w_1} - \theta^{*}_{q,w_2}}\nonumber\\
		&\leq \ltwo{C^{\kappa-1}_{w_1}}\ltwo{C^{\kappa}_{w_2} - C^{\kappa}_{w_1}} \ltwo{C^{\kappa-1}_{w_2}}\left( \ltwo{E^\kappa_{w_1}} + \ltwo{A^\kappa_{w_1}}\ltwo{\theta^*_{\rho,w_1}} \right)\nonumber\\
		&\quad  + \ltwo{C^{\kappa-1}_{w_2}} \left[\ltwo{E^\kappa_{w_1} - E^\kappa_{w_2}} + \ltwo{A^\kappa_{w_2}} \ltwo{\theta^*_{\rho,w_2} - \theta^*_{\rho,w_1}} + \ltwo{A^\kappa_{w_2} - A^\kappa_{w_1}}\ltwo{\theta^*_{\rho,w_1}} \right]\nonumber\\
		&\leq \left[\frac{L^\kappa_C(C^\kappa_E+C^\kappa_AR_\rho) + \lambda^{\kappa}_C(L^\kappa_E+ C^\kappa_AL\rho + L^\kappa_A R_\rho)}{(\lambda^{\kappa}_C)^2}\right]\ltwo{w_1-w_2} = L_q\ltwo{w_1-w_2},
	\end{flalign*}
	where $L_q = \frac{L^\kappa_C(C^\kappa_E+C^\kappa_AR_\rho) + \lambda^{\kappa}_C(L^\kappa_E+ C^\kappa_AL\rho + L^\kappa_A R_\rho)}{(\lambda^{\kappa}_C)^2}$.
	
	Finally, we consider the Lipschitz property of $\eta^*$. To bound $\ltwo{\eta^*_{w_1} - \eta^*_{w_2}}$, we proceed as follows
	\begin{flalign}
		& \ltwo{\eta^*_{w_1} - \eta^*_{w_2}} \nonumber\\
		&= \ltwo{D^\kappa_{w_1}\theta^*_{\rho,w_1} - D^\kappa_{w_2}\theta^*_{\rho,w_2}} \nonumber\\
		&= \ltwo{D^\kappa_{w_1}(\theta^*_{\rho,w_1} - \theta^*_{\rho,w_2}) + (D^\kappa_{w_1} - D^\kappa_{w_2})\theta^*_{\rho,w_2}}\nonumber\\
		&\leq \ltwo{D^\kappa_{w_1}(\theta^*_{\rho,w_1} - \theta^*_{\rho,w_2})} + \ltwo{(D^\kappa_{w_1} - D^\kappa_{w_2})\theta^*_{\rho,w_2}}\nonumber\\
		&\leq (C^\kappa_DL_\rho + L^\kappa_DR_\rho) \ltwo{w_1-w_2} = L_\eta\ltwo{w_1-w_2},
	\end{flalign}
	where $L_\eta = C^\kappa_DL_\rho + L^\kappa_DR_\rho$.
\end{proof}

\begin{lemma}\label{lemma4}
	Consider the policies $\pi_{w_1}$ and $\pi_{w_2}$, respectively, with the fixed points $\xi^*_1$ and $\xi^*_2$ defined in \cref{eq: 15}. We have
	\begin{flalign*}
		\ltwo{\xi^*_1-\xi^*_2}\leq L_{\xi}\ltwo{w_1-w_2},
	\end{flalign*}
	where $L_{\xi}=\frac{2C^2_\varphi C^\xi_bL_\pi + \lambda^{\xi}_AC_\varphi (C_{sc} L_{\pi} + L_{sc})}{(\lambda^{\xi}_A)^2}$.
\end{lemma}
\begin{proof}
	Recall that for $k=1$ or $2$, we have $\xi^*_{k}=P^{-1}_kp_k = [\theta^{*\top}_{\psi,w_k}, 0^\top]^\top$, where $\theta^{*\top}_{\psi,w_k}=A^{\xi-1}_{w_k}b^{\xi}_{w_k}$, with $A^{\xi}_{w}=\mE_{\tilde{\mcd}_d\cdot\pi_{w}}[(\varphi-\varphi^\prime)\varphi^{\prime\top}]$ and $b^{\xi}_{w}=\mE_{\tilde{\mcd}_d\cdot\pi_{w}}[\varphi^\prime\nabla_w\log\pi_{w}(a^\prime|{s}^\prime)]$, which implies that
	\begin{flalign*}
		\ltwo{\xi^*_1-\xi^*_2} = \ltwo{\theta^{*\top}_{\psi,w_1} - \theta^{*\top}_{\psi,w_2}}.
	\end{flalign*}
	To bound $\ltwo{\theta^{*\top}_{\psi,w_1} - \theta^{*\top}_{\psi,w_2}}$, following the steps similar to those in \cref{eq: 25} and \cref{eq: 22}, we obtain
	\begin{flalign}
		\ltwo{\theta^{*\top}_{\psi,w_1} - \theta^{*\top}_{\psi,w_2}}\leq \frac{C^\xi_b}{(\lambda^{\xi}_A)^2}\ltwo{A^{\xi}_{w_1} - A^{\xi}_{w_2}} + \frac{1}{\lambda^{\xi}_A}\ltwo{b^{\xi}_{w_1} - b^{\xi}_{w_2}}.\label{eq: 34}
	\end{flalign}
%	\begin{flalign}
%		&\quad \theta^{*\top}_{\psi,w_1} - \theta^{*\top}_{\psi,w_2}=A^{\xi-1}_{w_1}b^{\xi}_{w_1} - A^{\xi-1}_{w_2}b^{\xi}_{w_2} = (A^{\xi-1}_{w_1} - A^{\xi-1}_{w_2})b^{\xi}_{w_1} + A^{\xi-1}_{w_2}(b^{\xi}_{w_1} - b^{\xi}_{w_2})\nonumber\\
%		&=(A^{\xi-1}_{w_1}A^{\xi}_{w_2}A^{\xi-1}_{w_2} - A^{\xi-1}_{w_1}A^{\xi}_{w_1}A^{\xi-1}_{w_2})b^{\xi}_{w_1} + A^{\xi-1}_{w_2}(b^{\xi}_{w_1} - b^{\xi}_{w_2})\nonumber\\
%		&= A^{\xi-1}_{w_1}(A^{\xi}_{w_2} - A^{\xi}_{w_1})A^{\xi-1}_{w_2} b^{\xi}_{w_1} + A^{\xi-1}_{w_2}(b^{\xi}_{w_1} - b^{\xi}_{w_2}),\label{eq: 25}
%	\end{flalign}
%	which implies
%	\begin{flalign}
%		\ltwo{\theta^{*\top}_{\psi,w_1} - \theta^{*\top}_{\psi,w_2}}&\leq \ltwo{A^{\xi-1}_{w_1}}\ltwo{A^{\xi}_{w_2} - A^{\xi}_{w_1}}\ltwo{A^{\xi-1}_{w_2}} \ltwo{b^{\xi}_{w_1}} + \ltwo{A^{\xi-1}_{w_2}}\ltwo{b^{\xi}_{w_1} - b^{\xi}_{w_2}}\nonumber\\
%		&\leq \frac{C^\xi_b}{(\lambda^{\xi}_A)^2}\ltwo{A^{\xi}_{w_1} - A^{\xi}_{w_2}} + \frac{1}{\lambda^{\xi}_A}\ltwo{b^{\xi}_{w_1} - b^{\xi}_{w_2}}.\label{eq: 22}
%	\end{flalign}
	We first bound the term $\ltwo{A^{\xi}_{w_2} - A^{\xi}_{w_1}}$. Following the steps similar to those in \cref{eq: 26}, we obtain
	\begin{flalign}
		\ltwo{A^{\xi}_{w_2} - A^{\xi}_{w_1} }\leq 2C^2_\varphi L_\pi \ltwo{w_1-w_2}.\label{eq: 23}
	\end{flalign}
%	\begin{flalign}
%		&\quad \ltwo{A^{\xi}_{w_2} - A^{\xi}_{w_1} }\nonumber\\
%		&= \ltwo{\mE_{\tilde{\mcd}_d\cdot\pi_{w_1}}[(\varphi-\varphi^\prime)\varphi^{\prime\top}] - \mE_{\tilde{\mcd}_d\cdot\pi_{w_2}}[(\varphi-\varphi^\prime)\varphi^{\prime\top}]}\nonumber\\
%		&= \ltwo{\int \varphi(s,a)\varphi(s^\prime,a^\prime)^\top(\pi_{w_1}(da^\prime|s^\prime) - \pi_{w_2}(da^\prime|s^\prime))\tilde{\msP}(ds^\prime|s,a)\mcd(ds,da) }\nonumber\\
%		&\quad + \ltwo{\int \varphi(s^\prime,a^\prime)\varphi(s^\prime,a^\prime)^\top(\pi_{w_2}(da^\prime|s^\prime) - \pi_{w_1}(da^\prime|s^\prime))\tilde{\msP}(ds^\prime|s,a)\mcd(ds,da) }\nonumber\\
%		&\leq \int \ltwo{\varphi(s,a)\varphi(s^\prime,a^\prime)^\top} \lone{(\pi_{w_1}(da^\prime|s^\prime) - \pi_{w_2}(da^\prime|s^\prime))}\tilde{\msP}(ds^\prime|s,a)\mcd(ds,da) \nonumber\\
%		&\quad + \int \ltwo{\varphi(s^\prime,a^\prime)\varphi(s^\prime,a^\prime)^\top} \lone{(\pi_{w_2}(da^\prime|s^\prime) - \pi_{w_1}(da^\prime|s^\prime))}\tilde{\msP}(ds^\prime|s,a)\mcd(ds,da) \nonumber\\
%		&\leq 2C^2_\varphi \int \lone{\pi_{w_1}(da^\prime|s^\prime) - \pi_{w_2}(da^\prime|s^\prime)}\tilde{\msP}(ds^\prime|s,a)\mcd(ds,da)\nonumber\\
%		&\leq 2C^2_\varphi \lTV{\pi_{w_1}(\cdot) - \pi_{w_2}(\cdot)} \nonumber\\
%		&\leq 2C^2_\varphi L_\pi \ltwo{w_1-w_2}.\label{eq: 23}
%	\end{flalign}
	We then bound the term $\ltwo{b^{\xi}_{w_1} - b^{\xi}_{w_2}}$. Based on the definition of $b^{\xi}_{w}$, we have
	\begin{flalign}
		&\ltwo{b^{\xi}_{w_1} - b^{\xi}_{w_2}}\nonumber\\
		&= \ltwo{\mE_{\tilde{\mcd}_d\cdot\pi_{w_1}}[\varphi^\prime\nabla_w\log\pi_{w_1}(a^\prime|{s}^\prime)] - \mE_{\tilde{\mcd}_d\cdot\pi_{w_2}}[\varphi^\prime\nabla_w\log\pi_{w_2}(a^\prime|{s}^\prime)]}\nonumber\\
		&\leq \ltwo{\mE_{\tilde{\mcd}_d\cdot\pi_{w_1}}[\varphi^\prime\nabla_w\log\pi_{w_1}(a^\prime|{s}^\prime)] - \mE_{\tilde{\mcd}_d\cdot\pi_{w_2}}[\varphi^\prime\nabla_w\log\pi_{w_1}(a^\prime|{s}^\prime)]}\nonumber\\
		&\quad + \ltwo{\mE_{\tilde{\mcd}_d\cdot\pi_{w_2}}[\varphi^\prime\nabla_w\log\pi_{w_1}(a^\prime|{s}^\prime)] - \mE_{\tilde{\mcd}_d\cdot\pi_{w_2}}[\varphi^\prime\nabla_w\log\pi_{w_2}(a^\prime|{s}^\prime)]}\nonumber\\
		&=\ltwo{\int \varphi(s^\prime,a^\prime)\nabla_w\log\pi_{w_1}(a^\prime|{s}^\prime) (\pi_{w_1}(da^\prime|s^\prime) - \pi_{w_2}(da^\prime|s^\prime))\tilde{\msP}(ds^\prime|s,a)\mcd(ds,da) }\nonumber\\
		&\quad + \ltwo{\int \varphi(s^\prime,a^\prime)(\nabla_w\log\pi_{w_1}(a^\prime|{s}^\prime) - \nabla_w\log\pi_{w_2}(a^\prime|{s}^\prime)) \pi_{w_2}(da^\prime|s^\prime)\tilde{\msP}(ds^\prime|s,a)\mcd(ds,da) }\nonumber\\
		&=\int \ltwo{\varphi(s^\prime,a^\prime)\nabla_w\log\pi_{w_1}(a^\prime|{s}^\prime)} \lone{\pi_{w_1}(da^\prime|s^\prime) - \pi_{w_2}(da^\prime|s^\prime)} \tilde{\msP}(ds^\prime|s,a)\mcd(ds,da) \nonumber\\
		&\quad + \int \ltwo{\varphi(s^\prime,a^\prime)}\ltwo{\nabla_w\log\pi_{w_1}(a^\prime|{s}^\prime) - \nabla_w\log\pi_{w_2}(a^\prime|{s}^\prime)} \pi_{w_2}(da^\prime|s^\prime)\tilde{\msP}(ds^\prime|s,a)\mcd(ds,da) \nonumber\\
		&\leq C_\varphi C_{sc} \int \lone{\pi_{w_1}(da^\prime|s^\prime) - \pi_{w_2}(da^\prime|s^\prime)} \tilde{\msP}(ds^\prime|s,a)\mcd(ds,da) + C_\varphi L_{sc}\ltwo{w_1-w_2}\nonumber\\
		&\leq C_\varphi (C_{sc} L_{\pi} + L_{sc})\ltwo{w_1-w_2}.\label{eq: 24}
	\end{flalign}
	Substituting \cref{eq: 23} and \cref{eq: 24} into \cref{eq: 22} yields
	\begin{flalign*}
		\ltwo{\theta^{*\top}_{\psi,w_1} - \theta^{*\top}_{\psi,w_2}}\leq \frac{2C^2_\varphi C^\xi_bL_\pi + \lambda^{\xi}_AC_\varphi (C_{sc} L_{\pi} + L_{sc})}{(\lambda^{\xi}_A)^2}\ltwo{w_1-w_2}  = L_\xi\ltwo{w_1-w_2}.
	\end{flalign*}
	Thus, we have $\ltwo{\xi^*_1-\xi^*_2}\leq L_\xi\ltwo{w_1-w_2}$, which completes the proof.
\end{proof}

\begin{lemma}\label{lemma5}
	Consider the policies $\pi_{w_1}$ and $\pi_{w_2}$, respectively, with the fixed points $\zeta^*_1$ and $\zeta^*_2$ defined in \cref{eq: 18}. Then, we have
	\begin{flalign*}
		\ltwo{\zeta^*_1-\zeta^*_2}\leq L_\zeta\ltwo{w_1-w_2},
	\end{flalign*}
	where $L_\zeta = \frac{2C^2_x L_\pi C^\zeta_b + \lambda^{\zeta}_AC_\phi C_x (R_qL_{sc} + L_qC_{sc}) + \lambda^{\zeta}_AC_\phi R_qC_xC_{sc}L_\pi}{(\lambda^{\zeta}_A)^2}$.
\end{lemma}
	\begin{proof}
		Recall that for $k=1$ or $2$, we have $\zeta^*_k=U^{-1}_ku_k(\theta^*_{q,k}) = [\theta^{*\top}_{d_q,k}, 0^\top]^\top$, where $\theta^{*\top}_{d_q,w_k}=A^{\zeta-1}_{w_k}b^{\zeta}_{w_k}$, with $A^{\zeta}_{w}=\mE_{\mcd_d\cdot\pi_{w}}[(\gamma x(s^\prime,a^\prime)-x(s,a))x(s,a)^{\top}]$ and $b^{\zeta}_{w}=	\mE_{\mcd_d\cdot\pi_{w}}[\phi(s^\prime,a^\prime)^{\top}\theta^*_{q,w}x(s,a)\nabla_w\log\pi_{w}(a^\prime|s^\prime)]$, which implies that
		\begin{flalign*}
			\ltwo{\zeta^*_1-\zeta^*_2} = \ltwo{\theta^{*\top}_{d_q,1} - \theta^{*\top}_{d_q,2}}.
		\end{flalign*}
		To bound $\ltwo{\theta^{*\top}_{d_q,1} - \theta^{*\top}_{d_q,2}}$, following the steps similar to those in \cref{eq: 25} and \cref{eq: 22}, we obtain
		\begin{flalign}
			\ltwo{\theta^{*\top}_{d_q,1} - \theta^{*\top}_{d_q,2}} \leq \frac{C^\zeta_b}{(\lambda^{\zeta}_A)^2}\ltwo{A^{\zeta}_{w_1} - A^{\zeta}_{w_2}} + \frac{1}{\lambda^{\zeta}_A}\ltwo{b^{\zeta}_{w_1} - b^{\zeta}_{w_2}}.\label{eq: 29}
		\end{flalign}
		We first bound the term $\ltwo{A^{\zeta}_{w_1} - A^{\zeta}_{w_2}}$. Following the steps similar to those in \cref{eq: 26}, we obtain
		\begin{flalign}
			\ltwo{A^{\zeta}_{w_1} - A^{\zeta}_{w_2} }\leq 2C^2_x L_\pi \ltwo{w_1-w_2}.\label{eq: 32}
		\end{flalign}
%		\begin{flalign}
%		&\quad \ltwo{A^{\zeta}_{w_1} - A^{\zeta}_{w_2} }\nonumber\\
%		&= \ltwo{\mE_{\tilde{\mcd}_d\cdot\pi_{w_1}}[(\gamma x^\prime-x)x^{\top}] - \mE_{\tilde{\mcd}_d\cdot\pi_{w_2}}[(\gamma x^\prime-x)x^{\top}]}\nonumber\\
%		&= \ltwo{\int \gamma x(s^\prime,a^\prime)x(s,a)^\top(\pi_{w_1}(da^\prime|s^\prime) - \pi_{w_2}(da^\prime|s^\prime))\tilde{\msP}(ds^\prime|s,a)\mcd(ds,da) }\nonumber\\
%		&\quad + \ltwo{\int x(s,a)x(s,a)^\top(\pi_{w_2}(da^\prime|s^\prime) - \pi_{w_1}(da^\prime|s^\prime))\tilde{\msP}(ds^\prime|s,a)\mcd(ds,da) }\nonumber\\
%		&\leq \int \ltwo{\gamma x(s^\prime,a^\prime)x(s,a)^\top} \lone{(\pi_{w_1}(da^\prime|s^\prime) - \pi_{w_2}(da^\prime|s^\prime))}\tilde{\msP}(ds^\prime|s,a)\mcd(ds,da) \nonumber\\
%		&\quad + \int \ltwo{x(s,a)x(s,a)^\top} \lone{(\pi_{w_2}(da^\prime|s^\prime) - \pi_{w_1}(da^\prime|s^\prime))}\tilde{\msP}(ds^\prime|s,a)\mcd(ds,da) \nonumber\\
%		&\leq 2C^2_x \int \lone{\pi_{w_1}(da^\prime|s^\prime) - \pi_{w_2}(da^\prime|s^\prime)}\tilde{\msP}(ds^\prime|s,a)\mcd(ds,da)\nonumber\\
%		&\leq 2C^2_x \lTV{\pi_{w_1}(\cdot) - \pi_{w_2}(\cdot)} \nonumber\\
%		&\leq 2C^2_x L_\pi \ltwo{w_1-w_2}.\label{eq: 26}
%		\end{flalign}
		We then bound the term $\ltwo{b^{\zeta}_{w_1} - b^{\zeta}_{w_2}}$. Based on to the definition of $b^{\zeta}_{w}$, we have
		\begin{flalign}
		&\ltwo{b^{\zeta}_{w_1} - b^{\zeta}_{w_2}}\nonumber\\
		&= \ltwo{\mE_{\mcd_d\cdot\pi_{w_1}}[\phi(s^\prime,a^\prime)^{\top}\theta^*_{q,w_1}x(s,a)\nabla_w\log\pi_{w_1}(a^\prime|s^\prime)] - \mE_{\mcd_d\cdot\pi_{w_2}}[\phi(s^\prime,a^\prime)^{\top}\theta^*_{q,w_2}x(s,a)\nabla_w\log\pi_{w_2}(a^\prime|s^\prime)] }\nonumber\\
		&\leq \ltwo{\mE_{\mcd_d\cdot\pi_{w_1}}[\phi(s^\prime,a^\prime)^{\top}\theta^*_{q,w_1}x(s,a)\nabla_w\log\pi_{w_1}(a^\prime|s^\prime)] - \mE_{\mcd_d\cdot\pi_{w_2}}[\phi(s^\prime,a^\prime)^{\top}\theta^*_{q,w_1}x(s,a)\nabla_w\log\pi_{w_1}(a^\prime|s^\prime)] }\nonumber\\
		&\quad + \ltwo{\mE_{\mcd_d\cdot\pi_{w_2}}[\phi(s^\prime,a^\prime)^{\top}\theta^*_{q,w_1}x(s,a)\nabla_w\log\pi_{w_1}(a^\prime|s^\prime)] - \mE_{\mcd_d\cdot\pi_{w_2}}[\phi(s^\prime,a^\prime)^{\top}\theta^*_{q,w_2}x(s,a)\nabla_w\log\pi_{w_2}(a^\prime|s^\prime)] }\nonumber\\
		&\leq \ltwo{\mE_{\mcd_d\cdot\pi_{w_2}}[\phi(s^\prime,a^\prime)^{\top}\theta^*_{q,w_1}x(s,a)\nabla_w\log\pi_{w_1}(a^\prime|s^\prime)] - \mE_{\mcd_d\cdot\pi_{w_2}}[\phi(s^\prime,a^\prime)^{\top}\theta^*_{q,w_2}x(s,a)\nabla_w\log\pi_{w_2}(a^\prime|s^\prime)] }\nonumber\\
		&\quad + C_\phi R_qC_xC_{sc}L_\pi \ltwo{w_1-w_2}.\label{eq: 27}
		\end{flalign}
		Consider the term
		\begin{flalign*}
			\ltwo{\mE_{\mcd_d\cdot\pi_{w_2}}[\phi(s^\prime,a^\prime)^{\top}\theta^*_{q,w_1}x(s,a)\nabla_w\log\pi_{w_1}(a^\prime|s^\prime)] - \mE_{\mcd_d\cdot\pi_{w_2}}[\phi(s^\prime,a^\prime)^{\top}\theta^*_{q,w_2}x(s,a)\nabla_w\log\pi_{w_2}(a^\prime|s^\prime)] },
		\end{flalign*}
		we have
		\begin{flalign}
			& \ltwo{\mE_{\mcd_d\cdot\pi_{w_2}}[\phi(s^\prime,a^\prime)^{\top}\theta^*_{q,w_1}x(s,a)\nabla_w\log\pi_{w_1}(a^\prime|s^\prime)] - \mE_{\mcd_d\cdot\pi_{w_2}}[\phi(s^\prime,a^\prime)^{\top}\theta^*_{q,w_2}x(s,a)\nabla_w\log\pi_{w_2}(a^\prime|s^\prime)] }\nonumber\\
			&\leq \ltwo{\mE_{\mcd_d\cdot\pi_{w_2}}[\phi(s^\prime,a^\prime)^{\top}\theta^*_{q,w_1}x(s,a) (\nabla_w\log\pi_{w_1}(a^\prime|s^\prime)-\nabla_w\log\pi_{w_2}(a^\prime|s^\prime))] }\nonumber\\
			&\quad + \ltwo{\mE_{\mcd_d\cdot\pi_{w_2}}[\phi(s^\prime,a^\prime)^{\top}(\theta^*_{q,w_1} - \theta^*_{q,w_2})x(s,a)\nabla_w\log\pi_{w_2}(a^\prime|s^\prime)]  }\nonumber\\
			&\overset{(i)}{\leq} C_\phi C_x (R_qL_{sc} + L_qC_{sc})\ltwo{w_1-w_2},\label{eq: 28}
		\end{flalign}
		where $(i)$ follows from the Lipschitz property of $\theta^*_{q,w}$ given in \Cref{lemma6}. Combining \cref{eq: 26}, \cref{eq: 27}, \cref{eq: 28} and \cref{eq: 29} yields
		\begin{flalign*}
			\ltwo{\theta^{*\top}_{d_q,w_1} - \theta^{*\top}_{d_q,w_2}} \leq \frac{2C^2_x L_\pi C^\zeta_b + \lambda^{\zeta}_AC_\phi C_x (R_qL_{sc} + L_qC_{sc}) + \lambda^{\zeta}_AC_\phi R_qC_xC_{sc}L_\pi}{(\lambda^{\zeta}_A)^2} \ltwo{w_1-w_2} = L_\zeta\ltwo{w_1-w_2}.
		\end{flalign*}
		Thus we have $\ltwo{\zeta^*_1-\zeta^*_2}\leq L_\zeta\ltwo{w_1-w_2}$, which completes the proof.
	\end{proof}

\begin{lemma}\label{lemma7}
	Define $\Delta_t = \ltwo{\kappa_t-\kappa_t^*}^2 + \ltwo{\xi_t-\xi^*_t}^2 + \ltwo{\zeta_t-\zeta^*_t}^2$. Then, we have
	\begin{flalign*}
		\mE[\Delta_{t+1}|\mf_t] \leq \left( 1 - \frac{1}{2}\varrho\right)\Delta_t + \frac{2}{\varrho}(L^2_\kappa + L^2_\xi + L^2_\zeta)\mE[\ltwo{w_{t+1}-w_t}^2|\mf_t] + \frac{C_5}{(1-\varrho)N},
	\end{flalign*}
	where $\varrho = \frac{1}{4}\min\{\beta_1\lambda_M, \beta_2\lambda_P, \beta_3\lambda_U\}$ and $C_5=C_1 + C_2 + C_4$, where $C_1$, $C_2$ and $C_4$ are defined in \Cref{lemma1}, \Cref{lemma3}, and \Cref{lemma2}.
\end{lemma}
\begin{proof}
	Following from the iteration property developed in \Cref{lemma1}, \Cref{lemma3}, and \Cref{lemma2}, we have
	\begin{flalign}
		\mE\left[\ltwo{\kappa_{t+1} - \kappa^*_t}^2|\mf_t\right] &\leq \left(1-\frac{1}{2}\beta_1\lambda_M \right)\ltwo{\kappa_t - \kappa^*_t}^2 + \frac{C_1}{N},\label{eq: 35}\\
		\mE\left[\ltwo{\xi_{t+1} - \xi^*_t}^2|\mf_t\right] &\leq \left(1-\frac{1}{2}\beta_2\lambda_P \right)\ltwo{\xi_t - \xi^*_t}^2 + \frac{C_2}{N},\label{eq: 36}\\
		\mE\left[\ltwo{\zeta_{t+1} - \zeta^*_t}^2|\mf_t\right] &\leq \left(1 - \frac{1}{4}\beta_3\lambda_U\right)\ltwo{\zeta_t-\zeta^*_t}^2 + C_3\beta_3\ltwo{\theta_{q,t} - \theta^*_{q,t}}^2 + \frac{C_4}{N}\nonumber\\
		&\leq \left(1 - \frac{1}{4}\beta_3\lambda_U\right)\ltwo{\zeta_t-\zeta^*_t}^2 + C_3\beta_3\ltwo{\kappa_t - \kappa^*_t}^2 + \frac{C_4}{N}.\label{eq: 37}
	\end{flalign}
	Summarizing \cref{eq: 35}, \cref{eq: 36} and \cref{eq: 37}, we obtain
	\begin{flalign}
		&\mE\left[\ltwo{\kappa_{t+1}- \kappa^*_t}^2|\mf_t\right] + \mE\left[\ltwo{\xi_{t+1} - \xi^*_t}^2|\mf_t\right] + \mE\left[\ltwo{\zeta_{t+1} - \zeta^*_t}^2|\mf_t\right]\nonumber\\
		&\leq \left(1-\frac{1}{2}\beta_1\lambda_M + C_3\beta_3 \right)\ltwo{\kappa_t - \kappa^*_t}^2 + \left(1-\frac{1}{2}\beta_2\lambda_P \right)\ltwo{\xi_t - \xi^*_t}^2 + \left(1 - \frac{1}{4}\beta_3\lambda_U\right)\ltwo{\zeta_t-\zeta^*_t}^2 + \frac{C_1+C_2+C_4}{N}.\nonumber
	\end{flalign}
	Let $\beta_3\leq \frac{\beta_1\lambda_M}{4C_3}$ and define $\varrho = \frac{1}{4}\min\{\beta_1\lambda_M, \beta_2\lambda_P, \beta_3\lambda_U\}$, $C_5=C_1 + C_2 + C_4$. Then, we have
    \begin{flalign}
    	&\mE\left[\ltwo{\kappa_{t+1}- \kappa^*_t}^2|\mf_t\right] + \mE\left[\ltwo{\xi_{t+1} - \xi^*_t}^2|\mf_t\right] + \mE\left[\ltwo{\zeta_{t+1} - \zeta^*_t}^2|\mf_t\right]\nonumber\\
    	&\leq (1-\varrho)\left(\ltwo{\kappa_t - \kappa^*_t}^2 + \ltwo{\xi_t - \xi^*_t}^2 + \ltwo{\zeta_t-\zeta^*_t}^2\right) + \frac{C_5}{N}.\label{eq: 38}
    \end{flalign}
    We then proceed to investigate the iteration of $\Delta_t$. By Young's inequality, we have
    \begin{flalign}
    	\Delta_{t+1} &\leq \left(\frac{2-\varrho}{2-2\varrho}\right)\left( \mE\left[\ltwo{\kappa_{t+1}- \kappa^*_t}^2|\mf_t\right] + \mE\left[\ltwo{\xi_{t+1} - \xi^*_t}^2|\mf_t\right] + \mE\left[\ltwo{\zeta_{t+1} - \zeta^*_t}^2|\mf_t\right] \right) \nonumber\\
    	&\quad + \left(\frac{2-\varrho}{\varrho}\right)\left( \mE[\ltwo{\kappa^*_{t+1} - \kappa^*_t}^2|\mf_t] + \mE[\ltwo{\xi^*_{t+1} - \xi^*_t}^2|\mf_t] + \mE[\ltwo{\zeta^*_{t+1} - \zeta^*_t}^2|\mf_t] \right)\nonumber\\
    	&\overset{(i)}{\leq} \left( 1 - \frac{1}{2}\varrho\right)\Delta_t + \frac{2}{\varrho}\left( \mE[\ltwo{\kappa^*_{t+1} - \kappa^*_t}^2|\mf_t] + \mE[\ltwo{\xi^*_{t+1} - \xi^*_t}^2|\mf_t] + \mE[\ltwo{\zeta^*_{t+1} - \zeta^*_t}^2|\mf_t] \right) + \frac{C_5}{(1-\varrho)N}\nonumber\\
    	&\leq \left( 1 - \frac{1}{2}\varrho\right)\Delta_t + \frac{2}{\varrho}(L^2_\kappa + L^2_\xi + L^2_\zeta)\mE[\ltwo{w_{t+1}-w_t}^2|\mf_t] + \frac{C_5}{(1-\varrho)N}.
    \end{flalign}
    where $(i)$ follows from \cref{eq: 38}.
\end{proof}

\section{Proof of \Cref{thm2}}
In order to prove the convergence for the DR-Off-PAC algorithm, we first introduce the following Lipschitz property of $J(w)$, which was established in \cite{xu2020improving}.
\begin{proposition}\label{prop: lip}
	Suppose Assumptions \ref{ass2} and \ref{ass3} hold. For any $w,w^\prime\in \mR^d$, we have 
	$\ltwo{\nabla_wJ(w)-\nabla_wJ(w^\prime)}\leq L_J\ltwo{w-w^\prime}, \quad \text{for all } \; w,w^\prime\in \mR^d$, where $L_J=\Theta(1/(1-\gamma))$.
\end{proposition}
The Lipschitz property established in \Cref{prop: lip} is important to establish the local convergence of the gradient-based algorithm.

To proceed the proof of \Cref{thm2}, consider the update in \Cref{algorithm_drpg}. For brevity, we define $G_{\text{DR}}(w_t,\mcm_t)=\frac{1}{N}\sum_{i}G^i_{\text{DR}}(w_t)$, where $\mcm_t$ represents the sample set $\mcb_t\cup \mcb_{t,0}$. Following the $L_J$-Lipschitz property of objective $J(w)$, we have
\begin{flalign}
J(w_{t+1})&\geq J(w_t)+\langle \nabla_w J(w_t), w_{t+1}-w_t \rangle -\frac{L_J}{2}\ltwo{w_{t+1}-w_{t}}^2\nonumber\\
&= J(w_t)+ \alpha \langle \nabla_w J(w_t), G_{\text{DR}}(w_t,\mcm_t) - \nabla_w J(w_t) + \nabla_w J(w_t) \rangle -\frac{L_J\alpha^2}{2}\ltwo{G_{\text{DR}}(w_t,\mcm_t)}^2\nonumber\\
&= J(w_t) + \alpha \ltwo{\nabla_w J(w_t)}^2 + \alpha \langle \nabla_w J(w_t), G_{\text{DR}}(w_t,\mcm_t) - \nabla_w J(w_t) \rangle\nonumber\\
&\quad  - \frac{L_J\alpha^2}{2}\ltwo{G_{\text{DR}}(w_t,\mcm_t)-\nabla_w J(w_t)+\nabla_w J(w_t)}^2\nonumber\\
&\overset{(i)}{\geq} J(w_t) + \Big(\frac{1}{2}\alpha - L_J\alpha^2\Big) \ltwo{\nabla_w J(w_t)}^2 - \Big(\frac{1}{2}\alpha+ L_J\alpha^2\Big) \ltwo{G_{\text{DR}}(w_t,\mcm_t) - \nabla_w J(w_t)}^2, \label{eq: 39}
\end{flalign}
where $(i)$ follows because
\begin{flalign*}
\langle \nabla_w J(w_t), G_{\text{DR}}(w_t,\mcm_t) - \nabla_w J(w_t) \rangle \geq -\frac{1}{2}\ltwo{\nabla_w J(w_t)}^2 - \frac{1}{2}\ltwo{G_{\text{DR}}(w_t,\mcm_t) - \nabla_w J(w_t)}^2,
\end{flalign*}
and
\begin{flalign*}
\ltwo{G_{\text{DR}}(w_t,\mcm_t)-\nabla_w J(w_t)+\nabla_w J(w_t)}^2\leq 2\ltwo{G_{\text{DR}}(w_t,\mcm_t)-\nabla_w J(w_t)}^2 + 2\ltwo{\nabla_w J(w_t)}^2.
\end{flalign*}
Taking the expectation on both sides of \cref{eq: 39} conditioned on $\mf_t$  and rearranging \cref{eq: 39} yield
\begin{flalign}
\Big(\frac{1}{2}\alpha& - L_J\alpha^2\Big) \mE[\ltwo{\nabla_w J(w_t)}^2|\mf_t] \nonumber\\
&\leq \mE[J(w_{t+1})|\mf_t] - J(w_t) + \Big(\frac{1}{2}\alpha+ L_J\alpha^2\Big) \mE[\ltwo{G_{\text{DR}}(w_t,\mcm_t) - \nabla_w J(w_t)}^2|\mf_t].\label{eq: 40}
\end{flalign}
Then, we upper-bound the term $\mE[\ltwo{G_{\text{DR}}(w_t,\mcm_t) - \nabla_w J(w_t)}^2|\mf_t]$ as follows. By definition, we have
\begin{flalign}
	&\ltwo{G_{\text{DR}}(w_t,\mcm_t) - \nabla_w J(w_t)}^2\nonumber\\
	&= \ltwo{G_{\text{DR}}(w_t,\theta_{q,t}, \theta_{\rho,t}, \theta_{\psi,t}, \theta_{d_q,t},\mcm_t) - \nabla_w J(w_t)}^2\nonumber\\
	&\leq 6\ltwo{G_{\text{DR}}(w_t,\theta_{q,t}, \theta_{\rho,t}, \theta_{\psi,t}, \theta_{d_q,t},\mcm_t) - G_{\text{DR}}(w_t,\theta_{q,t}, \theta_{\rho,t}, \theta_{\psi,t}, \theta^*_{d_q,t},\mcm_t)}^2\nonumber\\
	&\quad + 6\ltwo{G_{\text{DR}}(w_t,\theta_{q,t}, \theta_{\rho,t}, \theta_{\psi,t}, \theta^*_{d_q,t},\mcm_t) - G_{\text{DR}}(w_t,\theta_{q,t}, \theta_{\rho,t}, \theta^*_{\psi,t}, \theta^*_{d_q,t},\mcm_t)}^2\nonumber\\
	&\quad + 6\ltwo{G_{\text{DR}}(w_t,\theta_{q,t}, \theta_{\rho,t}, \theta^*_{\psi,t}, \theta^*_{d_q,t},\mcm_t) - G_{\text{DR}}(w_t,\theta^*_{q,t}, \theta_{\rho,t}, \theta^*_{\psi,t}, \theta^*_{d_q,t},\mcm_t)}^2\nonumber\\
	&\quad + 6\ltwo{G_{\text{DR}}(w_t,\theta^*_{q,t}, \theta_{\rho,t}, \theta^*_{\psi,t}, \theta^*_{d_q,t},\mcm_t) - G_{\text{DR}}(w_t,\theta^*_{q,t}, \theta^*_{\rho,t}, \theta^*_{\psi,t}, \theta^*_{d_q,t},\mcm_t)}^2\nonumber\\
	&\quad + 6\ltwo{G_{\text{DR}}(w_t,\theta^*_{q,t}, \theta^*_{\rho,t}, \theta^*_{\psi,t}, \theta^*_{d_q,t},\mcm_t) - \mE[G^i_{\text{DR}}(w_t,\theta^*_{q,t}, \theta^*_{\rho,t}, \theta^*_{\psi,t}, \theta^*_{d_q,t})]}^2\nonumber\\
	&\quad + 6\ltwo{\mE[G^i_{\text{DR}}(w_t,\theta^*_{q,t}, \theta^*_{\rho,t}, \theta^*_{\psi,t}, \theta^*_{d_q,t})] - \nabla_w J(w_t)}^2\nonumber\\
	&\leq \frac{6}{N}\sum_i\ltwo{G^i_{\text{DR}}(w_t,\theta_{q,t}, \theta_{\rho,t}, \theta_{\psi,t}, \theta_{d_q,t}) - G^i_{\text{DR}}(w_t,\theta_{q,t}, \theta_{\rho,t}, \theta_{\psi,t}, \theta^*_{d_q,t})}^2\nonumber\\
	&\quad + \frac{6}{N}\sum_i\ltwo{G^i_{\text{DR}}(w_t,\theta_{q,t}, \theta_{\rho,t}, \theta_{\psi,t}, \theta^*_{d_q,t}) - G^i_{\text{DR}}(w_t,\theta_{q,t}, \theta_{\rho,t}, \theta^*_{\psi,t}, \theta^*_{d_q,t}}^2\nonumber\\
	&\quad + \frac{6}{N}\sum_i\ltwo{G^i_{\text{DR}}(w_t,\theta_{q,t}, \theta_{\rho,t}, \theta^*_{\psi,t}, \theta^*_{d_q,t}) - G^i_{\text{DR}}(w_t,\theta^*_{q,t}, \theta_{\rho,t}, \theta^*_{\psi,t}, \theta^*_{d_q,t})}^2\nonumber\\
	&\quad + \frac{6}{N}\sum_i\ltwo{G^i_{\text{DR}}(w_t,\theta^*_{q,t}, \theta_{\rho,t}, \theta^*_{\psi,t}, \theta^*_{d_q,t}) - G^i_{\text{DR}}(w_t,\theta^*_{q,t}, \theta^*_{\rho,t}, \theta^*_{\psi,t}, \theta^*_{d_q,t})}^2\nonumber\\
	&\quad + 6\ltwo{G_{\text{DR}}(w_t,\theta^*_{q,t}, \theta^*_{\rho,t}, \theta^*_{\psi,t}, \theta^*_{d_q,t},\mcm_t) - \mE[G^i_{\text{DR}}(w_t,\theta^*_{q,t}, \theta^*_{\rho,t}, \theta^*_{\psi,t}, \theta^*_{d_q,t})]}^2\nonumber\\
	&\quad + 6\ltwo{\mE[G^i_{\text{DR}}(w_t,\theta^*_{q,t}, \theta^*_{\rho,t}, \theta^*_{\psi,t}, \theta^*_{d_q,t})] - \nabla_w J(w_t)}^2.\label{eq: 41}
\end{flalign}
We next bound each term in \cref{eq: 41}. For $\ltwo{G^i_{\text{DR}}(w_t,\theta_{q,t}, \theta_{\rho,t}, \theta_{\psi,t}, \theta_{d_q,t}) - G^i_{\text{DR}}(w_t,\theta_{q,t}, \theta_{\rho,t}, \theta_{\psi,t}, \theta^*_{d_q,t})}$, we proceed as follows:
\begin{flalign}
	&\ltwo{G^i_{\text{DR}}(w_t,\theta_{q,t}, \theta_{\rho,t}, \theta_{\psi,t}, \theta_{d_q,t}) - G^i_{\text{DR}}(w_t,\theta_{q,t}, \theta_{\rho,t}, \theta_{\psi,t}, \theta^*_{d_q,t})}\nonumber\\
	&\leq (1-\gamma)\ltwo{\mE_{\pi_{w_t}}\left[ x(s_{0,i}, a_{0,i}) \right]^\top (\theta_{d_q,t}-\theta^*_{d_q,t})} + \ltwo{\hat{\rho}_{\pi_{w_t}}(s_i,a_i)x(s_i,a_i)(\theta^*_{d_q,t}-\theta_{d_q,t})} \nonumber\\
	&\quad + \gamma\ltwo{\hat{\rho}_{\pi_{w_t}}(s_i,a_i)\mE_{\pi_{w_t}}[x(s^\prime_i,a^\prime_i)]^\top(\theta_{d_q,t}-\theta^*_{d_q,t})}\nonumber\\
	&\leq \left(1+ 2 R_\rho C_\phi  \right)C_x\ltwo{\theta^*_{d_q,t}-\theta_{d_q,t}} \nonumber\\
	&= C_6\ltwo{\theta^*_{d_q,t}-\theta_{d_q,t}},\label{eq: 42}
\end{flalign}
where $C_6=\left(1+ 2 R_\rho C_\phi  \right)$.

For $\ltwo{G^i_{\text{DR}}(w_t,\theta_{q,t}, \theta_{\rho,t}, \theta_{\psi,t}, \theta^*_{d_q,t}) - G^i_{\text{DR}}(w_t,\theta_{q,t}, \theta_{\rho,t}, \theta^*_{\psi,t}, \theta^*_{d_q,t}}$, we proceed as follows
\begin{flalign}
	&\ltwo{G^i_{\text{DR}}(w_t,\theta_{q,t}, \theta_{\rho,t}, \theta_{\psi,t}, \theta^*_{d_q,t}) - G^i_{\text{DR}}(w_t,\theta_{q,t}, \theta_{\rho,t}, \theta^*_{\psi,t}, \theta^*_{d_q,t}}\nonumber\\
	&\leq \ltwo{\hat{\rho}_{\pi_{w_t}}(s_i,a_i)\varphi(s_i,a_i)^\top(\theta_{\psi,t} - \theta^*_{\psi,t})\left(r(s_i,a_i,s^\prime_i) + \hat{Q}_{\pi_{w_t}}(s_i,a_i) - \gamma\mE_{\pi_{w_t}}[\hat{Q}_{\pi_{w_t}}(s^\prime_i,a^\prime_i)]  \right)  }\nonumber\\
	&\leq R_\rho C_\varphi(r_{\max} + 2R_qC_\phi)\ltwo{\theta_{\psi,t} - \theta^*_{\psi,t}} \nonumber\\
	&= C_7\ltwo{\theta_{\psi,t} - \theta^*_{\psi,t}},\label{eq: 43}
\end{flalign}
where $C_7=R_\rho C_\varphi(r_{\max} + 2R_qC_\phi)$.

For $\ltwo{G^i_{\text{DR}}(w_t,\theta_{q,t}, \theta_{\rho,t}, \theta^*_{\psi,t}, \theta^*_{d_q,t}) - G^i_{\text{DR}}(w_t,\theta^*_{q,t}, \theta_{\rho,t}, \theta^*_{\psi,t}, \theta^*_{d_q,t})}$, we proceed as follows
\begin{flalign}
	&\ltwo{G^i_{\text{DR}}(w_t,\theta_{q,t}, \theta_{\rho,t}, \theta^*_{\psi,t}, \theta^*_{d_q,t}) - G^i_{\text{DR}}(w_t,\theta^*_{q,t}, \theta_{\rho,t}, \theta^*_{\psi,t}, \theta^*_{d_q,t})}\nonumber\\
	&\leq (1-\gamma)\ltwo{\mE_{\pi_{w_t}}\left[ \phi(s_{0,i},a_{0,i})^\top(\theta_{q,t} - \theta^*_{q,t})\nabla_w\log\pi_{w_t}(s_{0,i},a_{0,i}) \right]} \nonumber\\
	&\quad + \ltwo{\hat{\rho}_{\pi_{w_t}}(s_i,a_i)\varphi(s_i,a_i)^\top\theta^*_{\psi,t}\phi(s_i,a_i)^\top(\theta_{q,t} - \theta^*_{q,t})}\nonumber\\
	&\quad + \gamma\ltwo{\hat{\rho}_{\pi_{w_t}}(s_i,a_i)\varphi(s_i,a_i)^\top\theta^*_{\psi,t}\mE_{\pi_{w_t}}[\phi(s^\prime_i,a^\prime_i)]^\top(\theta_{q,t} - \theta^*_{q,t})}\nonumber\\
	&\leq \left[ (1-\gamma)C_\phi C_{sc} + (1+\gamma)R_\rho C^2_\phi C_\varphi R_\psi \right] \ltwo{\theta_{q,t} - \theta^*_{q,t}} \nonumber\\
	&= C_8\ltwo{\theta_{q,t} - \theta^*_{q,t}},\label{eq: 44}
\end{flalign}
where $C_8 =  (1-\gamma)C_\phi C_{sc} + (1+\gamma)R_\rho C^2_\phi C_\varphi R_\psi $.

For $\ltwo{G^i_{\text{DR}}(w_t,\theta^*_{q,t}, \theta_{\rho,t}, \theta^*_{\psi,t}, \theta^*_{d_q,t}) - G^i_{\text{DR}}(w_t,\theta^*_{q,t}, \theta^*_{\rho,t}, \theta^*_{\psi,t}, \theta^*_{d_q,t})}$, we proceed as follows
\begin{flalign}
	&\ltwo{G^i_{\text{DR}}(w_t,\theta^*_{q,t}, \theta_{\rho,t}, \theta^*_{\psi,t}, \theta^*_{d_q,t}) - G^i_{\text{DR}}(w_t,\theta^*_{q,t}, \theta^*_{\rho,t}, \theta^*_{\psi,t}, \theta^*_{d_q,t})}\nonumber\\
	&\leq \ltwo{\phi(s_i,a_i)^\top(\theta_{\rho,t} - \theta^*_{\rho,t}) \varphi(s_i,a_i)^\top\theta^*_{\psi,t} \left(r(s_i,a_i,s^\prime_i) - \phi(s_i,a_i)^\top\theta^*_{q,t} + \gamma\mE[\phi(s^\prime_i,a^\prime_i)]^\top\theta^*_{q,t} \right) }\nonumber\\
	&\quad + \ltwo{\phi(s_i,a_i)^\top(\theta_{\rho,t} - \theta^*_{\rho,t})\left(-x(s_i,a_i)^\top\theta^*_{d_q,t} + \gamma\mE_{\pi_{w_t}}[\phi(s^\prime_i,a^\prime_i)^\top\theta^*_{q,t}\nabla_w\log\pi_{w_t}(s^\prime_i,a^\prime_i) + x(s^\prime_i,a^\prime_i)^\top\theta^*_{d_q,t}] \right)}\nonumber\\
	&\leq \left[ C_\phi C_\varphi R_\psi(r_{\max} + (1+\gamma)C_\phi R_q) + C_\phi(C_xR_{d_q} + \gamma (C_\phi R_q C_{sc} + C_xR_{d_q}) ) \right]\ltwo{\theta_{\rho,t} - \theta^*_{\rho,t}} \nonumber\\
	&= C_9\ltwo{\theta_{\rho,t} - \theta^*_{\rho,t}},\label{eq: 45}
\end{flalign}
where $C_9 = C_\phi C_\varphi R_\psi(r_{\max} + (1+\gamma)C_\phi R_q) + C_\phi(C_xR_{d_q} + \gamma (C_\phi R_q C_{sc} + C_xR_{d_q}) )$.

For $\ltwo{G_{\text{DR}}(w_t,\theta^*_{q,t}, \theta^*_{\rho,t}, \theta^*_{\psi,t}, \theta^*_{d_q,t},\mcm_t) - \mE[G^i_{\text{DR}}(w_t,\theta^*_{q,t}, \theta^*_{\rho,t}, \theta^*_{\psi,t}, \theta^*_{d_q,t})]}^2$, note that for all $i$, we have
\begin{flalign}
	&\ltwo{G^i_{\text{DR}}(w_t,\theta^*_{q,t}, \theta^*_{\rho,t}, \theta^*_{\psi,t}, \theta^*_{d_q,t})}\nonumber\\
	& \leq (1-\gamma)\ltwo{\mE_{\pi_w}[\phi_{0,i}^\top\theta^*_{q,t}\nabla_w\log\pi_w(s_{0,i},a_{0,i}) + x_{0,i}^\top\theta^*_{d_q,t}]} + \ltwo{\psi_{i}^\top\theta^*_{\psi,t}\left(r(s_i,a_i,s^\prime_i) - \phi_{i}^\top\theta^*_{q,t} + \gamma \mE_{\pi_{w_t}} [\phi^{\prime\top}_{i}\theta^*_{q,t}] \right)} \nonumber\\
	& \quad + \ltwo{\phi_{i}^\top\theta^*_{\rho,t}(- x_{i}^\top\theta^*_{d_q,t} + \gamma \mE_{\pi_w}[\phi_{i}^\top\theta^*_{q,t}\nabla_w\log\pi_w(s_{t,i},a_{t,i}) + x_{i}^\top\theta^*_{d_q,t}])}\nonumber\\
	&\leq (1-\gamma)(C_\phi R_q C_{sc} + C_xR_{d_q}) + C_\psi R_\psi(r_{\max} + (1+\gamma) C_\phi R_q) + C_\phi R_\rho (C_xR_{d_q} + \gamma C_\phi R_q C_{sc} + \gamma C_x R_{d-q} ).\nonumber
\end{flalign}
Let $C_{10} = (1-\gamma)(C_\phi R_q C_{sc} + C_xR_{d_q}) + C_\psi R_\psi(r_{\max} + (1+\gamma) C_\phi R_q) + C_\phi R_\rho (C_xR_{d_q} + \gamma C_\phi R_q C_{sc} + \gamma C_x R_{d-q} )$. Then following the steps similar to those in \cref{eq: 11} and \cref{eq: 12}, we obtain
\begin{flalign}
	\mE\left[\ltwo{G_{\text{DR}}(w_t,\theta^*_{q,t}, \theta^*_{\rho,t}, \theta^*_{\psi,t}, \theta^*_{d_q,t},\mcm_t) - \mE[G^i_{\text{DR}}(w_t,\theta^*_{q,t}, \theta^*_{\rho,t}, \theta^*_{\psi,t}, \theta^*_{d_q,t})]}^2\Big|\mf_t\right] \leq \frac{4 C^2_{10}}{N}.\label{eq: 46}
\end{flalign}
Finally, consider the term $\ltwo{\mE[G^i_{\text{DR}}(w_t,\theta^*_{q,t}, \theta^*_{\rho,t}, \theta^*_{\psi,t}, \theta^*_{d_q,t})] - \nabla_w J(w_t)}^2$. Following the steps similar to those in proving \Cref{thm1}, we obtain
\begin{flalign}
	&\mE[G^i_{\text{DR}}(w_t,\theta^*_{q,t}, \theta^*_{\rho,t}, \theta^*_{\psi,t}, \theta^*_{d_q,t})] - \nabla_w J(w_t)\nonumber\\
	& = \mE_{\mcd}[(\hat{\rho}_{\pi_{w_t}}(s,a, \theta^*_{\rho,t}) - {\rho}_{\pi_{w_t}}(s,a))(-\hat{d}^q_{\pi_{w_t}}(s,a,\theta^*_{d_q,t}) + {d}^q_{\pi_{w_t}}(s,a))] \nonumber\\
	&\quad + \mE_{\mcd}[(-\hat{d}^\rho_{\pi_{w_t}}(s,a,\theta^*_{\rho,t}, \theta^*_{\psi,t}) + {d}^\rho_{\pi_{w_t}}(s,a))(-\hat{Q}_{\pi_{w_t}}(s,a,\theta^*_{q,t}) + {Q}_{\pi_{w_t}}(s,a))]\nonumber\\
	&\quad + \gamma\mE_{\mcd}[(\hat{\rho}_{\pi_{w_t}}(s,a,\theta^*_{\rho,t}) - {\rho}_{\pi_{w_t}}(s,a)) \mE_{\pi_{w_t}}[\hat{d}^q_{\pi_{w_t}}(s^\prime,a^\prime,\theta^*_{d_q,t}) - {d}^q_{\pi_{w_t}}(s^\prime,a^\prime)] ] \nonumber\\
	&\quad + \gamma\mE_{\mcd}[(\hat{\rho}_{\pi_{w_t}}(s,a,\theta^*_{\rho,t}) - {\rho}_{\pi_{w_t}}(s,a)) \mE_{\pi_{w_t}}[(\hat{Q}_{\pi_{w_t}}(s^\prime,a^\prime,\theta^*_{q,t}) - {Q}_{\pi_{w_t}}(s^\prime, a^\prime))\nabla_w\log\pi_{w_t}(s^\prime,a^\prime)] ] \nonumber\\
	&\quad + \gamma\mE_{\mcd}[(\hat{d}^\rho_{\pi_{w_t}}(s,a,\theta^*_{\rho,t},\theta^*_{\psi,t}) - {d}^\rho_{\pi_{w_t}}(s,a))\mE_{\pi_{w_t}}[\hat{Q}_{\pi_{w_t}}(s^\prime,a^\prime,\theta^*_{q,t}) - {Q}_{\pi_{w_t}}(s^\prime,a^\prime)]]\nonumber\\
	&\leq \sqrt{ \mE_{\mcd}[(\hat{\rho}_{\pi_{w_t}}(s,a, \theta^*_{\rho,t}) - {\rho}_{\pi_{w_t}}(s,a))^2]} \sqrt{\mE_\mcd[(\hat{d}^q_{\pi_{w_t}}(s,a,\theta^*_{d_q,t}) - {d}^q_{\pi_{w_t}}(s,a))^2]}\nonumber\\
	&\quad + \sqrt{\mE_{\mcd}[(\hat{d}^\rho_{\pi_{w_t}}(s,a,\theta^*_{\rho,t}, \theta^*_{\psi,t}) - {d}^\rho_{\pi_{w_t}}(s,a))^2]}\sqrt{\mE_\mcd[(\hat{Q}_{\pi_{w_t}}(s,a,\theta^*_{q,t}) - {Q}_{\pi_{w_t}}(s,a))^2]}\nonumber\\
	&\quad + \sqrt{ \mE_{\mcd}[(\hat{\rho}_{\pi_{w_t}}(s,a, \theta^*_{q,t}) - {\rho}_{\pi_{w_t}}(s,a))^2]}\sqrt{\mE_{\mcd_d\cdot\pi_{w_t}}[(\hat{d}^q_{\pi_{w_t}}(s^\prime,a^\prime,\theta^*_{d_q,t}) - {d}^q_{\pi_{w_t}}(s^\prime,a^\prime))^2]}\nonumber\\
	&\quad + C_{sc}\sqrt{ \mE_{\mcd}[(\hat{\rho}_{\pi_{w_t}}(s,a, \theta^*_{q,t}) - {\rho}_{\pi_{w_t}}(s,a))^2]}\sqrt{\mE_{\mcd_d\cdot\pi_{w_t}}[(\hat{Q}_{\pi_{w_t}}(s^\prime,a^\prime,\theta^*_{q,t}) - {Q}_{\pi_{w_t}}(s^\prime,a^\prime))^2]}\nonumber\\
	&\quad + \sqrt{\mE_{\mcd}[(\hat{d}^\rho_{\pi_{w_t}}(s,a,\theta^*_{\rho,t}, \theta^*_{\psi,t}) - {d}^\rho_{\pi_{w_t}}(s,a))^2]}\sqrt{\mE_{\mcd_d\cdot\pi_{w_t}}[(\hat{Q}_{\pi_{w_t}}(s^\prime,a^\prime,\theta^*_{q,t}) - {Q}_{\pi_{w_t}}(s^\prime,a^\prime))^2]}\nonumber\\
	&\leq 2\epsilon_\rho \epsilon_{d_q} + 2\epsilon_{d_\rho}\epsilon_q + C_{sc}\epsilon_\rho \epsilon_q.\label{eq: 47}
\end{flalign}
Recall that we define
\begin{flalign*}
	\epsilon_\rho &= \max_w\sqrt{ \mE_{\mcd}[(\hat{\rho}_{\pi_{w}}(s,a, \theta^*_{\rho,w}) - {\rho}_{\pi_{w}}(s,a))^2]}\nonumber\\
	\epsilon_{d_\rho} &=  \max_w\sqrt{ \mE_{\mcd}[(\hat{d}^\rho_{\pi_{w}}(s,a,\theta^*_{\rho,w}, \theta^*_{\psi,w}) - {d}^\rho_{\pi_{w}}(s,a))^2]}\nonumber\\
	\epsilon_q & = \max\left\{ \max_w\sqrt{\mE_{\mcd}[(\hat{Q}_{\pi_{w}}(s,a,\theta^*_{q,w}) - {Q}_{\pi_{w}}(s,a))^2]}, \max_w\sqrt{\mE_{\mcd_d\cdot\pi_{w}}[(\hat{Q}_{\pi_{w}}(s^\prime,a^\prime,\theta^*_{q,w}) - {Q}_{\pi_{w}}(s^\prime,a^\prime))^2]} \right\}\nonumber\\
	\epsilon_{d_q} & = \max\left\{ \max_w\sqrt{\mE_{\mcd}[(\hat{d}^q_{\pi_{w}}(s,a,\theta^*_{d_q,w}) - {d}^q_{\pi_{w}}(s,a))^2]}, \max_w\sqrt{\mE_{\mcd_d\cdot\pi_{w}}[(\hat{d}^q_{\pi_{w}}(s^\prime,a^\prime,\theta^*_{d_q,w}) - {d}^q_{\pi_{w}}(s^\prime,a^\prime))^2]} \right\},
\end{flalign*}
where 
\begin{flalign*}
	\hat{\rho}_{\pi_{w}}(s,a, \theta^*_{q,w})&=\phi(s,a)^\top\theta^*_{\rho,w},\nonumber\\
	\hat{d}^\rho_{\pi_{w}}(s,a,\theta^*_{\rho,w}, \theta^*_{\psi,w}) &= \phi(s,a)^\top\theta^*_{\rho,w} \varphi(s,a)^\top\theta^*_{\psi,w},\nonumber\\
	\hat{Q}_{\pi_{w}}(s,a,\theta^*_{q,w})&=\phi(s,a)^\top\theta^*_{q,w},\nonumber\\
	\hat{d}^q_{\pi_{w}}(s^\prime,a^\prime,\theta^*_{d_q,w})&=x(s,a)^\top\theta^*_{d_q,w}.
\end{flalign*}
Then, \cref{eq: 47} implies that 
\begin{flalign}
	\ltwo{\mE[G^i_{\text{DR}}(w_t,\theta^*_{q,t}, \theta^*_{\rho,t}, \theta^*_{\psi,t}, \theta^*_{d_q,t})] - \nabla_w J(w_t)}^2\leq 6\epsilon^2_\rho \epsilon^2_{d_q} + 6\epsilon^2_{d_\rho}\epsilon^2_q + 3C_{sc}\epsilon^2_\rho \epsilon^2_q.\label{eq: 48}
\end{flalign}
Substituting \cref{eq: 42}, \cref{eq: 43}, \cref{eq: 44}, \cref{eq: 45}, \cref{eq: 46} and \cref{eq: 48} into \cref{eq: 41} yields
\begin{flalign}
	&\mE\left[\ltwo{G_{\text{DR}}(w_t,\mcm_t) - \nabla_w J(w_t)}^2|\mf_t\right]\nonumber\\
	&\leq 6C^2_6\ltwo{\theta^*_{d_q,t}-\theta_{d_q,t}}^2 + 6C^2_7\ltwo{\theta_{\psi,t} - \theta^*_{\psi,t}}^2 + 6C^2_8\ltwo{\theta_{q,t} - \theta^*_{q,t}} + 6C^2_9\ltwo{\theta_{\rho,t} - \theta^*_{\rho,t}}^2 + \frac{24 C^2_{10}}{N} \nonumber\\
	&\quad + 36\epsilon^2_\rho \epsilon^2_{d_q} + 36\epsilon^2_{d_\rho}\epsilon^2_q + 18C_{sc}\epsilon^2_\rho \epsilon^2_q\nonumber\\
	&\leq C_{11} \left( \ltwo{\theta^*_{d_q,t}-\theta_{d_q,t}}^2 + \ltwo{\theta_{\psi,t} - \theta^*_{\psi,t}}^2 + \ltwo{\theta_{q,t} - \theta^*_{q,t}} + \ltwo{\theta_{\rho,t} - \theta^*_{\rho,t}}^2 \right) + \frac{24 C^2_{10}}{N}\nonumber\\
	&\quad + C_{12}(\epsilon^2_\rho \epsilon^2_{d_q} + \epsilon^2_{d_\rho}\epsilon^2_q + \epsilon^2_\rho \epsilon^2_q)\nonumber\\
	&\leq C_{11} \Delta_t + \frac{24 C^2_{10}}{N} + C_{12}(\epsilon^2_\rho \epsilon^2_{d_q} + \epsilon^2_{d_\rho}\epsilon^2_q + \epsilon^2_\rho \epsilon^2_q),\label{eq: 49}
\end{flalign}
where $C_{11} = 6\max\{C^2_6, C^2_7, C^2_8, C^2_9\}$, and $C_{12}=\max\{36, 18C_{sc}\}$. Substituting \cref{eq: 49} into \cref{eq: 40} yields
\begin{flalign}
\Big(\frac{1}{2}\alpha& - L_J\alpha^2\Big)\mE[\ltwo{\nabla_w J(w_t)}^2|\mf_t] \nonumber\\
&\leq \mE[J(w_{t+1})|\mf_t] - J(w_t) + \Big(\frac{1}{2}\alpha+ L_J\alpha^2\Big) C_{11}\Delta_t + \Big(\frac{1}{2}\alpha+ L_J\alpha^2\Big)\frac{24 C^2_{10}}{N} \nonumber\\
&\quad + \Big(\frac{1}{2}\alpha+ L_J\alpha^2\Big)C_{12}(\epsilon^2_\rho \epsilon^2_{d_q} + \epsilon^2_{d_\rho}\epsilon^2_q + \epsilon^2_\rho \epsilon^2_q).\label{eq: 50}
\end{flalign}
Taking the expectation on both sides of \cref{eq: 50} and taking the summation over $t=0\cdots T-1$ yield
\begin{flalign}
	\Big(\frac{1}{2}\alpha& - L_J\alpha^2\Big)\sum_{t=0}^{T-1}\mE[\ltwo{\nabla_w J(w_t)}^2] \nonumber\\
	&\leq \mE[J(w_{T})] - J(w_0) + \Big(\frac{1}{2}\alpha+ L_J\alpha^2\Big) C_{11} \sum_{t=0}^{T-1}\mE[\Delta_t] + \Big(\frac{1}{2}\alpha+ L_J\alpha^2\Big)\frac{24 C^2_{10} T}{N} \nonumber\\
	&\quad + \Big(\frac{1}{2}\alpha+ L_J\alpha^2\Big)C_{12}(\epsilon^2_\rho \epsilon^2_{d_q} + \epsilon^2_{d_\rho}\epsilon^2_q + \epsilon^2_\rho \epsilon^2_q)T.\label{eq: 51}
\end{flalign}
We then proceed to bound the term $\sum_{t=0}^{T-1}\mE[\Delta_t]$. \Cref{lemma7} implies that
\begin{flalign}
\mE&[\Delta_{t+1}|\mf_t] \nonumber\\
&\leq \left( 1 - \frac{1}{2}\varrho\right)\Delta_t + \frac{2L^2_{fix}}{\varrho}\mE[\ltwo{w_{t+1}-w_t}^2|\mf_t] + \frac{C_5}{(1-\varrho)N}\nonumber\\
&= \left( 1 - \frac{1}{2}\varrho\right)\Delta_t + \frac{2L^2_{fix}}{\varrho}\alpha^2\mE[\ltwo{G_{\text{DR}}(w_t,\mcm_t)}^2|\mf_t] + \frac{C_5}{(1-\varrho)N}\nonumber\\
&\leq \left( 1 - \frac{1}{2}\varrho\right) \Delta_t + \frac{4L^2_{fix}}{\varrho}\alpha^2\mE[\ltwo{G_{\text{DR}}(w_t,\mcm_t) - \nabla_w J(w_t)}^2|\mf_t] + \frac{4L^2_{fix}}{\varrho}\alpha^2\mE[\ltwo{\nabla_w J(w_t)}^2|\mf_t] + \frac{C_5}{(1-\varrho)N}\nonumber\\
&\overset{(i)}{\leq} \left( 1 - \frac{1}{2}\varrho + \frac{4C_{11}L^2_{fix}\alpha^2}{\varrho}\right) \Delta_t + \frac{4L^2_{fix}}{\varrho}\alpha^2\mE[\ltwo{\nabla_w J(w_t)}^2|\mf_t] + \left[\frac{96L^2_{fix} C^2_{10}\alpha^2}{\varrho} + \frac{C_5}{1-\varrho}\right]\frac{1}{N} \nonumber\\
&\quad + \frac{4C_{12}L^2_{fix}\alpha^2}{\varrho}(\epsilon^2_\rho \epsilon^2_{d_q} + \epsilon^2_{d_\rho}\epsilon^2_q + \epsilon^2_\rho \epsilon^2_q)\nonumber\\
&\overset{(ii)}{\leq} \left( 1 - \frac{1}{4}\varrho \right) \Delta_t + C_{12}\alpha^2\mE[\ltwo{\nabla_w J(w_t)}^2|\mf_t] + \frac{C_{13}}{N} + C_{14}(\epsilon^2_\rho \epsilon^2_{d_q} + \epsilon^2_{d_\rho}\epsilon^2_q + \epsilon^2_\rho \epsilon^2_q),\label{eq: 52}
\end{flalign}
where $L^2_{fix} = (L^2_\kappa + L^2_\xi + L^2_\zeta)$, $(i)$ follows from \cref{eq: 49}, $(ii)$ follows from the fact that $C_{12} = \frac{4L^2_{fix}}{\varrho}$, $C_{13} = \frac{96L^2_{fix} C^2_{10}\alpha^2}{\varrho} + \frac{C_5}{1-\varrho}$ and $C_{14} = \frac{4C_{12}L^2_{fix}\alpha^2}{\varrho}$, and for small enough $\alpha$, we have $\alpha\leq \frac{\varrho}{4\sqrt{C_{11}}L_{fix}}$. Taking the expectation on both sides of \cref{eq: 52} and applying it iteratively yield
\begin{flalign}
	\mE[\Delta_{t}] &\leq  \left( 1 - \frac{1}{4}\varrho \right)^{t}\Delta_0 + C_{12}\alpha^2\sum_{i=0}^{t-1}\left( 1 - \frac{1}{4}\varrho \right)^{t-1-i}\mE[\ltwo{\nabla_w J(w_i)}^2] + \frac{C_{13}}{N}\sum_{i=0}^{t-1}\left( 1 - \frac{1}{4}\varrho \right)^{t-1-i}\nonumber\\
	&\quad + C_{14}(\epsilon^2_\rho \epsilon^2_{d_q} + \epsilon^2_{d_\rho}\epsilon^2_q + \epsilon^2_\rho \epsilon^2_q)\sum_{i=0}^{t-1}\left( 1 - \frac{1}{4}\varrho \right)^{t-1-i}.\label{eq: 53}
\end{flalign}
Taking the summation on \cref{eq: 53} over $t=0,\cdots,T-1$ yields
\begin{flalign}
	\sum_{t=0}^{T-1}\mE[\Delta_t]&\leq \Delta_0\sum_{t=0}^{T-1}\left( 1 - \frac{1}{4}\varrho \right)^{t} + C_{12}\alpha^2\sum_{t=0}^{T-1}\sum_{i=0}^{t-1}\left( 1 - \frac{1}{4}\varrho \right)^{t-1-i}\mE[\ltwo{\nabla_w J(w_i)}^2] + \frac{C_{13}}{N}\sum_{t=0}^{T-1}\sum_{i=0}^{t-1}\left( 1 - \frac{1}{4}\varrho \right)^{t-1-i}\nonumber\\
	&\quad + C_{14}(\epsilon^2_\rho \epsilon^2_{d_q} + \epsilon^2_{d_\rho}\epsilon^2_q + \epsilon^2_\rho \epsilon^2_q)\sum_{t=0}^{T-1}\sum_{i=0}^{t-1}\left( 1 - \frac{1}{4}\varrho \right)^{t-1-i}\nonumber\\
	&\leq \frac{4}{\varrho}\Delta_0 + \frac{4C_{12}\alpha^2}{\varrho}\sum_{t=0}^{T-1}\mE[\ltwo{\nabla_w J(w_t)}^2] + \frac{4C_{13}T}{\varrho N} + \frac{4C_{14}T}{\varrho}(\epsilon^2_\rho \epsilon^2_{d_q} + \epsilon^2_{d_\rho}\epsilon^2_q + \epsilon^2_\rho \epsilon^2_q).\label{eq: 54}
\end{flalign}
Substituting \cref{eq: 54} into \cref{eq: 51} yields
\begin{flalign}
\Big(\frac{1}{2}\alpha& - L_J\alpha^2\Big)\sum_{t=0}^{T-1}\mE[\ltwo{\nabla_w J(w_t)}^2] \nonumber\\
&\leq \mE[J(w_{T})] - J(w_0) + \Big(\frac{1}{2}\alpha+ L_J\alpha^2\Big) \frac{4C_{11}}{\varrho}\Delta_0 + \frac{4C_{12}\alpha^2}{\varrho}\Big(\frac{1}{2}\alpha+ L_J\alpha^2\Big)\sum_{t=0}^{T-1}\mE[\ltwo{\nabla_w J(w_i)}^2]  \nonumber\\
&\quad + \Big(\frac{1}{2}\alpha+ L_J\alpha^2\Big)\frac{24 \varrho C^2_{10} T + 4C_{13}T}{\varrho N} + \Big(\frac{1}{2}\alpha+ L_J\alpha^2\Big)\left(C_{12} + \frac{4C_{14}T}{\varrho}\right)(\epsilon^2_\rho \epsilon^2_{d_q} + \epsilon^2_{d_\rho}\epsilon^2_q + \epsilon^2_\rho \epsilon^2_q)T,
\end{flalign}
which implies
\begin{flalign}
\Big(\frac{1}{2}\alpha& - L_J\alpha^2 - \frac{4C_{12}\alpha^3}{\varrho}\Big(\frac{1}{2}+ L_J\alpha\Big)\Big)\sum_{t=0}^{T-1}\mE[\ltwo{\nabla_w J(w_t)}^2] \nonumber\\
&\leq \mE[J(w_{T})] - J(w_0) + \Big(\frac{1}{2}\alpha+ L_J\alpha^2\Big) \frac{4C_{11}}{\varrho}\Delta_0 + \Big(\frac{1}{2}\alpha+ L_J\alpha^2\Big)\frac{24 \varrho C^2_{10} T + 4C_{13}T}{\varrho N} \nonumber\\
&\quad + \Big(\frac{1}{2}\alpha+ L_J\alpha^2\Big)\left(C_{12} + \frac{4C_{14}T}{\varrho}\right)(\epsilon^2_\rho \epsilon^2_{d_q} + \epsilon^2_{d_\rho}\epsilon^2_q + \epsilon^2_\rho \epsilon^2_q)T.\label{eq: 55}
\end{flalign}
For small enough $\alpha$, we can guarantee that $\frac{1}{2}\alpha - L_J\alpha^2 - \frac{4C_{12}\alpha^3}{\varrho}\Big(\frac{1}{2}+ L_J\alpha\Big)>0$. Dividing both sides of \cref{eq: 55} by $T\Big(\frac{1}{2}\alpha - L_J\alpha^2 - \frac{4C_{12}\alpha^3}{\varrho}\Big(\frac{1}{2}+ L_J\alpha\Big)\Big)$, we obtain
\begin{flalign*}
	\mE[\ltwo{\nabla_w J(w_{\hat{T}})}^2]=\frac{1}{T}\sum_{t=0}^{T-1}\mE[\ltwo{\nabla_w J(w_t)}^2]\leq \Theta\left(\frac{1}{T}\right) + \Theta\left(\frac{1}{N}\right) + \Theta(\epsilon^2_\rho \epsilon^2_{d_q} + \epsilon^2_{d_\rho}\epsilon^2_q + \epsilon^2_\rho \epsilon^2_q).
\end{flalign*}
Note that $	\mE[\ltwo{\nabla_w J(w_{\hat{T}})}] \leq\sqrt{\mE[\ltwo{\nabla_w J(w_{\hat{T}})}^2]}$ and $\sqrt{\sum_i a_i}\leq \sum_i \sqrt{a_i}$ for $a_i \geq 0$. We obtain
\begin{flalign*}
	\mE[\ltwo{\nabla_w J(w_{\hat{T}})}] \leq  \Theta\left(\frac{1}{\sqrt{T}}\right) + \Theta\left(\frac{1}{\sqrt{N}}\right) + \Theta(\epsilon_\rho \epsilon_{d_q} + \epsilon_{d_\rho}\epsilon_q + \epsilon_\rho \epsilon_q),
\end{flalign*}
which completes the proof.

\section{Proof of \Cref{thm3}}
%We first provide the following propertywhich is proved by \cite{agarwal2019optimality} to characterized the global convergence of NPG update in the general function approximation setting, and is later adopted by \cite{liu2020improved} to connected with PG update with NPG update.
The following proposition can be directly obtained from Corollary 6.10. in \cite{agarwal2019optimality}.
\begin{proposition}\label{prop1}
	Consider the DR-Off-PAC update given in \Cref{algorithm_drpg}. Suppose \Cref{ass: fisher} holds. Let $w^*_t = F^{-1}(w_t)\nabla_w J(w_t)$ be the exact NPG update direction at $w_t$. Then, we have
	\begin{flalign}
		&J(\pi^*) - J(w_{\hat{T}})\nonumber\\
		&\leq \frac{{\epsilon_{approx}}}{1-\gamma} + \frac{1}{\alpha T}\mE_{\nu_{\pi^*}}[\text{KL}(\pi^*(\cdot|s)||\pi_{w_0}(\cdot|s))] + \frac{L_{sc}\alpha}{2T}\sum_{t=0}^{T-1}\ltwo{G_{\text{DR}}(w_t,\mcm_t)}^2 + \frac{C_{sc}}{T}\sum_{t=0}^{T-1}\ltwo{G_{\text{DR}}(w_t,\mcm_t) - w^*_t}.\label{eq: 56}
	\end{flalign}
\end{proposition}
\Cref{prop1} indicates that, as long as the DR-Off-PAC update is close enough to the exact NPG update, then DP-Off-PAC is guaranteed to converge to a neighbourhood of the global optimal $J(\pi^*)$ with a $\left(  \frac{{\epsilon_{approx}}}{1-\gamma}\right)$-level gap. We then proceed to prove \Cref{thm3}.
\begin{proof}
	We start with \cref{eq: 56} as follows:
	\begin{flalign}
		&J(\pi^*) - \mE[J(w_{\hat{T}})]\nonumber\\
		&\leq \frac{{\epsilon_{approx}}}{1-\gamma} + \frac{1}{\alpha T}\mE_{\nu_{\pi^*}}[\text{KL}(\pi^*(\cdot|s)||\pi_{w_0}(\cdot|s))] + \frac{L_{sc}\alpha}{2T}\sum_{t=0}^{T-1}\mE[\ltwo{G_{\text{DR}}(w_t,\mcm_t)}^2] + \frac{C_{sc}}{T}\sum_{t=0}^{T-1}\mE\left[\ltwo{G_{\text{DR}}(w_t,\mcm_t) - w^*_t}\right]\nonumber\\
		&\leq \frac{{\epsilon_{approx}}}{1-\gamma} + \frac{1}{\alpha T}\mE_{\nu_{\pi^*}}[\text{KL}(\pi^*(\cdot|s)||\pi_{w_0}(\cdot|s))] + \frac{L_{sc}\alpha}{T}\sum_{t=0}^{T-1}\mE[\ltwo{G_{\text{DR}}(w_t,\mcm_t) - \nabla_w J(w_t)}^2] + \frac{L_{sc}\alpha}{T}\sum_{t=0}^{T-1}\mE\left[\ltwo{\nabla_w J(w_t)}^2\right] \nonumber\\
		&\quad + \frac{C_{sc}}{T}\sum_{t=0}^{T-1}\mE[\ltwo{G_{\text{DR}}(w_t,\mcm_t) - \nabla_w J(w_t)}] + \frac{C_{sc}}{T}\sum_{t=0}^{T-1}\mE[\ltwo{\nabla_w J(w_t) - w^*_t}].\label{eq: 57}
	\end{flalign}
	We then bound the error terms on the right-hand side of \cref{eq: 57} separately.
	
	First consider the term $\frac{1}{T}\sum_{t=0}^{T-1}\mE\left[\ltwo{\nabla_w J(w_t)}^2\right]$. We have
	\begin{flalign}
		\frac{1}{T}\sum_{t=0}^{T-1}&\mE\left[\ltwo{\nabla_w J(w_t)}^2\right] = \mE\left[\ltwo{\nabla_w J(w_{\hat{T}})}^2\right]\nonumber\\
		&\overset{(i)}{\leq}  \Theta\left(\frac{1}{T}\right) + \Theta\left(\frac{1}{N}\right) + \Theta(\epsilon^2_\rho \epsilon^2_{d_q} + \epsilon^2_{d_\rho}\epsilon^2_q + \epsilon^2_\rho \epsilon^2_q),\label{eq: 64}
	\end{flalign}
	where $(i)$ follows from \Cref{thm2}.
	
	Then we consider the term $\frac{1}{T}\sum_{t=0}^{T-1}\ltwo{\nabla_w J(w_t) - w^*_t}$. We proceed as follows.
	\begin{flalign}
		\frac{1}{T}\sum_{t=0}^{T-1}&\mE\left[\ltwo{\nabla_w J(w_t) - w^*_t}\right]\nonumber\\
		&=\frac{1}{T}\sum_{t=0}^{T-1}\mE\left[\ltwo{(I- F^{-1}(w_t))\nabla_w J(w_t)}\right]\leq \left( 1+\frac{1}{\lambda_F} \right)\frac{1}{T}\sum_{t=0}^{T-1}\mE\left[\ltwo{\nabla_w J(w_t)}\right]\nonumber\\
		&=  \left( 1+\frac{1}{\lambda_F} \right)\mE\left[\ltwo{\nabla_w J(w_{\hat{T}})}\right] = \left( 1+\frac{1}{\lambda_F} \right)\sqrt{\mE\left[\ltwo{\nabla_w J(w_{\hat{T}})}^2\right]}\nonumber\\
		&\overset{(i)}{\leq} \Theta\left(\frac{1}{\sqrt{T}}\right) + \Theta\left(\frac{1}{\sqrt{N}}\right) + \Theta(\epsilon_\rho \epsilon_{d_q} + \epsilon_{d_\rho}\epsilon_q + \epsilon_\rho \epsilon_q),\label{eq: 58}
	\end{flalign}
	where $(i)$ follows from \cref{eq: 64} and the fact that $\sqrt{\sum_i a_i}\leq\sum_i\sqrt{a_i}$ for $a_i\geq 0$. 
	
	We then consider the term $\frac{1}{T}\sum_{t=0}^{T-1}\mE[\ltwo{G_{\text{DR}}(w_t,\mcm_t) - \nabla_w J(w_t)}] $. Recalling \cref{eq: 49}, we have
	\begin{flalign}
		\mE\left[\ltwo{G_{\text{DR}}(w_t,\mcm_t) - \nabla_w J(w_t)}^2|\mf_t\right]\leq C_{11} \Delta_t + \frac{24 C^2_{10}}{N} + C_{12}(\epsilon^2_\rho \epsilon^2_{d_q} + \epsilon^2_{d_\rho}\epsilon^2_q + \epsilon^2_\rho \epsilon^2_q),\nonumber
	\end{flalign}
	which implies
	\begin{flalign}
		\mE[\ltwo{G_{\text{DR}}(w_t,\mcm_t) - \nabla_w J(w_t)}|\mf_t]&\leq \sqrt{\mE\left[\ltwo{G_{\text{DR}}(w_t,\mcm_t) - \nabla_w J(w_t)}^2|\mf_t\right]}\nonumber\\
		&\leq \sqrt{C_{11}}\sqrt{\Delta_t} + \frac{5 C_{10}}{\sqrt{N}} + \sqrt{C_{12}}(\epsilon_\rho \epsilon_{d_q} + \epsilon_{d_\rho}\epsilon_q + \epsilon_\rho \epsilon_q).\label{eq: 59}
	\end{flalign}
	Taking the expectation on both sides of \cref{eq: 59} and taking the summation over $t=1\cdots T-1$ yield
	\begin{flalign}
		\frac{1}{T}\sum_{t=0}^{T-1}\mE[\ltwo{G_{\text{DR}}(w_t,\mcm_t) - \nabla_w J(w_t)}]\leq \frac{\sqrt{C_{11}}}{T}\sum_{t=0}^{T-1}\mE\left[\sqrt{\Delta_t}\right] + \frac{5 C_{10}}{\sqrt{N}} + \sqrt{C_{12}}(\epsilon_\rho \epsilon_{d_q} + \epsilon_{d_\rho}\epsilon_q + \epsilon_\rho \epsilon_q).\label{eq: 60}
	\end{flalign}
	We then bound the term $\sum_{t=0}^{T-1}\mE\left[\sqrt{\Delta_t}\right]$. Note that 
	\begin{flalign}
		\frac{1}{T}\sum_{t=0}^{T-1}\mE[\sqrt{\Delta_t}] = \mE[\sqrt{\Delta_{\hat{T}}}]\leq \sqrt{\mE[\Delta_{\hat{T}}]} = \sqrt{\frac{1}{T}\sum_{t=0}^{T-1}\mE[\Delta_t]}.\label{eq: 61}
	\end{flalign}
	Recalling \cref{eq: 54}, we have
	\begin{flalign}
		\sum_{t=0}^{T-1}\mE[\Delta_t]\leq \frac{4}{\varrho}\Delta_0 + \frac{4C_{12}\alpha^2}{\varrho}\sum_{t=0}^{T-1}\mE[\ltwo{\nabla_w J(w_t)}^2] + \frac{4C_{13}T}{\varrho N} + \frac{4C_{14}T}{\varrho}(\epsilon^2_\rho \epsilon^2_{d_q} + \epsilon^2_{d_\rho}\epsilon^2_q + \epsilon^2_\rho \epsilon^2_q),\nonumber
	\end{flalign}
	which implies
	\begin{flalign}
		\frac{1}{T}\sum_{t=0}^{T-1}\mE[\sqrt{\Delta_t}]&\leq\sqrt{\frac{1}{T}\sum_{t=0}^{T-1}\mE[\Delta_t]} \nonumber\\
		&\leq \sqrt{\frac{4\Delta_0}{\varrho T}} + \frac{2\sqrt{C_{12}}\alpha}{\sqrt{\varrho}}\sqrt{\mE[\ltwo{\nabla_w J(w_{\hat{T}})}^2]} + \sqrt{\frac{4C_{13}}{\varrho N}} + \sqrt{\frac{4C_{14}}{\varrho}}(\epsilon_\rho \epsilon_{d_q} + \epsilon_{d_\rho}\epsilon_q + \epsilon_\rho \epsilon_q)\nonumber\\
		&\overset{(i)}{\leq} \Theta\left(\frac{1}{\sqrt{T}}\right) + \Theta\left(\frac{1}{\sqrt{N}}\right) + \Theta(\epsilon_\rho \epsilon_{d_q} + \epsilon_{d_\rho}\epsilon_q + \epsilon_\rho \epsilon_q),\label{eq: 62}
	\end{flalign}
	where $(i)$ follows from \cref{eq: 64}. Substituting \cref{eq: 62} into \cref{eq: 60} yields
	\begin{flalign}
		\frac{1}{T}\sum_{t=0}^{T-1}\mE[\ltwo{G_{\text{DR}}(w_t,\mcm_t) - \nabla_w J(w_t)}]\leq \Theta\left(\frac{1}{\sqrt{T}}\right) + \Theta\left(\frac{1}{\sqrt{N}}\right) + \Theta(\epsilon_\rho \epsilon_{d_q} + \epsilon_{d_\rho}\epsilon_q + \epsilon_\rho \epsilon_q).\label{eq: 63}
	\end{flalign}
	
	Finally, we consider the term $\frac{1}{T}\sum_{t=0}^{T-1}\mE[\ltwo{G_{\text{DR}}(w_t,\mcm_t) - \nabla_w J(w_t)}^2] $. Taking the expectation on both sides of \cref{eq: 49} and taking the summation over $t=0,\cdots,T-1$ yield
	\begin{flalign}
		&\quad \frac{1}{T}\sum_{t=0}^{T-1}\mE\left[\ltwo{G_{\text{DR}}(w_t,\mcm_t) - \nabla_w J(w_t)}^2\right]\leq \frac{C_{11}}{T} \sum_{t=0}^{T-1}\mE[\Delta_t] + \frac{24 C^2_{10}}{N} + C_{12}(\epsilon^2_\rho \epsilon^2_{d_q} + \epsilon^2_{d_\rho}\epsilon^2_q + \epsilon^2_\rho \epsilon^2_q)\nonumber\\
		&\overset{(i)}{\leq} \frac{4C_{11}}{\varrho T}\Delta_0 + \frac{4C_{11}C_{12}\alpha^2}{\varrho T}\sum_{t=0}^{T-1}\mE[\ltwo{\nabla_w J(w_t)}^2] + \frac{4C_{11}C_{13}}{\varrho N} + \frac{4C_{11}C_{14}}{\varrho}(\epsilon^2_\rho \epsilon^2_{d_q} + \epsilon^2_{d_\rho}\epsilon^2_q + \epsilon^2_\rho \epsilon^2_q)\nonumber\\
		&\quad + \frac{24 C^2_{10}}{N} + C_{12}(\epsilon^2_\rho \epsilon^2_{d_q} + \epsilon^2_{d_\rho}\epsilon^2_q + \epsilon^2_\rho \epsilon^2_q)\nonumber\\
		&\overset{(ii)}{\leq} \Theta\left(\frac{1}{T}\right) + \Theta\left(\frac{1}{N}\right) + \Theta(\epsilon^2_\rho \epsilon^2_{d_q} + \epsilon^2_{d_\rho}\epsilon^2_q + \epsilon^2_\rho \epsilon^2_q),\label{eq: 65}
	\end{flalign}
	where $(i)$ follows from \cref{eq: 54}, and $(ii)$ follows from \Cref{thm2}.
	
	Substituting \cref{eq: 64}, \cref{eq: 58}, \cref{eq: 63} and \cref{eq: 65} into \cref{eq: 57} yields
	\begin{flalign*}
		J(\pi^*) - J(w_{\hat{T}})\leq \frac{{\epsilon_{approx}}}{1-\gamma} + \Theta\left(\frac{1}{\sqrt{T}}\right) + \Theta\left(\frac{1}{\sqrt{N}}\right) + \Theta(\epsilon_\rho \epsilon_{d_q} + \epsilon_{d_\rho}\epsilon_q + \epsilon_\rho \epsilon_q),
	\end{flalign*}
	which completes the proof.

%	Recall \cref{eq: 52}, we have
%	\begin{flalign}
%		\mE[\Delta_{t+1}|\mf_t]\leq \left( 1 - \frac{1}{4}\varrho \right) \Delta_t + C_{12}\alpha^2\mE[\ltwo{\nabla_w J(w_t)}^2|\mf_t] + \frac{C_{13}}{N} + C_{14}(\epsilon^2_\rho \epsilon^2_{d_q} + \epsilon^2_{d_\rho}\epsilon^2_q + \epsilon^2_\rho \epsilon^2_q),\nonumber
%	\end{flalign}
%	which implies that
%	\begin{flalign}
%		\mE[\sqrt{\Delta_{t+1}}|\mf_t]&\leq \sqrt{\mE[\Delta_{t+1}|\mf_t]} \nonumber\\
%		&\leq \sqrt{1 - \frac{1}{4}\varrho}  \sqrt{\Delta_t} + \sqrt{C_{12}}\alpha\sqrt{\mE[\ltwo{\nabla_w J(w_t)}^2|\mf_t]} + \frac{\sqrt{C_{13}}}{\sqrt{N}} + \sqrt{C_{14}}(\epsilon_\rho \epsilon_{d_q} + \epsilon_{d_\rho}\epsilon_q + \epsilon_\rho \epsilon_q)\nonumber\\
%		&\leq \left( 1 - \vartheta \right)\sqrt{\Delta_t} + \sqrt{C_{12}}\alpha\sqrt{\mE[\ltwo{\nabla_w J(w_t)}^2|\mf_t]} + \frac{\sqrt{C_{13}}}{\sqrt{N}} + \sqrt{C_{14}}(\epsilon_\rho \epsilon_{d_q} + \epsilon_{d_\rho}\epsilon_q + \epsilon_\rho \epsilon_q)
%	\end{flalign}
\end{proof}

%\section{Other Supporting Lemmas}
%In this section, we include the proof of boundness properties of matrices that are necessary for the proof of lemmas in \cref{sc: lemma_pf_thm2}.
%\begin{lemma}\label{lemma8}
%	Consider $M_{i,t}$ and $m_{i,t}$ defined in \cref{eq: 66}, $P_{i,t}$ and $p_{i,t}$ defined in \cref{eq: 67}, and $U_{i,t}$ and $u_{i,t}$ defined in \cref{eq: 67}, there exist positive constant $C_M$ and $C_m$, $C_P$ and $C_p$, and $C_U$ and $C_u$, such that $\lF{M_{i,t}}\leq C_M$ and $\ltwo{m_{i,t}}\leq C_m$, where $C_M=\sqrt{9C^4_\phi + 2C^2_\phi + 1}$, $C_m= \sqrt{C^2_\phi + 1}$, $C_P=\sqrt{8C^4_\varphi + d_2}$, $C_p = {C_\varphi C_{sc}}$, and $C_U=\sqrt{8C^4_x+d_3}$, $C_u= R_q C_{sc}$.
%\end{lemma}
%\begin{proof}
%	Recall the definition of $M_{i,t}$ and $m_{i,t}$ in \Cref{lemma2}, we have
%	\begin{flalign*}
%	M_{i,t} &= \left[\begin{array}{ccc}
%	-\phi_i^\top\phi_i & -(\phi_i-\gamma\phi^\prime_i)\phi^\top_i & 0 \\
%	\phi_i(\phi^\top_i-\gamma\phi^{\prime\top}_i) & 0 & -\phi_i \\
%	0 & \phi^\top_i & -1
%	\end{array}
%	\right],\qquad
%	m_{i,t} = \left[\begin{array}{c}
%	(1-\gamma)\phi_{0,i}\\
%	-r_i\phi_i\\
%	-1
%	\end{array}
%	\right].
%	\end{flalign*}
%	
%	Following similar steps, we can proceed to derive the following upper bounds:
%	\begin{flalign*}
%		\lF{P_t}&\leq\sqrt{8C^4_\varphi + d_2},\qquad\text{and}\qquad \ltwo{p_t}\leq {C_\varphi C_{sc}},\nonumber\\
%		\lF{U_t}&\leq\sqrt{8C^4_x+d_3},\qquad\text{and}\qquad \ltwo{u_t}\leq C_\phi R_q C_{sc}.\nonumber
%	\end{flalign*}
%\end{proof}

\end{document}